\documentclass{article}

\PassOptionsToPackage{numbers, compress}{natbib}
  \usepackage[main, final]{neurips_2026}

\usepackage[utf8]{inputenc} 
\usepackage[T1]{fontenc}    
\usepackage{hyperref}       
\usepackage{url}            
\usepackage{booktabs}       
\usepackage{amsfonts}       
\usepackage{nicefrac}       
\usepackage{microtype}      
\usepackage{xcolor}         

\usepackage{amsmath}
\usepackage{amssymb}
\usepackage{mathtools}
\usepackage{amsthm}
\usepackage{siunitx}
\usepackage{graphicx}
\usepackage{subcaption}
\usepackage{algorithmic}

\usepackage{wrapfig}

\usepackage[capitalize,noabbrev]{cleveref}

\theoremstyle{plain}
\newtheorem{theorem}{Theorem}[section]

\newtheorem{lemma}[theorem]{Lemma}

\theoremstyle{definition}
\newtheorem{definition}[theorem]{Definition}

\theoremstyle{remark}
\newtheorem{remark}[theorem]{Remark}

\usepackage{comment}
\usepackage{bbm}
\usepackage{algorithm}          

\usepackage{amsmath,amsfonts,bm,mathtools,dsfont,bigints}

\def\eqref#1{equation~\ref{#1}}

\def\1{\bm{1}}

\def\vzero{{\bm{0}}}

\def\vmu{{\bm{\mu}}}

\def\vrho{{\bm{\rho}}}

\def\ve{{\bm{e}}}

\def\vm{{\bm{m}}}

\def\vx{{\bm{x}}}
\def\vy{{\bm{y}}}
\def\vz{{\bm{z}}}

\def\evm{{m}}

\def\evy{{y}}

\def\mC{{\bm{C}}}

\def\mG{{\bm{G}}}

\def\mm{{\bm{m}}}

\def\mZ{{\bm{Z}}}

\def\eye{{\bm{I}}}
\def\mmu{{\bm{\mu}}}

\DeclareMathAlphabet{\mathsfit}{\encodingdefault}{\sfdefault}{m}{sl}
\SetMathAlphabet{\mathsfit}{bold}{\encodingdefault}{\sfdefault}{bx}{n}

\def\sN{{\mathbb{N}}}

\def\sR{{\mathbb{R}}}

\def\emm{{m}}

\usepackage{thmtools}
\usepackage{thm-restate}
\usepackage{enumitem}

\renewcommand{\epsilon}{\varepsilon}
\newcommand{\dd}{\mathrm{d}}

\title{Sampling-Free Privacy Accounting\\for Matrix Mechanisms under Random Allocation}

\author{%
  Jan Schuchardt\\
  Machine Learning Research, Morgan Stanley\\
  \texttt{jan.a.schuchardt@morganstanley.com}\\
  \And
  Nikita Kalinin \\
  Institute of Science and Technology Austria \\
  \texttt{nikita.kalinin@ist.ac.at} \\
}

\begin{document}

\maketitle

\begin{abstract}
 We study privacy amplification for differentially private model training with matrix factorization under random allocation (also known as the balls-in-bins model). Recent work by Choquette-Choo et al. (2025) proposes a sampling-based Monte Carlo  approach to compute amplification parameters in this setting.
 However, their guarantees either only hold with some high probability or require random abstention by the mechanism.
 Furthermore, the required number of samples for ensuring $(\epsilon,\delta)$-DP is inversely proportional to $\delta$.
 In contrast, we develop sampling-free bounds based on Rényi divergence and conditional composition. 
 The former
 is facilitated by a dynamic programming formulation to efficiently compute the bounds. 
 The latter complements it by offering stronger privacy guarantees for small $\epsilon$, where R\'enyi divergence bounds inherently lead to an over-approximation.
 Our framework applies to arbitrary banded and non-banded matrices. Through numerical comparisons, we demonstrate the efficacy of our approach across a broad range of matrix mechanisms used in research and practice.
\end{abstract}

\section{Introduction}

Differential privacy (DP) \citep{dwork2006differential} offers one of the most well-established and mathematically rigorous definitions of privacy in statistical learning. In the context of deep learning, it is typically enforced via differentially private stochastic gradient descent (DP-SGD) \citep{abadi2016deep}. In DP-SGD, per-example gradients are clipped to bound the contribution of any individual data point, and Gaussian noise is added to guarantee privacy. While the injected noise inevitably degrades model utility, this loss can be mitigated through privacy \textit{amplification by subsampling}~\cite{li2012sampling}. Specifically, by randomly sampling a subset of the data at each iteration, we can achieve the same privacy guarantees with less additive noise. This improvement arises from the adversary’s uncertainty about which examples are included in each batch, introducing additional randomness that strengthens the overall privacy guarantees.

Several subsampling schemes have been studied in the context of differentially private training. One of the most commonly used is Poisson subsampling \citep{abadi2016deep, mironov2017renyi, koskela2020computing, zhu2022optimal}, in which each data point is independently included in a batch with a fixed probability chosen to achieve a desired expected batch size. However, this scheme poses practical challenges for implementation \citep{lebeda2025avoiding, beltran2024towards}. In particular, for large datasets and long training runs, two issues arise. First, drawing an independent random subset at every iteration is often computationally infeasible. Second, subsampling can lead to uneven data coverage, with a substantial fraction of examples (approximately $1/e$) not appearing within an epoch, and may increase the variance of gradients~\cite{feldman2025privacy}. As a result, a common but problematic practice is to train on shuffled data while reporting privacy guarantees derived for Poisson subsampling \citep{lebeda2025avoiding}.

A principled alternative that has emerged to address this issue is the  ``balls-in-bins'' scheme \citep{chua2024balls, choquette2024near}, also known as balanced iteration subsampling~\cite{dong2025leveraging} or as random allocation \citep{feldman2025privacy,feldman2026efficient}. This approach predefines batches for an entire epoch and reuses them throughout training, ensuring a fixed number of participations per data point. As a result, it mitigates both the per-iteration sampling overhead and the uneven within-epoch data coverage, while requiring a noise multiplier comparable to that of Poisson subsampling to attain the same privacy guarantees. 

Beyond better subsampling schemes, another principled approach to improve upon the utility DP-SGD is to introduce correlated noise. The key idea is to add more noise at each iteration while carefully correlating it across steps so that the aggregate noise has lower variance than in standard DP-SGD, thereby improving model utility. The most common way to implement such noise correlations is through the matrix factorization mechanism \citep{li2015matrix, denisov2022improved, choquette2023multi, choquette-choo2023amplified}, where Gaussian noise $Z$ is correlated across iterations via the linear transformation $C^{-1}$, i.e., sampling $C^{-1}Z$, and then adding it to the gradients.
Matrix factorization, however, can be impractical at scale, motivating memory-efficient structured variants. Prior work proposes banded matrix factorizations~\citep{kalinin2024banded, mckenna2024scaling}, where the strategy matrix $C$ is banded and Toeplitz; banded-inverse factorizations~\citep{kalinin2025back}, where $C^{-1}$ is banded; and hybrid constructions such as the Buffered Linear Toeplitz (BLT) matrix~\citep{dvijotham2024efficient, mcmahan2024hassle}. These designs enable correlated-noise mechanisms with substantially reduced memory requirements.

Privacy amplification by subsampling and noise correlation techniques can be combined. This has been demonstrated for banded matrix mechanisms under Poisson subsampling \citep{choquette-choo2023amplified, choquette2023privacy}, and more recently for general correlation matrices using balls-in-bins subsampling \citep{choquette2024near}. The latter, however, relies on Monte Carlo estimation, which can require an intractably large number of samples in the high-privacy regime and thus become computationally expensive. Moreover, it does not directly yield deterministic DP guarantees, but rather privacy bounds that only hold with high probability.

Motivated by the recent independent works of~\citet{feldman2025privacy,feldman2026efficient} and~\citet{dong2025leveraging}, who derived privacy bounds for DP-SGD under balls-in-bins batching without Monte Carlo sampling, we propose two procedures for computing amplification bounds for matrix mechanisms under this batching scheme. Our methods avoid Monte Carlo sampling and instead compute the bounds explicitly, leveraging R\'enyi divergence via dynamic programming
as well as conditional composition while accounting for the inter-step correlations unique to matrix mechanisms.

\textbf{Contribution}
\begin{itemize}[nosep]
    \item We introduce deterministic, Rényi divergence based and conditional composition based privacy analyses (accountants) for subsampled matrix mechanisms.
    \item For the special case of DP-SGD, we improve the computational complexity of R\'enyi accounting under the random allocation (balls-in-bins) batching scheme.
\end{itemize}

\section{Background}

\subsection{Definitions of differential privacy}\label{section:privacy_definitions}

\textit{Differential Privacy (DP)}  is a widely accepted framework for providing rigorous privacy guarantees. We recall its definition below:

\begin{definition}[Differential Privacy \citep{dwork2006differential}]
A randomized mechanism is said to provide $(\varepsilon, \delta)$-differential privacy under a neighboring relation $\simeq$ if, for all datasets $D \simeq D'$, and for all measurable subsets $S$ of the mechanism's output space, 
$
\Pr[M(D) \in S] \leq e^{\varepsilon} \cdot \Pr[M(D') \in S] + \delta.
$\end{definition}
In the following, we  write $M(\cdot \mid D)$ for the distribution of random variable $M(D)$.
We further use 
an equivalent definition of differential privacy that is more convenient for privacy accounting, i.e., tracking privacy over multiple applications of a mechanism.
\begin{lemma}[Proposition 2 from~\cite{barthe2013beyond}]
    A mechanism provides $(\epsilon, \delta)$-DP under $\simeq$ if and only if, for all $D \simeq D'$,
    $
        H_{e^\epsilon}\left(M(\cdot \mid D), M(\cdot \mid D')\right) \leq \delta
    $
    with hockey-stick divergence
    \begin{equation}\label{definition:hockey_stick_divergence}
        H_{\gamma}(P|| Q) = \mathbb{E}_{x \sim P} \left[ \max\left\{ 1 - \gamma \frac{\dd Q}{\dd P}(x),\; 0 \right\} \right].
    \end{equation}
\end{lemma}
While $(\epsilon,\delta)$-DP corresponds to a bound on $H_\gamma$ for a single choice of 
$\gamma=e^\epsilon$, tight privacy accounting requires bounds for all $\gamma \geq 0$.
In the following, we recall multiple equivalent~\cite{zhu2022optimal} characterizations of this~\emph{privacy profile} that are immediately relevant to  our work, beginning with dominating pairs:
\begin{definition}\label{definition:dominating_pairs}
    We say that the pair of distributions $(\hat{P}, \hat{Q})$ dominates $(P,Q)$, denoted by $(P,Q) \prec (\hat{P},\hat{Q})$, 
    if $H_\gamma(P||Q) \leq H_\gamma(\hat{P},\hat{Q})$ for all $\gamma > 0$.
\end{definition}
\begin{definition}[\citet{zhu2022optimal}]
    We say that $(\hat{P}, \hat{Q})$ are a \emph{dominating pair} of mechanism $M$ under neighboring relation $\simeq$ if
    $(M(\cdot \mid D), M(\cdot \mid D')) \prec (\hat{P}, \hat{Q})$ for all $D \simeq D'$.
\end{definition}
Given a dominating pair $(P,Q)$, the problem of privacy accounting reduces to quantifying the hockey-stick divergence $H_\gamma(P||Q)$. Specifically, the divergence does not directly depend on the (potentially high-dimensional) structure of $(P, Q)$ but on their log-likelihood ratio, the ``privacy loss'':
\begin{definition}
    The privacy loss random variable $L_{P,Q}$ of $P,Q$ is $\log(\frac{\dd P}{\dd Q}(x))$ with $x \sim P$.
\end{definition}
From~\cref{definition:hockey_stick_divergence} and $\gamma \frac{\dd P}{\dd Q}(x) = e^{\log(\gamma) - L_{P,Q}}$, we see that the hockey-stick divergence quantifies  how much the privacy loss random variable (PLRV) exceeds $\epsilon=\log(\gamma)$ in expectation.
While the distribution of $L_{P,Q}$ fully determines the privacy profile, it is often analytically more convenient to work with its moment generating properties. Specifically, the distribution   
is uniquely characterized by its cumulant generating function~\cite{zhu2022optimal}, which is proportional to the Rényi divergence:
\footnote{$R_\alpha(P||Q)$ of dominating pair $P,Q$ is not to be confused with R\'enyi DP~\cite{mironov2017renyi}, which is strictly less informative, see~\cite{zhu2022optimal}.}
\begin{definition}
    The R\'enyi divergence of $P$ and $Q$ is
    $
        R_\alpha(P||Q) = \frac{1}{\alpha - 1} \log \mathbb{E}\left[  \exp( (\alpha - 1) L_{P,Q}) \right].
    $
\end{definition}

\subsection{(Subsampled) matrix mechanisms}\label{section:background_matrix_mechanisms}

Training a model with $d$ parameters using differentially private stochastic gradient descent (DP-SGD) for $k$ epochs with $b$ batches per epoch (and thus $N = k \cdot b$ iterations) can be viewed as computing the noisy prefix sums $\mathbf{E}(\mG + \mZ)$ row by row, where $\mathbf{E} \in \{0,1\}^{N \times N}$ is the all-ones lower-triangular matrix; the entries of $\mZ \in \sR^{N \times d}$ are drawn i.i.d.\ from a normal distribution with standard deviation $\sigma > 0$; and $\mG \in \sR^{N \times d}$ contains the aggregated gradients for each batch, clipped to a fixed $\ell_2$-norm.

Over the past few years, various works have theoretically and empirically demonstrated that improved estimation can be achieved by correlating the noise across iterations. Specifically, by introducing a so-called correlation matrix $\mC^{-1} \in \sR^{N \times N}$, the matrix factorization mechanism \citep{li2015matrix} adds correlated noise via $\mathbf{E}\bigl(\mG + \mC^{-1}\mZ\bigr)$. This approach, however, requires adjustments to the privacy accounting. Without subsampling, it is equivalent (by the post-processing property) to applying the Gaussian mechanism to make $\mC\mG$ private by adding noise $\mZ$, and it therefore requires computing the corresponding sensitivity \citep{choquette-choo2023amplified}. The resulting scheme can offer superior utility: although the sensitivity may increase (and thus the noise scale must be adjusted accordingly), the total injected noise $\mathbf{E}\mC^{-1}\mZ$ can have a much smaller expected norm. By adding less noise to the intermediate models, training utility can improve; see the recent survey \citep{pillutla2025correlated} for further discussion on selecting a correlation matrix.

The utility of the matrix factorization mechanism can be further improved by combining it with amplification by subsampling. A recent work by \citet{choquette2024near} introduces an efficient balls-in-bins subsampling scheme, along with a corresponding privacy accountant that supports matrix factorization with an arbitrary positive lower-triangular \emph{strategy} matrix~$\mC$.

\begin{definition}\label{definition:bnb}
    In \emph{balls-in-bins} subsampling with $k \in \sN$ epochs and $b \in \sN $ batches per epoch,
    a record contributes to iterations $(i, b + i, 2b + i, \dots, (k-1)b + i)$ with $i \sim \mathrm{Uniform}(\{1,\dots, b\})$.
\end{definition}
They further derive  (almost) tight dominating pairs for matrix mechanism under this scheme.

\begin{restatable}{lemma}{fulldominatingpair}[Lemma 3.2 from \cite{choquette2024near}]
\label{lem:dominating_pair}
Assume lower-triangular, non-negative strategy matrix $\mC$, gradient clipping norm $\Delta=1$, and balls-in-bins sampling with $k \in \sN$ epochs and $b \in \sN$ batches per epoch. 
Define distributions $P = \frac{1}{b}\sum_{i=1}^{b} \mathcal{N}(\mm_i, \sigma^2 \eye_N)$
and $Q = \mathcal{N}(\vzero, \sigma^2 \eye_N)$ supported on $\sR^{N}$
with $N = k \cdot b$ and mixture means
\label{eq:dominating_pair}
$\mm_i = \sum_{j = 0}^{k - 1} |\mathbf{C}|_{:,\, i + jb}$,
Then the subsampled matrix mechanism is dominated by $(P,Q)$ under the ``remove'' \footnote{The original Lemma 3.2 swaps the ``add'' and ``remove'' direction, which appears to be a typing error, see~\cref{appendix:add_remove_flip}.} relation and by
$(Q,P)$ under the ``add'' relation.
\end{restatable}
However, the main algorithmic challenge in privacy accounting is not to just specify a dominating pair $(P,Q)$, but to efficiently evaluate or bound the hockey-stick divergence $H_\gamma(P || Q)$.
Monte Carlo privacy accountants~\cite{wang2023randomized} 
compute a bound by leveraging the  
equality of $H_\gamma(P || Q)$
to the expectation $\mathbb{E}\left[\max\{1 - e^{\log(\gamma) - L_{P,Q}}, 0\}\right]$
we discussed in~\cref{section:privacy_definitions}:
They sample from the privacy loss distribution 
and determine a confidence interval on the expectation to show that the mechanism is $(\epsilon,\delta)$-DP with some high probability.
This is however a qualitatively weaker guarantee than a deterministic bound on $\delta$.\footnote{Note that, despite a common misinterpretation, privacy parameter $\delta$ is not the probability of violating $\epsilon$-DP.}
One can recover deterministic guarantees, but this requires modification of the mechanism with random abstentions.

Furthermore, to provide high probability bounds for the privacy level, the instantiation for matrix mechanisms in~\cite{choquette2024near} requires a sample size inversely proportional to $\delta$. Thus, showing $(\epsilon,10^{-8})$-DP is $\SI{100000}{}$ times more expensive than showing $(\epsilon,10^{-3})$-DP.
So far, no sampling-free method is available to overcome these inherent limitations of Monte Carlo accounting for matrix mechanisms. 

\section{Problem statement}\label{section:problem_statement}
Our goal is to develop methods that enable efficient, deterministic privacy accounting for matrix mechanisms under random allocation.
The main challenge is that the multidimensional mixture $P$ in~\cref{eq:dominating_pair} results in a privacy loss random variable (PLRV) $L_{P,Q}$ that cannot be decomposed into simpler PLRVs, ruling out a central technique in privacy accounting.

Concretely, standard privacy accountants explicitly (e.g.~\citep{sommer2019privacyloss,koskela2020computing,zhu2022optimal}) or implicitly (e.g.~\citep{abadi2016deep}) use that the dominating pair $(P,Q)$ for
any $N$-fold adaptive composition of mechanisms can be factorized as $P = \bigtimes_{n=1}^N P_n$ and $Q = \bigtimes_{n=1}^N Q_n$, which makes the PLRV a simple sum $L_{P,Q} = \sum_{n=1}^N L_{P_n,Q_n}$.
However, even without noise correlation ($\mC = \eye_N$), random allocation breaks this independence assumption, as can be seen from our $P$ in~\cref{eq:dominating_pair} being a mixture distribution.
In concurrent work, \citet{feldman2026efficient} show that $L_{P,Q}$ can nevertheless be decomposed into a function of $N$ simpler PLRVs by exploiting that, for $b=N$, each of the $b$ mixture components only differs from $Q$ in a single dimension ($\vm_i = \ve_i$). However, as we discuss more formally in~\cref{appendix:pld_inapplicability}, matrix mechanisms ($\mC \neq \eye_N$) break this assumption and make their analysis inapplicable.
In short, \emph{random allocation prevents decomposition into per-step privacy losses, while noise correlation further prevents decomposition into per-component privacy losses}.

In the following, we propose two complementary algorithms to overcome this challenge:
The first one analyzes the moment generating function of $L_{P,Q}$ instead of its distribution (R\'enyi accountant).
The second one constructs a pessimistic dominating pair consisting of product distributions, allowing us to apply standard accounting techniques to obtain sound bounds (conditional composition accountant).

\section{Rényi accountant}\label{section:rdp_accountant}

\begin{figure}[t]
    \centering
    \begin{subfigure}[t]{0.23\linewidth}
        \centering
        \includegraphics[width=\linewidth]{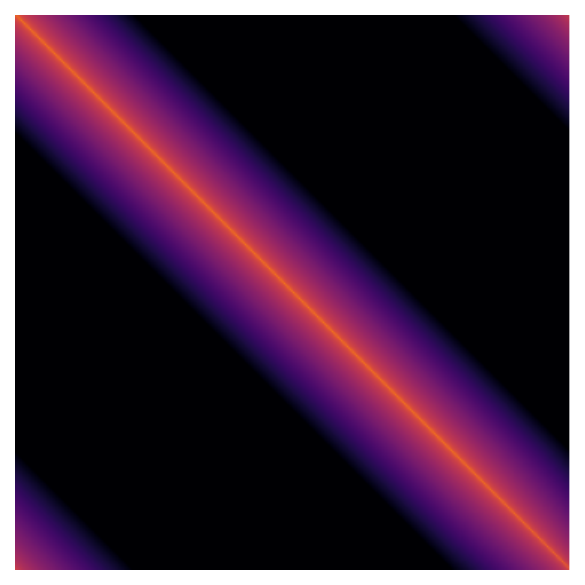}
        \subcaption{BandMF, $p=64$}
    \end{subfigure}\hfill
    \begin{subfigure}[t]{0.23\linewidth}
        \centering
        \includegraphics[width=\linewidth]{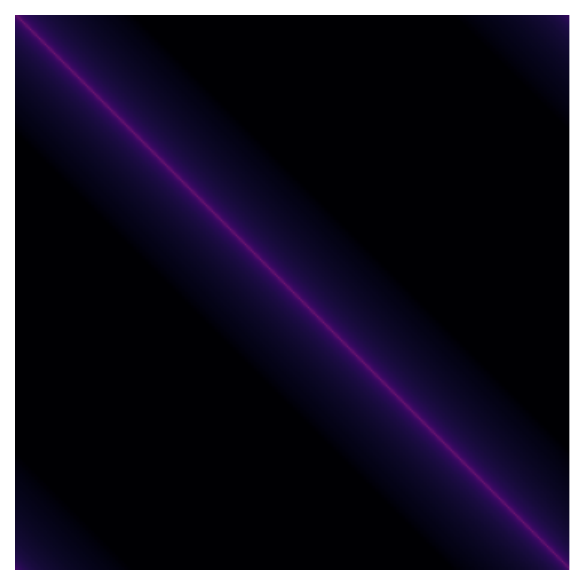}
        \subcaption{BSR, $p=64$}
    \end{subfigure}\hfill
    \begin{subfigure}[t]{0.23\linewidth}
        \centering
        \includegraphics[width=\linewidth]{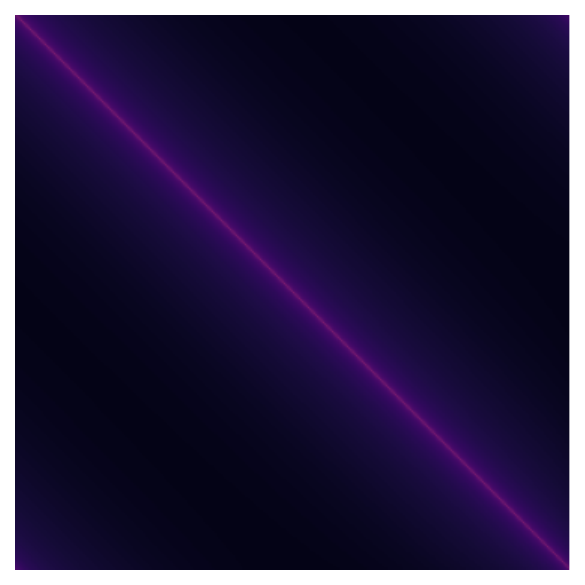}
        \subcaption{BISR, $p=64$}
    \end{subfigure}\hfill
    \begin{subfigure}[t]{0.23\linewidth}
        \centering
        \includegraphics[width=\linewidth]{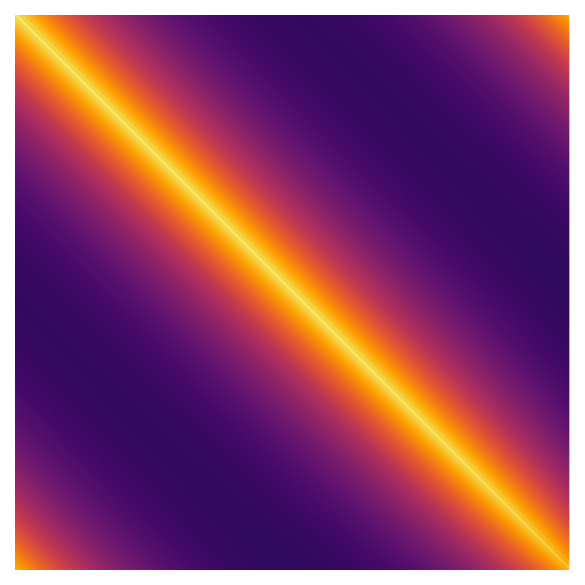}
        \subcaption{BandInvMF, $p=4$}
    \end{subfigure}

    \caption{The Gram matrix $\mathbf{G}$ for matrices of size $N=3000$ with $k=10$. For banded matrices $\mathbf{C}$, we consider Banded Matrix Factorization (BandMF) \citep{mckenna2024scaling} and Banded Square Root Factorization (BSR) \citep{kalinin2024banded} with bandwidth $p=64$. For banded inverse matrices, we consider Banded Inverse Matrix Factorization (BandInvMF) and Banded Inverse Square Root (BISR) \citep{kalinin2025back} with bandwidths $p = 4$ and $p=64$ accordingly.
  We note that although the matrix $\mathbf{C}$ is banded, the matrix $\mathbf{G}$ is not; rather, it is \emph{cyclically banded}, with the corner blocks filled by strictly positive entries. Moreover, we observe that banded inverse matrices show similar behavior: their Gram matrices have entries that decay rapidly outside the cyclic band.}
    \label{fig:gram_subplots}
\end{figure}

The main obstacle in evaluating the hockey-stick divergence $H_\gamma(P || Q)$ for our non-factorizable dominating pair from~\cref{eq:dominating_pair} is that it requires integration over the non-convex region $\frac{\dd P}{\dd Q}(x) > \gamma$.
In contrast, the R\'enyi divergence can be evaluated by integrating a smooth function of the privacy loss over the entirety of $\sR^N$,
avoiding the need to explicitly characterize this region.
In this section, we propose an efficient algorithm for evaluating or upper-bounding this integral.
We then translate these Rényi divergence bounds at order $\alpha$ into $(\varepsilon,\delta)$-differential privacy via the following theorem.

\begin{theorem}{Prop. 12 in  \citet{canonne2020discrete}.}
\label{thm:renyi_to_hockey_stick_translation}
Given two distributions $P, Q$, if $R_{\alpha}(P \,\|\, Q) \leq \rho$, then
\begin{equation}
H_{e^\epsilon}(P \,\|\, Q) \leq \frac{1}{\alpha - 1} e^{(\alpha - 1)(\rho - \epsilon)} \left(1 - \frac{1}{\alpha} \right)^{\alpha}.
\end{equation}
\end{theorem}

In the following lemma, we compute the Rényi divergence for the dominating pair from Lemma~\ref{lem:dominating_pair}, i.e., 
$
\hat{P} = \frac{1}{b}\sum_{i=1}^{b} \mathcal{N}(\mathbf{m}_i, \sigma^2 I)$
and
$
\hat{Q} = \mathcal{N}(0, \sigma^2 I),
$
in the ``remove'' direction.

\begin{restatable}{lemma}{RenyiDivergenceBound}
\label{lem:renyi-bound}
The Rényi divergence in the ``remove'' direction $R_\alpha(\hat{P}\,\|\,\hat{Q})$ is given by
\begin{equation}
\begin{aligned}
\label{eq:divergence_add}
\frac{1}{\alpha - 1}\,
\log\!\!\sum_{(r_1,\dots,r_\alpha)\in[1,b]^\alpha}
\!\!\!\!\!\exp\bigg(\!
  \sum_{\substack{j_1,j_2=1\\ j_1\ne j_2}}^\alpha
  \frac{\mathbf{G}_{r_{j_1}, r_{j_2}}}{2\sigma^2}
\!\bigg)-\frac{\alpha \log b}{1-\alpha},
\end{aligned}
\end{equation}
where $G$ is the Gram matrix with entries $\mathbf{G}_{i,j} = \langle \mathbf{m}_i, \mathbf{m}_j\rangle$.
\end{restatable}

For a general Gram matrix $\mathbf{G}$, there is no efficient way to evaluate this sum in \eqref{eq:divergence_add}. More precisely, in Lemma~\ref{lem:sum-sharp-p-complete}, we prove by a reduction from clique counting that computing this sum is $\#P$-complete, where $\#P$ is the class of counting problems associated with $\mathit{NP}$ decision problems. However, if we impose structural restrictions on $\mathbf{C}$, for instance by requiring it to be $p$ banded, then $\mathbf{G}$ inherits additional structure and becomes \emph{cyclic banded} (see Figure~\ref{fig:gram_subplots}). In this case, the sum can be computed via dynamic programming in time $O\!\left(bp\,\alpha^{2p}\right)$. In particular, for $p=1$, which corresponds to the DP-SGD case, this yields a runtime of $O(b\alpha^2)$, improving upon the exponential $O(2^\alpha)$ complexity reported by \citet{feldman2025privacy}.
In~\cref{appendix:renyi_runtime_vs_feldman}, we verify this dramatic speedup
via wall-clock runtime measurements.

If $\mathbf{C}$ is not naturally $p$ banded, as is the case for banded inverse matrices in~\citet{kalinin2025back}, the procedure above is no longer efficient. Nevertheless, many such matrices are effectively close to banded in the sense that entries sufficiently far from the diagonal are small (see Figure~\ref{fig:gram_subplots}). This observation leads to the following bound. Let
\begin{equation}\label{eq:rdp_correction_factor}
    \tau \coloneqq \max_{\min(|i-j|, b - |i-j|)\ge p} \mathbf{G}_{i,j}.
\end{equation}
Then $\mathbf{G}$ admits the elementwise bound $\mathbf{G} \le \mathbf{G}^{p} + \tau \mathbf{E}$, where $\mathbf{G}^{p}$ is the $p$ cyclic banded truncation of $\mathbf{G}$ and $\mathbf{E}$ denotes the all ones matrix. Consequently, by adding an extra term $\frac{\alpha \tau}{2\sigma^2}$ to the Rényi divergence, we obtain a computable upper bound on the privacy loss (for an empirical demonstration of the effect of truncation on bound tightness, see~\cref{appendix:effect_of_truncation}).

Combining these observations, Algorithm~\ref{alg:renyi_dynamic_program}, which we present in~\cref{appendix:full_renyi_accountant} for space reasons, computes the Rényi divergence in the ``remove'' direction exactly when $\mathbf{C}$ is $p$ banded, and an upper bound otherwise. Formally:

\begin{restatable}{lemma}{RenyiDynamicProgram}
\label{lem:renyi-dynamic_program}
Algorithm~\ref{alg:renyi_dynamic_program} computes the Rényi divergence based value in eq.~\ref{eq:divergence_add} exactly for a $p$-banded strategy matrix $\mathbf{C}$, and returns an upper bound otherwise.
The algorithm runs in time $O\!\left(bp \alpha^{2p} \right)$ and requires $O(\alpha^{p} + bp)$ memory.
\end{restatable}

Algorithm~\ref{alg:renyi_dynamic_program} evaluates the sum in~\eqref{eq:divergence_add} exactly when the Gram matrix $\mathbf{G}$ is $p$ cyclically banded, and returns an upper bound otherwise by truncating $\mathbf{G}$ to its $p$ cyclically banded part $\mathbf{G}^{(p)}$ and adding the correction term $\frac{\tau\alpha}{2\sigma^2}$. The indices in~\eqref{eq:divergence_add} interact cyclically; we break this cycle by conditioning on the short prefix $\mathbf{l}=(l_0,\dots,l_{p-2})$ and then compute the remaining contribution via a forward dynamic program. We glue the cycle back at the end by adding the interactions between indices $b-p+2,\dots,b-1$ and $0,\dots,p-2$. For numerical stability, all aggregations are performed using the log-sum-exp (LSE) operation. See the proof of Lemma~\ref{lem:renyi-dynamic_program} for a detailed description.

However, this is not sufficient to prove $(\varepsilon,\delta)$-DP, since the notion of neighboring datasets also requires a bound in the ``add'' direction. We found it substantially harder to compute the divergence $\mathrm{R}_{\alpha}(\hat{Q}\,\|\,\hat{P})$ explicitly. Nevertheless, it admits a tight and easily computable upper bound, which we present in the following lemma and formalize in Algorithm~\ref{alg:renyi_dynamic_program_remove} and~\cref{appendix:full_renyi_accountant}.

\begin{restatable}{lemma}{RemoveRenyiBound}
\label{lem:remove-renyi-bound}
The Rényi divergence in the ``add'' direction is bounded as
\begin{equation}
    \mathrm{R}_{\alpha}(\hat{Q}\|\hat{P})
    \le
    \frac{1}{2b\sigma^2} \sum_{j = 1}^b \mathbf{G}_{j,j}
    \;+\;
    \frac{\alpha - 1}{2b^2\sigma^2} \sum_{i = 1}^b \sum_{j = 1}^b \mathbf{G}_{i,j},
\end{equation}
where $\mathbf{G}$ is the Gram matrix with entries $\mathbf{G}_{i,j} = \langle \mathbf{m}_i, \mathbf{m}_j\rangle$.
\end{restatable}

Taking the maximum over the ``add'' and ``remove'' directions, $
\max\!\bigl(\mathrm{R}_{\alpha}(\hat{Q}\,\|\,\hat{P}),\,\mathrm{R}_{\alpha}(\hat{P}\,\|\,\hat{Q})\bigr)$, and converting it to a  hockey stick divergence via Theorem~\ref{thm:renyi_to_hockey_stick_translation} then completes the computation; see Algorithm~\ref{alg:renyi_dynamic_program_full}. In practice, we found the ``remove'' direction bound to consistently dominate, which 
aligns with observations from prior work (see details in~\cref{appendix:renyi_add_remove_slit_experiments}). To run the algorithm, we must choose the R\'enyi order $\alpha$ that yields the tightest privacy guarantee. For practical values of $\epsilon$ and $\delta$, the optimal $\alpha$ typically lies between $2$ and $40$. Theoretically, Theorem~\ref{thm:renyi_to_hockey_stick_translation} suggests that for small $\delta$, the optimal $\alpha$ grows logarithmically with $1/\delta$.

\begin{algorithm}[t!]
\caption{Rényi accountant full algorithm}
\label{alg:renyi_dynamic_program_full}
\begin{algorithmic}[1]

\REQUIRE Correlation matrix $\mathbf{C}$, bandwidth $p$, separation $b$, parameter $\alpha \in \mathbb{N}_{+}$.

\STATE $\rho_{\mathrm{remove}} \leftarrow \mathrm{ReAcc_{remove}}(\mathbf{C}, p, b, \alpha)$ \qquad \quad \, \text{\# Algorithm \ref{alg:renyi_dynamic_program}}

\STATE $\rho_{\mathrm{add}} \leftarrow \mathrm{ReAcc_{add}}(\mathbf{C}, p, b, \alpha)$ \quad \text{\# Algorithm \ref{alg:renyi_dynamic_program_remove}}

\STATE $\rho \leftarrow \max(\rho_{\mathrm{remove}}, \rho_{\mathrm{add}})$
\STATE
$\delta(\alpha,\varepsilon) \leftarrow
\exp\!\big((\rho-\varepsilon)(\alpha-1)\big)
(\alpha-1)^{-1}(1-1/\alpha)^\alpha$

\STATE \textbf{return} $\delta(\alpha,\varepsilon)$
\end{algorithmic}
\end{algorithm}

\begin{algorithm}[t!]
\caption{Rényi accountant ``add'' direction}
\label{alg:renyi_dynamic_program_remove}
\begin{algorithmic}[1]

\REQUIRE Correlation matrix $C$, bandwidth $p$, separation $b$, parameter $\alpha \in \mathbb{N}_{+}$.

\STATE $\mathbf{m}_i \leftarrow \sum_{j=0}^{k-1} |\mathbf{C}|_{:,\, i + jb}$

\STATE $\mathbf{G}_{i,j} \leftarrow \langle \mathbf{m}_i, \mathbf{m}_j \rangle \quad \forall i,j \in \{1,\dots,b\}$

\STATE $\rho \leftarrow \frac{1}{2b\sigma^2} \sum_{j = 1}^b \mathbf{G}_{j,j}
    \;+\;
    \frac{\alpha - 1}{2b^2\sigma^2} \sum_{i = 1}^b \sum_{j = 1}^b \mathbf{G}_{i,j}$
\STATE \textbf{return} $\rho$
\end{algorithmic}
\end{algorithm}

\section{Conditional composition accountant}

The proposed dynamic program allows for an exact evaluation of the R\'enyi divergence for any $\alpha_+ \in \sN_+$.\footnote{Real-valued orders $\alpha$ can be evaluated via interpolation, see Corollary 10 in~\cite{wang2019subsampled}}
These bounds can be tightly converted to privacy profiles $\delta(\epsilon)$ via the inverse Laplace transform~\cite{zhu2022optimal}, but this is generally intractable. Conversion formulae such as~\cref{thm:renyi_to_hockey_stick_translation} sacrifice some of this tightness, and are best suited for capturing the large deviation behavior of privacy loss $L_{P,Q}$~\cite{alghamdi2023saddle}, i.e., ``large $\epsilon$''.
We therefore derive a complementary method that---as our experiments demonstrate---yields tighter bounds under strict privacy budgets, i.e., ``small $\epsilon$''.

Concretely, recall from~\cref{section:problem_statement} 
that random allocation and noise correlation prevent
decomposition of the PLRV $L_{P,Q}$ into multiple simpler PLRVs.
This means that we cannot directly apply more recent privacy accountants that avoid the looseness of R\'enyi conversion formulae (e.g. PLD convolution~\cite{sommer2019privacyloss,koskela2020computing}).
In principle, we could force a factorization via product rule $p(\vy) = \prod_{n=1}^N p(y_n \mid \vy_{1:n-1})$
and $q(\vy) = \prod_{n=1}^N q(y_n \mid \vy_{1:n-1})$ to derive dominating pairs $(P_n, Q_n)$ with PLRV $L_{P_n,Q_n}$ for each step $n$.
However, by~\cref{definition:dominating_pairs}, these dominating pairs would need to hold for \emph{all} previous outcomes $\vy_{1:n-1}$. This includes outcomes that provide certain evidence of $\vy$ being drawn from a specific mixture component, i.e.,
a specific participation pattern chosen by random allocation. 
We would thus lose all privacy amplification.
Instead, we propose to leverage that, due to Gaussianity, such outcomes are exceedingly unlikely, which allows us to retain the benefit of subsampling.

The tool we use for this purpose is conditional composition\footnote{It is also referred to as posterior sampling in~\cite{feldman2025privacy}, following earlier work in~\cite{erlingsson2019amplification}.}
It lets us bound the overall mechanism's privacy profile via a factorizable dominating pair for ``good'' outcomes that provide little evidence for the participation pattern 
while separately accounting for low-probability ``bad'' outcomes.
The hockey-stick divergence for this dominating pair can then be evaluated using standard accountants.
In our work, we use the following formulation (for its relation to the general theorem, see~\cref{appendix:proofs_conditional_composition}).

\begin{restatable}{lemma}{conditionalcomposition}[Special case of Theorem 3.1 in~\cite{choquette2023privacy}]
\label{lemma:conditional_composition}
    Let $P, Q$ be distributions on $\sR^{N}$.
    For every step $n \in {N}$, choose a pair of distributions $P^{(n)}, Q^{(n)}$.
    Let $A^{(n)} \subseteq \sR^N$ refer to all ``good'' outcomes $\tilde{y}$ such that
    the conditional distributions $P_{y_n | \tilde{y}_{1:n-1}}$
    and $Q_{y_n | \tilde{y}_{1:n-1}}$ are dominated by $P^{(n)}, Q^{(n)}$.
    Define ``bad'' outcomes $E = \bigcup\nolimits_{n=1}^N \overline{A^{(n)}}$ with probability $\delta_E = P(E)$. Then, 
    \begin{equation}\label{eq:conditional_composition_bound}
        H_\gamma(P || Q) \leq H_\gamma \left( \bigtimes\nolimits_{n=1}^N P^{(n)}  || \bigtimes\nolimits_{n=1}^N  Q^{(N)} \right) + \delta_E.
    \end{equation}
\end{restatable}

While~\cref{lemma:conditional_composition} provides a valid composition rule, the challenge in applying it lies in 
(a) choosing appropriate per-step distributions $P^{(n)}, Q^{(n)}$,
(b) finding a criterion for whether they are a valid dominating pair for the conditional distribution at each step $n$, 
(c) analyzing the probability of this criterion being violated.
We address these challenges in~\cref{section:bad_outcome_from_dominating_pairs}.
After that, we use these results to define an algorithm that automatically yields valid per-step dominating pairs for privacy accounting 
given a user-specified probability $\delta_E$ in~\cref{section:dominating_pairs_from_bad_outcome}.
All proofs can be found in~\cref{appendix:proofs_conditional_composition}.

\subsection{Non-dominance probability given per-step distributions}\label{section:bad_outcome_from_dominating_pairs}

\textbf{Per-step distributions.}
From the definition of $P, Q$ in~\cref{lem:dominating_pair} and the law of total probability, we know that for any previous outcome $\vy_{1:n-1} \in \sR^{n-1}$ the conditional density of step $n$ is 
\begin{equation}
    p(y_n \mid \vy_{1:n-1}) = \sum_{i=1}^b \mathcal{N}(\evy_n \mid \emm_{i,n}, \sigma) p(z=i \mid \vy_{1:n-1}),
    \quad
    q(y_n \mid \vy_{1:n-1}) = \mathcal{N}(\evy_n \mid 0, \sigma),
    \label{eq:per_step_distributions_posterior}
\end{equation}
with non-negative means $\evm_{i,n} \in \sR_+$ and 
$p(z=i) = \frac{1}{b}$, where $b$ is the number of batches per epoch.
Thus, a natural choice for our per-step distributions $P^{(n)}, Q^{(n)}$ that should be dominating is via 
\begin{equation}
    p^{(n)}(y_n) = \sum_{i=1}^b \mathcal{N}(\evy_n \mid \emm_{i,n}, \sigma) p^{(n)}(z=i),
    \quad
    q^{(n)}(y_n) = \mathcal{N}(\evy_n \mid 0, \sigma),
    \label{eq:per_step_distributions_fixed}
\end{equation}
with some appropriate choice of categorical mixture weights $p^{(n)}(z=i)$.
Note that $P^{(n)}$ and $Q^{(n)}$ in~\cref{eq:per_step_distributions_fixed} are fixed distributions. In particular, the mixture weights $p^{(n)}(z=i)$ do not 
depend on earlier outcomes $\vy_{1:n-1}$,
unlike the posterior mixture weights in~\cref{eq:per_step_distributions_posterior}.

\textbf{Dominance criterion.}
For any \emph{fixed} earlier outcomes $\vy_{1:n-1}$, \cref{eq:per_step_distributions_posterior} and~\cref{eq:per_step_distributions_fixed} contain Gaussian mixtures with identical means but different weights.
Intuitively, one pair dominates the other if it assigns more weight to larger mixture components, since these correspond
to participation patterns with more privacy leakage.
This intuition has been formally proven by~\citet{choquette2023privacy} in terms of stochastic dominance between priors.
However, the following derived criterion in terms of~\emph{reverse hazard functions} is more convenient to work with for reasons we shall see shortly:
\begin{restatable}{lemma}{reversehazardtodominance}
    \label{lemma:reverse_hazard_dominance_criterion}
    Assume non-decreasing mixture means $0 \leq m_{1,n} \leq \dots \leq m_{b,n}$ in~\cref{eq:per_step_distributions_posterior,eq:per_step_distributions_fixed}.
    If the posterior reverse hazard function is pointwise less or equal, i.e.,
    \begin{equation}\label{eq:dominance_criterion_reverse_hazard}
        P(z=i \mid z \leq i, \vy_{1:n-1}) \leq P^{(n)}(z=i \mid z \leq i ) \quad \text{ for all $1 \leq i \leq b$},
    \end{equation}
    then the distributions from~\cref{eq:per_step_distributions_fixed} dominate those from~\cref{eq:per_step_distributions_posterior}, i.e., 
    $(P_{y_n | y_{1:n-1}}, Q_{y_n | y_{1:n-1}}) \prec (P^{(n)}_{y_n}, Q^{(n)}_{y_n})$
    and $(Q_{y_n | y_{1:n-1}}, P_{y_n | y_{1:n-1}}) \prec (Q^{(n)}_{y_n},P^{(n)}_{y_n})$.
\end{restatable}
The benefit of this formulation is that it reduces the problem of testing for dominance to
a standard machine learning problem:
For each $i \in [b]$, we simply need to 
evaluate the probability of component $i$ in a Gaussian mixture model with components $1,\dots,i$.
From a privacy perspective, Bayes law shows that~\cref{eq:dominance_criterion_reverse_hazard} is equivalent to testing a tail bound on a Gaussian mixture privacy loss:
\begin{restatable}{lemma}{reversehazardloss}\label{lemma:reverse_hazard_loss}
    Let $s(x) = (1 + e^{-x})^{-1}$ be the logistic sigmoid function.
    For any $i \in [b]$,
    the conditional reverse hazard function is 
     $P(z =i \mid z \leq i, \vy_{1:n-1}) = s(- \log(i-1) - L_i(\vy_{1:n-1}))$ with
    \begin{equation}\label{eq:gaussian_mixture_privacy_loss}
        L_i(\vy_{1:n-1}) =
        \log
        \left(
            \frac{ p(\vy_{1:n-1} \mid z < i)}{p(\vy_{1:n-1} \mid z = i)}
        \right)
        =
        \log\left(
        \frac{\sum_{j=1}^{i-1} \mathcal{N}(\vy_{1:n-1} \mid \vm_{j,:n-1}, \sigma^2 \eye)  \frac{1}{i-1}}{\mathcal{N}(\vy_{1:n-1} \mid \vm_{i,:n-1}, \sigma^2 \eye)}
        \right).
    \end{equation}
\end{restatable}

\textbf{Non-dominance probability.}
Finally, we need to calculate the probability $\delta_E$
that our dominance criterion is violated (``bad'' outcome $\bigcup_{n=1}^N \overline{A^{(n)}}$)
under the randomness of the overall dominating pair from~\cref{lem:dominating_pair},
i.e., $\vy \sim P$ (``remove'' relation) or $\vy \sim Q$ (``add'' relation).
It follows from~\cref{lemma:reverse_hazard_loss}
that this corresponds to tail probabilities on Gaussian mixture privacy loss $L_i$:
\begin{theorem}\label{corollary:prob_of_dominance}
    Let $\vy \sim R$ with 
    $R = P$ (``remove'') or $R = Q$ (``add')
    as defined in~\cref{lem:dominating_pair}.
    Define the shorthand
    $\lambda_i = P^{(n)}(z =i \mid z \leq i)$
    for the reverse hazard function.
    Choose thresholds $\tau_i$ with 
    $\lambda_i = s(-\log(i-1) - \tau_i)$.
    Then, $\Pr[\overline{A^{(n)}}] \leq \sum_{i=1}^b \Pr_{\vy \sim R}[L_i(\vy_{1:n-1}) \leq \tau_i]$.
\end{theorem}

\textbf{Analytical bounds.} \cref{corollary:prob_of_dominance}
reduces our non-dominance probabilities to tail bounds on a ternary form of privacy loss
$\smash{L_{\tilde{P},\tilde{Q},\tilde{R}}(\vx) = \log( \dd \tilde{P} \mathbin{/} \dd \tilde{Q})(\vx)}$
where $\tilde{P},\tilde{Q},\tilde{R}$ are (mixtures of) Gaussians.
However, if we could evaluate them exactly, we could also directly evaluate the privacy profile $\delta(\epsilon)$ of our overall mechanism and would not need privacy accounting in the first place (see, e.g., Theorem 5 in~\cite{balle2018improving}).
Luckily, bounding the extreme tails (small $\delta_E$) of privacy loss distributions is exactly what R\'enyi accounting excels at.
We could thus adapt the methods from~\cref{section:rdp_accountant} to evaluate the R\'enyi divergence / MGF of 
$\smash{L_{\tilde{P},\tilde{Q},\tilde{R}}}$ and compute a Chernoff bound (see Theorem 2 of~\citet{abadi2016deep}).
Instead, we choose to directly apply the AM-GM inequality underlying~\cref{lem:remove-renyi-bound} to bypass the lossy Chernoff bound. 
This results in the following analytical tail bound (for ``remove'', see~\cref{theorem:tail_amgm_remove}).
\begin{restatable}{theorem}{tailamgmadd}\label{theorem:tail_amgm_add}
    Let $\vy \sim R$ with $R = Q$ (``add')
    as defined in~\cref{lem:dominating_pair}.
    Define standard normal CDF $\Phi : \sR \rightarrow [0,1]$
    and $\psi = \mathrm{Uniform}([i-1])$
    Then $\Pr[L_i(\vy_{1:n-1}) \leq \tau_i] \leq \Phi\left(\frac{\tau_i - \nu}{\xi}\right)$ with 
    \begin{equation*}
        \nu = (||\vm_{i,:n-1}||_2^2 - \mathbb{E}_{j \sim  \psi} [||\vm_{j,:n-1}||_2^2])] \mathbin{/} (2 \sigma^2), 
        \qquad
        \xi =   \lvert\lvert \vm_{i:n-1} - \mathbb{E}_{j \sim \psi}[\vm_{j,:n-1}] \rvert\rvert_2 \mathbin{/} \sigma.
    \end{equation*}
\end{restatable}

\begin{wrapfigure}{r}{0.5\textwidth} 
\vskip-0.65cm
\begin{minipage}{\linewidth}
\begin{algorithm}[H]
\caption{Conditional composition pairs}
\label{algorithm:conditional_composition}
\begin{algorithmic}
\REQUIRE Mixture means $\mm \in \sR_+^{b \times N}$, iteration $n$, batches per epoch $b$, significance $\beta$, relation $r$
    \STATE $\pi \gets \mathrm{argsort\_asc}(\mm_{:, n})$
    \STATE $\mmu_{1}, \dots, \mmu_{b} \gets \mm_{\pi(1), :n-1}, \dots, \mm_{\pi(b), :n-1}$
    \FOR{$i = b$ to $1$}
        \IF{$i = 1$}
            \STATE $\overline{\lambda}_i \gets 1$
        \ELSE
        \STATE $\tilde{P} \gets \sum_{j=1}^{i-1} \mathcal{N}(\mmu_j, \sigma^2 \eye) \cdot \frac{1}{i-1}$
        \STATE $\tilde{Q} \gets \mathcal{N}(\mmu_i, \sigma^2 \eye)$
        \IF{r = \texttt{REMOVE}}
            \STATE $\tilde{R} \gets \sum_{k=i}^{b} \mathcal{N}(\mmu_k, \sigma^2 \eye_{n-1}) \cdot \frac{1}{b}$
        \ELSE
            \STATE $\tilde{R} \gets \mathcal{N}(\vzero, \sigma^2 \eye_{n-1})$
        \ENDIF
        \STATE $L_{\tilde{P},\tilde{Q},\tilde{R}} \gets \log\left(\frac{\dd \tilde{P}}{\dd \tilde{Q}} (\vx)\right)$ with $\vx \sim \tilde{R}$
        \STATE{$\beta_i \gets \beta \mathbin{/} (b-1)$} \# Per-bound significance
        \STATE $\tau_i \gets \texttt{whp\_lower}(\mathrm{Law}(L_{\tilde{P},\tilde{Q},\tilde{R}}), \beta_i)$
        \STATE $\overline{\lambda}_i \gets \texttt{sigmoid}(- \log(i-1) - \tau_i)$
        \ENDIF
        \STATE $p_i \gets \lambda_i \cdot \prod_{j=i+1}^b (1 - \lambda_j)$
    \ENDFOR
    \IF{r = \texttt{REMOVE}}
        \STATE \textbf{return} $(\sum_{j=1}^b \mathcal{N}(\mu_j, \sigma) \cdot p_j), \mathcal{N}(0, \sigma)$
    \ELSE
        \STATE \textbf{return} $\mathcal{N}(0, \sigma), (\sum_{j=1}^b \mathcal{N}(\mu_j, \sigma) \cdot p_j)$
    \ENDIF
\end{algorithmic}
\end{algorithm}
  \end{minipage}
  \vskip-1.1cm
\end{wrapfigure}

\subsection{Per-step distributions given non-dominance probability}\label{section:dominating_pairs_from_bad_outcome}
Picking mixture weights for our per-step dominating pairs from~\cref{eq:per_step_distributions_fixed}
by hand is not practical. 
We thus propose~\cref{algorithm:conditional_composition},
which returns a dominating pair 
given iteration $n$ and significance $\beta = \Pr[\overline{A^{(n)}}]$.
Via union bound, $\beta = \delta_E \mathbin{/} N$ yields the desired non-dominance probability.
Using our analytical bounds from~\cref{theorem:tail_amgm_add,theorem:tail_amgm_remove}, the tails $\tau_i$ can be found via bisection.
\begin{restatable}{theorem}{condcompbothcorrect}
    For the ``remove'' relation, the distributions returned by~\cref{algorithm:conditional_composition} dominate $P_{\evy_n \mid \tilde{\vy}_{1:n-1}}$ and $Q_{\evy_n \mid \tilde{\vy}_{1:n-1}}$
    with probability $1 - \beta$ under $P$.
    For the ``add'' relation, the returned distributions dominate $Q_{\evy_n \mid \tilde{\vy}_{1:n-1}}$ and $P_{\evy_n \mid \tilde{\vy}_{1:n-1}}$
    with probability $1 - \beta$ under $Q$.
\end{restatable}

\textbf{Complexity and amortization.}
When using~\cref{theorem:tail_amgm_add,theorem:tail_amgm_remove} with a constant number of bisection steps to obtain tail bounds $\tau_i$, the
overall complexity of evaluating~\cref{algorithm:conditional_composition} for every iteration $n \in [N]$
is $\mathcal{O}(N^2 b^2)$ for ``remove'' 
and $\mathcal{O}(N^2 b)$ for ``add''.
However, most of the computational cost in finding the $\tau_i$ arises from terms that  
are linear in noise multiplier $\sigma \in \sR_+$.
Thus, subsequent evaluations for different $\sigma$ 
can be performed in $\mathcal{O}(N b^2)$ (``remove) or $\mathcal{O}(N b \log b)$ (``add'').
This can be used to significantly speed up mechanism calibration, where $\sigma$ is adjusted to attain desired privacy parameters $(\epsilon,\delta)$ 
(see~\cref{appendix:condcomp_runtime_proof} for proofs and~\cref{appendix:empirical_runtime} for empirical runtimes).

\textbf{Variational tail bounds.}
A finding that may be of independent interest 
is that our privacy loss tail bounds from~\cref{theorem:tail_amgm_add,theorem:tail_amgm_remove}
are an instance of a more general family of variational bounds (see proofs in~\cref{appendix:proofs_variational}).
That is, every categorical distribution $\psi$ supported on $\{1,\dots,i\}$
yields a different bound on the CDF of ternary Gaussian mixture privacy loss random variable $L_{\tilde{P},\tilde{Q},\tilde{R}}$ where $R$ has $i$ components.
As demonstrated in~\cref{appendix:choice_of_variational_family}, choosing the best bound from a variational family $\Psi$ of distributions
can enable tighter privacy accounting.

\begin{figure}[b!]
    \centering
    \begin{subfigure}[t]{0.31\linewidth}
        \centering
        \includegraphics[width=\linewidth]{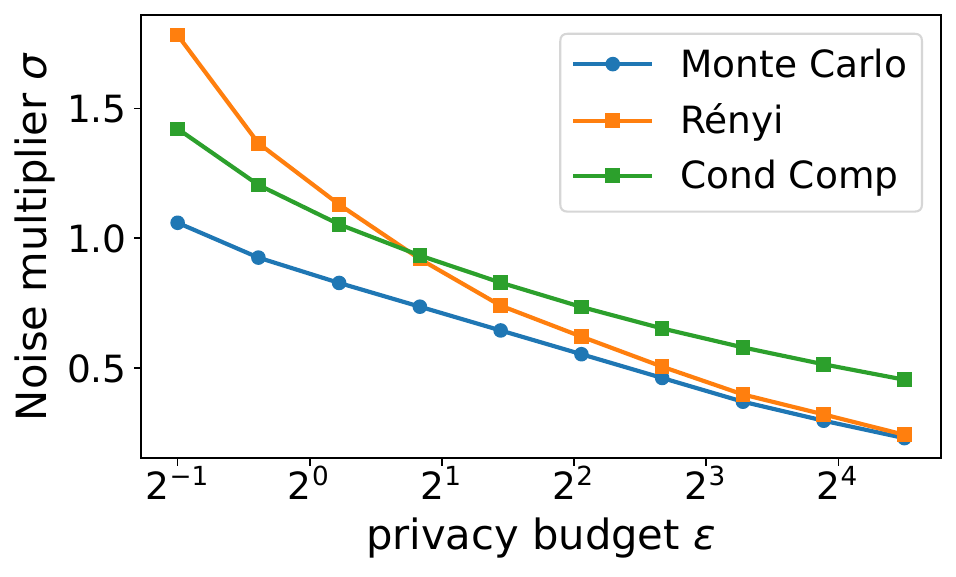}
        \subcaption{DP-SGD, $N\mathbin{=}100,\ k\mathbin{=}1$}
         \label{fig:rdp_mcmc_dpsgd_n_100_k_1}
    \end{subfigure}\hfill
    \begin{subfigure}[t]{0.31\linewidth}
        \centering
        \includegraphics[width=\linewidth]{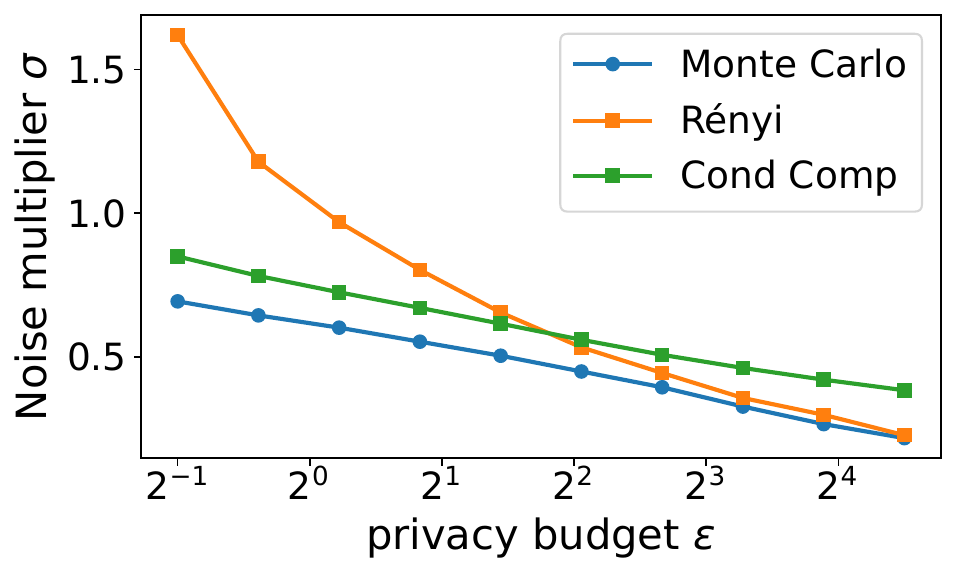}
        \subcaption{DP-SGD, $N\mathbin{=}1000,\ k\mathbin{=}1$}
        \label{fig:rdp_mcmc_dpsgd_n_1000_k_1}
    \end{subfigure}\hfill
    \begin{subfigure}[t]{0.31\linewidth}
        \centering
        \includegraphics[width=\linewidth]{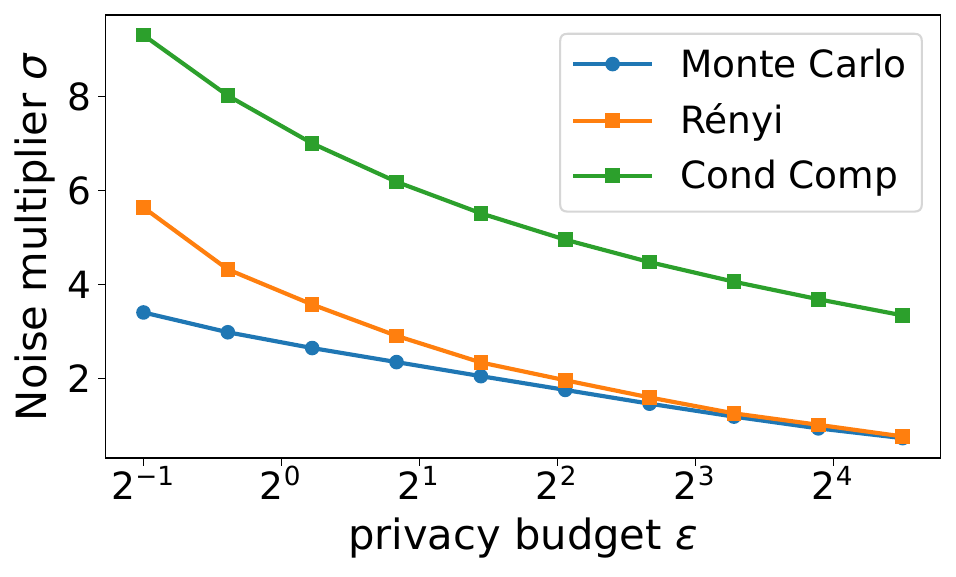}
        \subcaption{DP-SGD, $N\mathbin{=}1000,\ k\mathbin{=}10$}
        \label{fig:rdp_mcmc_dpsgd_n_1000_k_10}
    \end{subfigure}

    \begin{subfigure}[t]{0.31\linewidth}
        \centering
        \includegraphics[width=\linewidth]{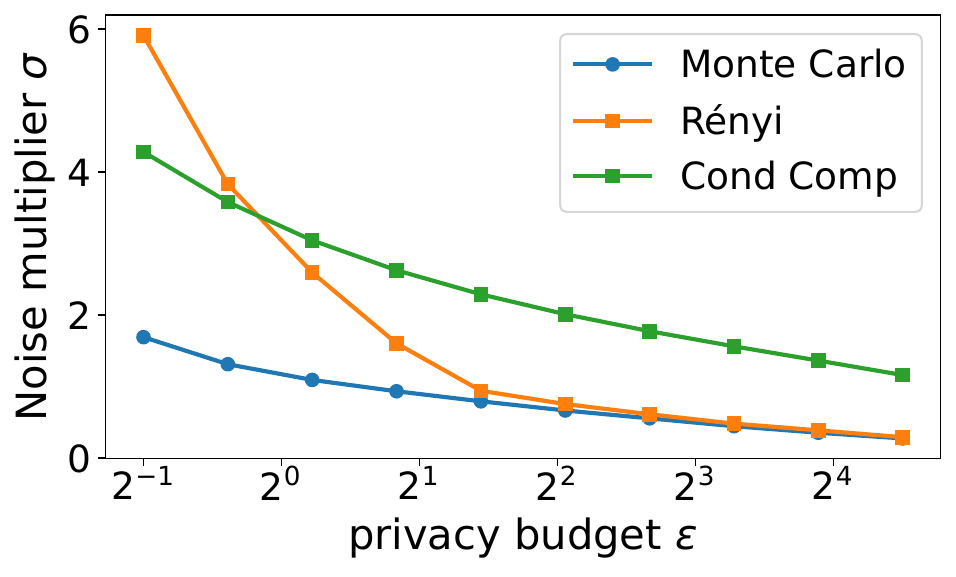}
        \subcaption{BSR, $N\mathbin{=}100,\ k\mathbin{=}1,\ p\mathbin{=}4$}
        \label{fig:rdp_mcmc_bsr_n_100_k_1_p_4}
    \end{subfigure}\hfill
    \begin{subfigure}[t]{0.31\linewidth}
        \centering
        \includegraphics[width=\linewidth]{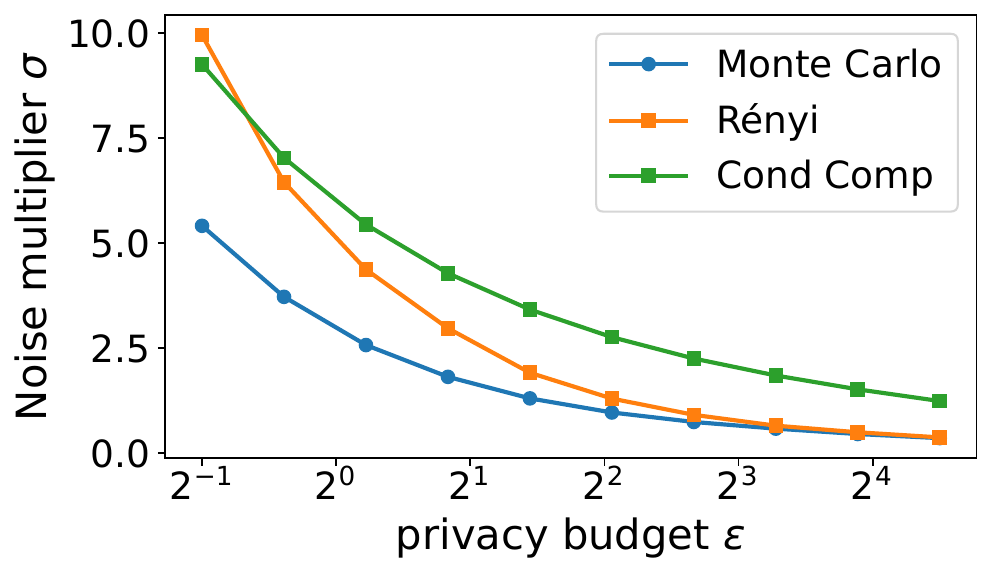}
        \subcaption{BSR, $N\mathbin{=}100,\ k\mathbin{=}1,\ p\mathbin{=}64$}
        \label{fig:rdp_mcmc_bsr_n_100_k_1_p_64}
    \end{subfigure}\hfill
    \begin{subfigure}[t]{0.31\linewidth}
        \centering
        \includegraphics[width=\linewidth]{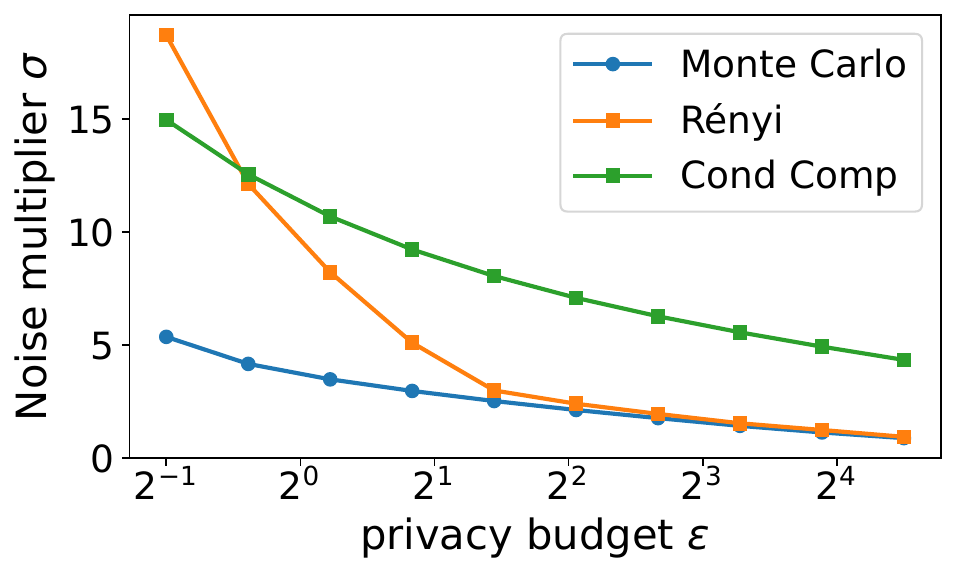}
        \subcaption{BSR, $N\mathbin{=}1000,\ k\mathbin{=}10,\ p\mathbin{=}4$}
        \label{fig:rdp_mcmc_bsr_n_1000_k_10_p_4}
    \end{subfigure}

    \caption{Comparison of the Rényi based and conditional composition based accountants with the Monte Carlo (MC) accountant for DP-SGD and Banded Square Root (BSR) at $\delta=10^{-5}$, with calibrated noise multipliers $\sigma$ shown as a function of the privacy budget $\varepsilon$. Each plot corresponds to a different choice of the number of iterations $n$, bandwidth $p$ and the number of epochs $k$.}
    \label{fig:rdp_mcmc_dpsgd}
    \vspace{-0.3cm}
\end{figure}

\section{Experiments}\label{section:experiments}

In this section, we evaluate the R\'enyi and conditional composition (Cond Comp) based accountants against the sampling based Monte Carlo accountant \citep{choquette2024near}.
We emphasize that privacy guarantees derived via a Monte Carlo accountant are qualitatively different from the $(\epsilon,\delta)$-DP guarantees of our proposed accountants: they hold only with high probability, or require the mechanism to abstain (see~\cref{section:background_matrix_mechanisms}).
In favor of the baseline, we show raw Monte Carlo estimates (for probabilistic upper bounds, see~\cref{appendix:monte_carlo_whp_bounds}).
We perform privacy accounting across different modern matrix mechanisms, including  
DP-SGD ($\mC = \eye$) and BSR \citep{kalinin2024banded} in~\cref{fig:rdp_mcmc_dpsgd},
as well as BISR, BandInvMF \citep{kalinin2025back}, BLT \citep{dvijotham2024efficient} and BandMF \citep{mckenna2024scaling}  in~\cref{fig:rdp_mcmc_bisr} of~\cref{appendix:extra_experiments}. 
While DP-SGD, BSR, and BISR are defined analytically,
BandMF, BandInvMF, and BLT define the entries of $\mC$ through an optimization problem (see details in~\cref{appendix:experimental_setup}).
With the R\'enyi accountant, we set bandwidth truncation $p=4$, which is exact for DP-SGD and BSR(4). For $\epsilon \leq 2$, we additionally evaluate $p=2$ to accommodate larger R'enyi orders (up to $\alpha=25$) and thereby tighten our bounds at comparable computational cost.
For consistency with the experiments in~\cite{choquette2024near}, we set $\delta=10^{-5}$
and plot the required noise multiplier $\sigma$ for different $\epsilon > 0$.
This matches the practical use of DP machine learning in which we 
first define a desired level $(\epsilon,\delta)$ of privacy and then calibrate the mechanism to attain it.
An implementation will be made available at \href{https://github.com/morganstanley/MSML/tree/main/papers/Matrix_Mechanism_Accounting}{github.com/morganstanley/MSML/tree/main/papers/Matrix\_Mechanism\_Accounting}.

Across all experiments, our R\'enyi accountant closely matches the Monte Carlo reference in the low-privacy regime (roughly $\epsilon \ge 2$--$4$), consistent with its strength in capturing large deviations. In the high-privacy regime ($\epsilon \le 1$), our conditional composition accountant yields tighter bounds and enables smaller calibrated noise multipliers, especially under stronger subsampling (larger $b$, smaller batch size). In multi-epoch settings, the same qualitative pattern holds for non-trivial matrix mechanisms (BSR/BISR), whereas for multi-epoch DP-SGD the conditional composition bound can degrade and provides little benefit over R\'enyi accounting (see~\cref{appendix:reverse_hazard_jumps} for a detailed explanation). Finally, varying correlation bandwidth (e.g., $p=4$ vs.\ $p=64$) does not change these conclusions: the observed low- vs.\ high-privacy behavior is robust to the strength of noise correlation.

\textbf{Model utility.} The previously cited works already evaluate model utility for different factorizations across various noise multipliers $\sigma$. Nevertheless, the utility with our specific $\sigma$ may be of interest.
We thus conduct experiments on CIFAR-10 in~\cref{appendix:cifar_experiments}. In particular, we find that our guarantees are strong enough to capture the utility benefit of using correlated instead of uncorrelated noise.

\textbf{DP-SGD specific bounds.} Our matrix mechanism accountant can also be applied to the special case of uncorrelated noise (DP-SGD).
For completeness, we compare to existing accountants for DP-SGD under random allocation in~\cref{appendix:feldman_comparison}.
The method from~\citet{feldman2025privacy} is slightly tighter, but qualitatively similar.
As we explain in more detail in~\cref{appendix:old_feldman_experiments}, our R\'enyi accountant is exact and our conditional composition accountant can be understood as a strict generalization of their ``Poisson decomposition''. The latter just uses more general bounds that are slightly less tight when restricted to the special case of DP-SGD.
In contrast, the PLD method from concurrent work of~\citet{feldman2026efficient} matches the Monte Carlo accountant almost exactly.
This makes its generalization to correlated noise mechanisms a promising but non-trivial direction for future work.

\textbf{Empirical runtime.} One motivation for our methods was efficiency, i.e., overcoming the $\mathcal{O}(\delta^{-1})$ dependence of Monte Carlo sampling.
In~\cref{appendix:empirical_runtime,appendix:renyi_runtime_vs_feldman}, we complement our asymptotic analyses by measuring runtime across parameter ranges.
The results confirm that different accountants are preferable in different regimes: Monte Carlo accounting is ill-suited for small $\delta$, and, for example, 
takes longer than our accountants for BSR with bandwidth $4$, $N=1000$ steps and $b=100$ batches per epoch when $\delta < 10^{-5}$. The conditional composition runtime is quadratic in $N$ and thus takes longer than Monte Carlo accounting for $\delta=10^{-7}$ when $N \geq 10^4$ on first execution. However, subsequent runs for different $\sigma$ are significantly cheaper and take less time than Monte Carlo accounting with $\delta=10^{-6}$ even at $N = 10^{4}$.
The R\'enyi accountant exhibits (potentially optimal) exponential runtime in bandwidth $p$, but can trade it off for tightness via truncation.
Finally, it is important to remember that, regardless of whether the Monte Carlo accountant is faster or slower in a given setting, its privacy guarantees are \emph{always} qualitatively weaker than our deterministic bounds.

\section{Conclusion}\label{section:conclusion}
The only available accountant for matrix mechanisms under balls-in-bins sampling relies on  Monte Carlo sampling \citep{choquette2024near}, which yields guarantees that hold only with high probability or require random abstention. In this work, we provide a way to perform accounting deterministically, that is, without relying on sampling. Our Rényi  accountant computes the guarantees efficiently via dynamic programming, improving the time complexity for DP-SGD under random allocation \citep{feldman2025privacy} from exponential to polynomial, and generalizing the analysis to subsampled matrix mechanisms. The resulting accountant offers tight privacy guarantees in the low privacy regime.
Our conditional composition accountant upper-bounds the privacy profile via
a sequence of independent Gaussian mixture mechanisms, which lets us apply tight accountants and offers better guarantees in the high privacy regime.

\textbf{Broader impact.} Our work is primarily targeted at mitigating negative societal impact by contributing towards improving model training with formal privacy guarantees. However, there are risks in misinterpretation or misrepresentation of these guarantees, which we discuss in  in~\cref{appendix:broader_impact}.

\textbf{Limitations and future work.}
There are inherent limitations to the Rényi accountant, as its complexity in the bandwidth is exponential, leading to a trade-off between runtime and tightness.
With a limited compute budget, one either has to truncate the bandwidth or the range of R\'enyi orders,
both of which will result in pessimistic bounds unless strategy $\mC$ is inherently banded and privacy budget $\epsilon$ is large.
Beyond that, the conversion from R\'enyi to hockey-stick divergences is inherently lossy.
With a more efficient implementation, the accountant could be improved; however, the exponential dependence on bandwidth is provably unavoidable.
An limitation of conditional composition is that it uses a coarse partition into ``good'' and ``bad'' outcomes, which may result in overly pessimistic dominating pairs. 
Future work should explore more fine-grained partitions.
However, note that the limit case of infinitely fine-grained partitions is just standard privacy loss distribution accounting, which appears challenging for reasons outlined in~\cref{section:problem_statement,appendix:pld_inapplicability}.
A limitation of our instantiation of conditional composition is that it relies on inherently conservative analytical bounds on posterior subsampling probabilities..
An avenue for improvement is to expand upon the variational bounds from~\cref{appendix:proofs_variational} through computational variational inference.
Alternatively, one could use a tighter privacy accountant to improve the tail bounds
on the ternary privacy loss underlying the posterior.
Overall, we expect future work to develop even tighter sampling-free accountants
to fully leverage the benefits of combining correlated subsampling and correlated noise
for more practical differentially private machine learning with strong privacy--utility trade-offs.

\clearpage

\section*{Acknowledgments}
We thank Arun Ganesh for insightful discussions and for sharing the code for the Monte Carlo accountant.
We are also grateful to Moshe Shenfeld for helping in evaluating their PLD accountant and providing feedback on our manuscript.
Nikita Kalinin’s research was funded in part by the Austrian Science Fund (FWF) [10.55776/COE12].

\bibliography{references}

@string{ML = "Machine Learning"}

@string{NeurIPS = "Conference on Neural Information Processing Systems (NeurIPS)"}

@string{ICML = "International Conference on Machine Learning (ICML)"}

@string{ICLR = "International Conference on Learning Representations (ICLR)"}

@string{AISTATS = "International Conference on Artificial Intelligence and Statistics (AISTATS)"}

@string{EMNLP = "Conference on Empirical Methods in Natural Language Processing (EMNLP)"}

@string{VLDB = "International Conference on Very Large Data Bases (VLDB)"}

@string{ML = "Machine Learning (ML)"}

@string{ICALP = "International Conference on Automata, Languages, and Programming"}

@inproceedings{alghamdi2023saddle,
  title={The saddle-point method in differential privacy},
  author={Alghamdi, Wael and Gomez, Juan Felipe and Asoodeh, Shahab and Calmon, Flavio and Kosut, Oliver and Sankar, Lalitha},
  booktitle=ICML,
  year={2023},
}

@article{li2015matrix,
    author ={Li, Chao and Miklau, Gerome and Hay, Michael and McGregor, Andrew and Rastogi, Vibhor},
    title ={The Matrix Mechanism: Optimizing Linear Counting Queries Under {D}ifferential {P}rivacy},
    journal =VLDB,
    year ={2015} 
}

@inproceedings{balle2018improving,
  title={Improving the {Gaussian} mechanism for differential privacy: Analytical calibration and optimal denoising},
  author={Balle, Borja and Wang, Yu-Xiang},
  booktitle=ICML,
  year={2018}
}

@inproceedings{abadi2016deep,
  title={Deep learning with differential privacy},
  author={Abadi, Martin and Chu, Andy and Goodfellow, Ian and McMahan, H Brendan and Mironov, Ilya and Talwar, Kunal and Zhang, Li},
  booktitle={Conference on Computer and Communications Security (CCS)},
  year={2016}
}

@inproceedings{dvijotham2024efficient,
  title={Efficient and near-optimal noise generation for streaming differential privacy},
  author={Dvijotham, Krishnamurthy Dj and McMahan, H Brendan and Pillutla, Krishna and Steinke, Thomas and Thakurta, Abhradeep},
  booktitle={Symposium on Foundations of Computer Science (FOCS)},  
  year={2024},
}

@inproceedings{kalinin2024banded,
  title={Banded square root matrix factorization for differentially private model training},
  author={Kalinin, Nikita P and Lampert, Christoph},
  booktitle=NeurIPS,
  year={2024}
}

@inproceedings{choquette2023multi,
  author    = {Christopher A. Choquette-Choo and Hugh Brendan McMahan and J. Keith Rush and Abhradeep Guha Thakurta},
  title     = {Multi epoch matrix factorization mechanisms for private machine learning},
  booktitle = ICML,
  year      = {2023},
}

@inproceedings{dwork2006differential,
  title={Differential privacy},
  author={Dwork, Cynthia},
  booktitle={International colloquium on automata, languages, and programming},
  year={2006},
}

@inproceedings{kalinin2025back,
      title={Back to Square Roots: An Optimal Bound on the Matrix Factorization Error for Multi-Epoch Differentially Private {SGD}}, 
      author={Nikita P. Kalinin and Ryan McKenna and Jalaj Upadhyay and Christoph H. Lampert},
      booktitle = ICLR,
      year={2026},
}

@inproceedings{mckenna2024scaling,
  title={Scaling up the Banded Matrix Factorization Mechanism for Differentially Private {ML}},
  author={McKenna, Ryan},
  booktitle=ICLR,
  year={2025}
}

@inproceedings{mcmahan2024hassle,
  title={A Hassle-free Algorithm for Strong Differential Privacy in Federated Learning Systems},
  author={McMahan, Hugh Brendan and Xu, Zheng and Zhang, Yanxiang},
  booktitle=EMNLP,
  year={2024}
}

@inproceedings{denisov2022improved,
  title={Improved differential privacy for {SGD} via optimal private linear operators on adaptive streams},
  author={Denisov, Sergey and McMahan, H Brendan and Rush, John and Smith, Adam and Guha Thakurta, Abhradeep},
  booktitle=NeurIPS,
  year={2022}
}

@inproceedings{
choquette-choo2023amplified,
title={(Amplified) Banded Matrix Factorization: A unified approach to private training},
author={Christopher A. Choquette-Choo and Arun Ganesh and Ryan McKenna and Hugh Brendan McMahan and J Keith Rush and Abhradeep Guha Thakurta and Zheng Xu},
booktitle={Thirty-seventh Conference on Neural Information Processing Systems},
year={2023},
}

@article{ponomareva2023dp,
  title={How to {DP-fy ML}: A practical guide to machine learning with differential privacy},
  author={Ponomareva, Natalia and Hazimeh, Hussein and Kurakin, Alex and Xu, Zheng and Denison, Carson and McMahan, H Brendan and Vassilvitskii, Sergei and Chien, Steve and Thakurta, Abhradeep Guha},
  journal={Journal of Artificial Intelligence Research},
  year={2023}
}

@inproceedings{choquette2024near,
  title={Near exact privacy amplification for matrix mechanisms},
  author={Choquette-Choo, Christopher A and Ganesh, Arun and Haque, Saminul and Steinke, Thomas and Thakurta, Abhradeep},
  booktitle=ICLR,
  year={2025}
}

@inproceedings{choquette2023privacy,
  title={Privacy amplification for matrix mechanisms},
  author={Choquette-Choo, Christopher A and Ganesh, Arun and Steinke, Thomas and Thakurta, Abhradeep},
  booktitle=ICLR,
  year={2024}
}

@inproceedings{feldman2025privacy,
  title={Privacy amplification by random allocation},
  author={Feldman, Vitaly and Shenfeld, Moshe},
  booktitle=NeurIPS,
  year={2025}
}

@inproceedings{dong2025leveraging,
  title={Leveraging randomness in model and data partitioning for privacy amplification},
  author={Dong, Andy and Chen, Wei-Ning and Ozgur, Ayfer},
  booktitle=ICML,
  year={2025}
}

@InProceedings{zhu2022optimal,
  title = {Optimal Accounting of Differential Privacy via Characteristic Function},
  author = {Zhu, Yuqing and Dong, Jinshuo and Wang, Yu-Xiang},
  booktitle =AISTATS,
  year = 	 {2022},
}

@inproceedings{mironov2017renyi,
  title={{R{\'e}nyi} differential privacy},
  author={Mironov, Ilya},
  booktitle={Computer Security Foundations Symposium (CSF)},
  year={2017},
}

@inproceedings{koskela2020computing,
  title={Computing tight differential privacy guarantees using {FFT}},
  author={Koskela, Antti and J{\"a}lk{\"o}, Joonas and Honkela, Antti},
  booktitle=AISTATS,
  year={2020},
}

@inproceedings{lebeda2025avoiding,
  title={Avoiding pitfalls for privacy accounting of subsampled mechanisms under composition},
  author={Lebeda, Christian Janos and Regehr, Matthew and Kamath, Gautam and Steinke, Thomas},
  booktitle={Conference on Secure and Trustworthy Machine Learning (SaTML)},
  year={2025},
}

@inproceedings{beltran2024towards,
  title={Efficient and Scalable Implementation of Differentially Private Deep Learning without Shortcuts}, 
  author={Beltran, Sebastian Rodriguez and Tobaben, Marlon and J{\"a}lk{\"o}, Joonas and Loppi, Niki and Honkela, Antti},
  booktitle={Conference on Secure and Trustworthy Machine Learning (SaTML)},
  year={2026}
}

@inproceedings{chua2024balls,
  title={Balls-and-bins sampling for {DP-SGD}},
  author={Chua, Lynn and Ghazi, Badih and Harrison, Charlie and Leeman, Ethan and Kamath, Pritish and Kumar, Ravi and Manurangsi, Pasin and Sinha, Amer and Zhang, Chiyuan},
  booktitle=AISTATS,
  year={2025}
}

@inproceedings{wang2019subsampled,
  title={Subsampled r{\'e}nyi differential privacy and analytical moments accountant},
  author={Wang, Yu-Xiang and Balle, Borja and Kasiviswanathan, Shiva Prasad},
  booktitle=AISTATS,
  year={2019},
}

@inproceedings{sommer2019privacyloss,
  author    = {Sommer, David and Meiser, Sebastian and Mohammadi, Esfandiar},
  title     = {Privacy loss classes: The central limit theorem in differential privacy},
  booktitle = {Privacy Enhancing Technologies Symposium (PETS)},
  year      = {2019},
}

@inproceedings{erlingsson2019amplification,
author = {Erlingsson, {\'U}lfar and Feldman, Vitaly and Mironov, Ilya and Raghunathan, Ananth and Talwar, Kunal and Thakurta, Abhradeep},
title = {Amplification by shuffling: from local to central differential privacy via anonymity},
booktitle = {Symposium on Discrete Algorithms},
year = {2019},
}

@inproceedings{canonne2020discrete,
  title     = {The Discrete {G}aussian for Differential Privacy},
  author    = {Canonne, Cl{\'e}ment L. and Kamath, Gautam and Steinke, Thomas},
  booktitle = NeurIPS,
  year      = {2020},
}

@inproceedings{barthe2013beyond,
  title={Beyond differential privacy: Composition theorems and relational logic for f-divergences between probabilistic programs},
  author={Barthe, Gilles and Olmedo, Federico},
  booktitle=ICALP,
  year={2013},
}

@article{wang2023randomized,
  title={A randomized approach to tight privacy accounting},
  author={Wang, Jiachen Tianhao and Mahloujifar, Saeed and Wu, Tong and Jia, Ruoxi and Mittal, Prateek},
  journal= NeurIPS,
  year={2023}
}

@article{blei2017variational,
author = {David M. Blei and Alp Kucukelbir and Jon D. McAuliffe},
title = {Variational Inference: A Review for Statisticians},
journal = {Journal of the American Statistical Association},
volume = {112},
number = {518},
pages = {859--877},
year = {2017},
publisher = {Taylor \& Francis},
}

@misc{pillutla2025correlated,
  title={Correlated Noise Mechanisms for Differentially Private Learning},
  author={Pillutla, Krishna and Upadhyay, Jalaj and Choquette-Choo, Christopher A and Dvijotham, Krishnamurthy and Ganesh, Arun and Henzinger, Monika and Katz, Jonathan and McKenna, Ryan and McMahan, H Brendan and Rush, Keith and others},
  note={arXiv preprint arXiv:2506.08201},
  year={2025}
}

@software{jaxprivacy2022github,
 author = {Balle, Borja and Berrada, Leonard and Charles, Zachary and
Choquette-Choo, Christopher A and De, Soham and Doroshenko, Vadym and Dvijotham,
Dj and Galen, Andrew and Ganesh, Arun and Ghalebikesabi, Sahra and Hayes, Jamie
and Kairouz, Peter and McKenna, Ryan and McMahan, Brendan and Pappu, Aneesh and
Ponomareva, Natalia and Pravilov, Mikhail and Rush, Keith and Smith, Samuel L
and Stanforth, Robert},
  title = {{JAX}-{P}rivacy: Algorithms for Privacy-Preserving Machine Learning in {JAX}},
  url = {http://github.com/google-deepmind/jax_privacy},
  version = {0.4.0},
  year = {2025},
}

@article{doroshenko2022connect,
  title={Connect the Dots: Tighter Discrete Approximations of Privacy Loss Distributions},
  author={Doroshenko, Vadym and Ghazi, Badih and Kamath, Pritish and Kumar, Ravi and Manurangsi, Pasin},
  journal={Privacy Enhancing Technologies Symposium (PETS)},
  year={2022}
}

@misc{dpaccountinglibrary,
    author        = {{Google DP Team}},
      title         = {Privacy Loss Distributions.},
      year          = {2025},
      howpublished={\url{https://raw.githubusercontent.com/google/differential-privacy/main/common_docs/Privacy_Loss_Distributions.pdf}},
      note = "Accessed January 29, 2026"
}

@inproceedings{li2012sampling,
  title={On sampling, anonymization, and differential privacy or, k-anonymization meets differential privacy},
  author={Li, Ninghui and Qardaji, Wahbeh and Su, Dong},
  booktitle={Symposium on Information, Computer and Communications Security},
  year={2012}
}

@inproceedings{feldman2026efficient,
  title={Efficient privacy loss accounting for subsampling and random allocation},
  author={Feldman, Vitaly and Shenfeld, Moshe},
  booktitle=ICML,
  year={2026},
}

@misc{krizhevsky2009learning,
  title={Learning multiple layers of features from tiny images},
  author={Krizhevsky, Alex and Hinton, Geoffrey and others},
  year={2009},
}

@article{namatevs2025privacy,
  title={Privacy auditing in differential private machine learning: The current trends},
  author={Namatevs, Ivars and Sudars, Kaspars and Nikulins, Arturs and Ozols, Kaspars},
  journal={Applied Sciences},
  volume={15},
  number={2},
  pages={647},
  year={2025},
  publisher={MDPI}
}
\bibliographystyle{icml2026}

\newpage
\appendix
\onecolumn

\section{Additional Experimental Results}\label{appendix:extra_experiments}

\subsection{Additional matrix mechanisms}
\begin{figure}[h!]
    \centering

    \begin{subfigure}[t]{0.31\linewidth}
        \centering
        \includegraphics[width=\linewidth]{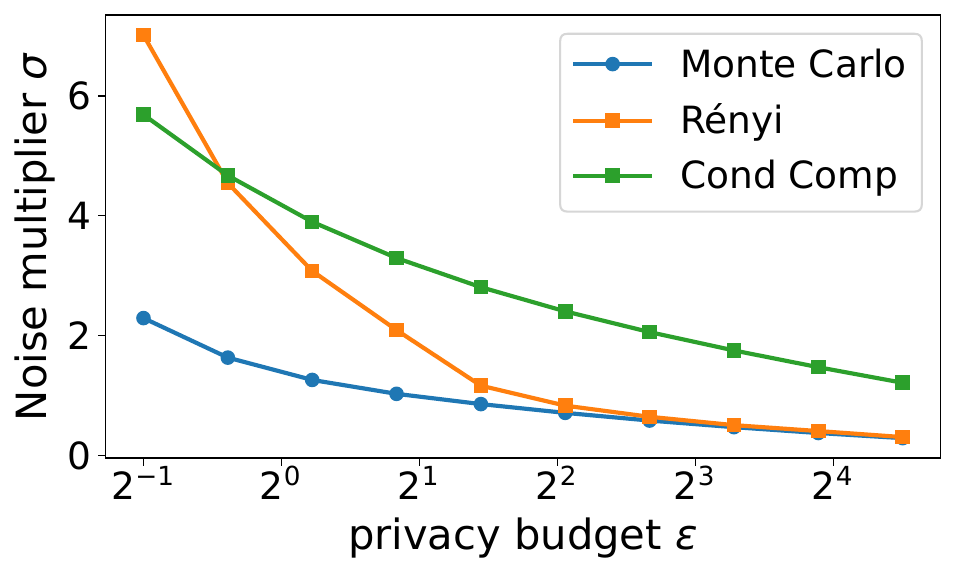}
        \subcaption{BISR, $N\mathbin{=}100,\ k\mathbin{=}1,\ p\mathbin{=}4$}
        \label{fig:rdp_mcmc_bisr_k_1_p_4}
    \end{subfigure}\hfill
    \begin{subfigure}[t]{0.31\linewidth}
        \centering
        \includegraphics[width=\linewidth]{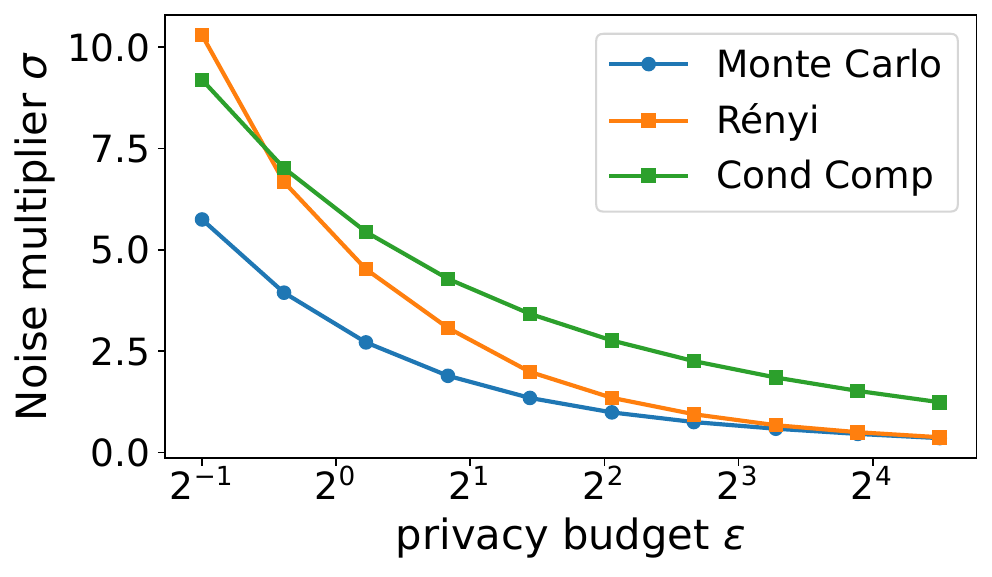}
        \subcaption{BISR, $N\mathbin{=}100,\ k\mathbin{=}1,\ p\mathbin{=}64$}
        \label{fig:rdp_mcmc_bisr_k_1_p_64}
    \end{subfigure}\hfill
    \begin{subfigure}[t]{0.31\linewidth}
        \centering
        \includegraphics[width=\linewidth]{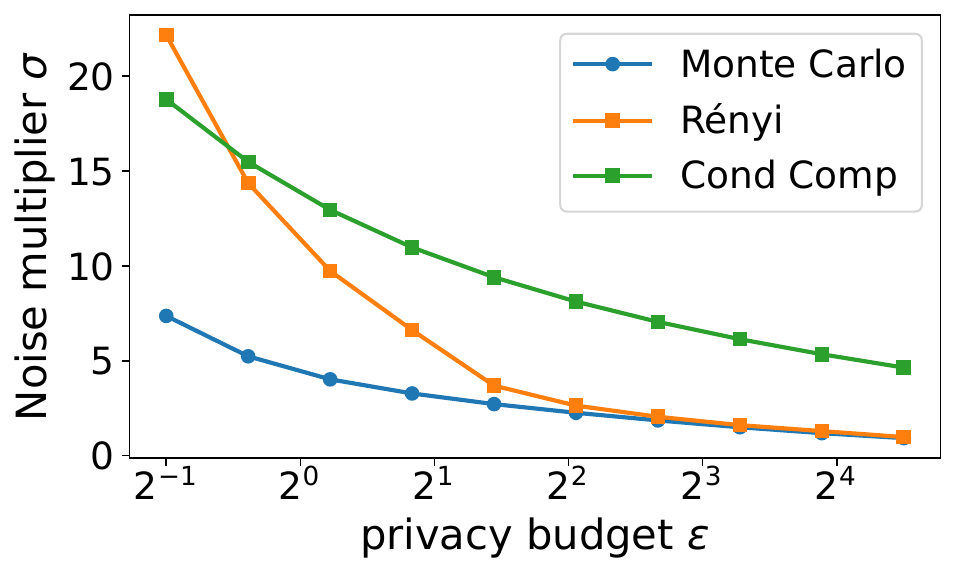}
        \subcaption{BISR, $N\mathbin{=}1000,\ k\mathbin{=}10,\ p\mathbin{=}4$}
        \label{fig:rdp_mcmc_bisr_k_10_p_4}
    \end{subfigure}
    
    \begin{subfigure}[t]{0.31\linewidth}
        \centering
        \includegraphics[width=\linewidth]{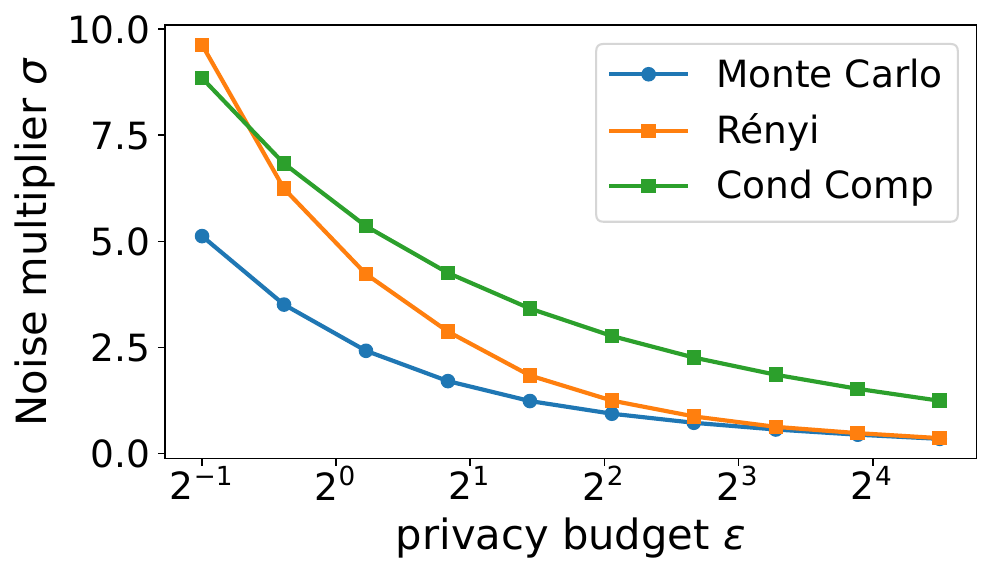}
        \subcaption{BLT, $N\mathbin{=}100,\ k\mathbin{=}1$}
    \end{subfigure}\hfill
    \begin{subfigure}[t]{0.31\linewidth}
        \centering
        \includegraphics[width=\linewidth]{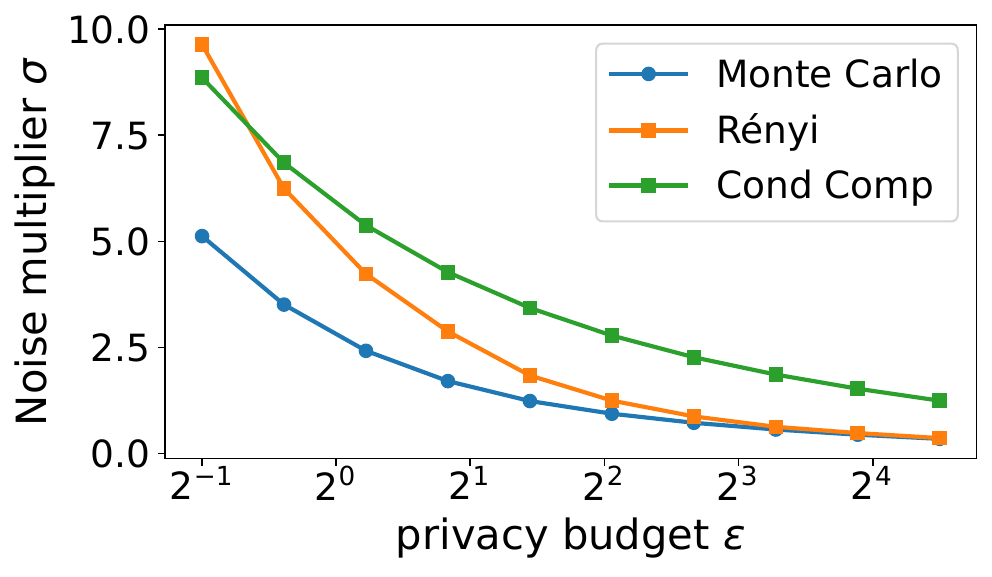}
        \subcaption{BandMF, $N\mathbin{=}100,\ k\mathbin{=}1,\ p\mathbin{=}64$}
    \end{subfigure}\hfill
    \begin{subfigure}[t]{0.31\linewidth}
        \centering
        \includegraphics[width=\linewidth]{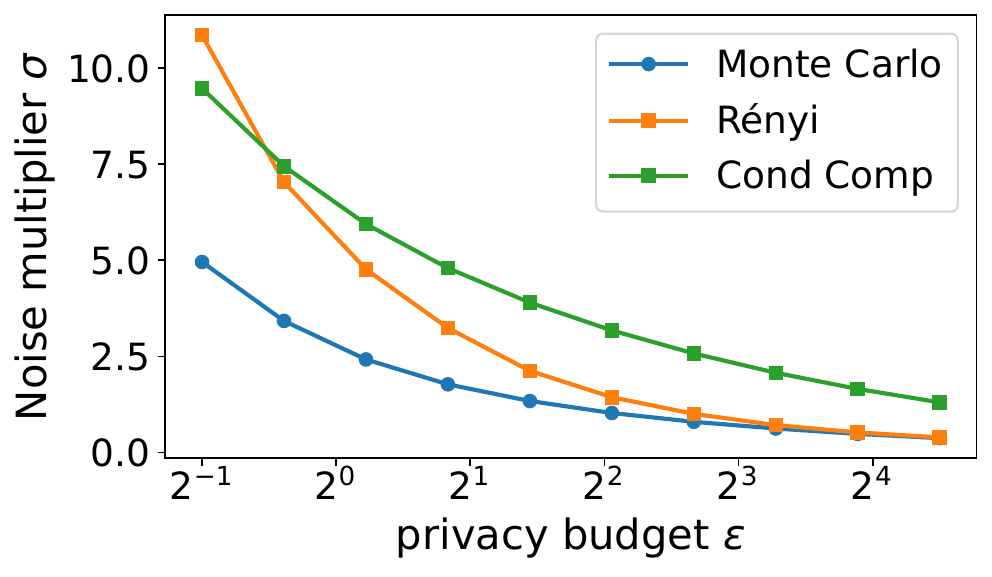}
        \subcaption{BandInvMF, $N\mathbin{=}100,\ k\mathbin{=}1,\ p\mathbin{=}4$}
    \end{subfigure}

    \caption{Comparison of the RDP based and conditional composition based accountants with the MC accountant for BISR, BandInvMF, BLT and BandMF at $\delta=10^{-5}$, with calibrated noise multipliers $\sigma$ shown as a function of the privacy budget $\varepsilon$. Each plot corresponds to a different choice of the number of iterations $n$, the number of epochs $k$, and parameter $p$.}
    \label{fig:rdp_mcmc_bisr}
\end{figure}

\subsection{Monte Carlo high-probability 
bounds.}\label{appendix:monte_carlo_whp_bounds}

\begin{figure*}[h!]
    \centering
    \begin{subfigure}[t]{0.31\linewidth}
        \centering
        \includegraphics[width=\linewidth]{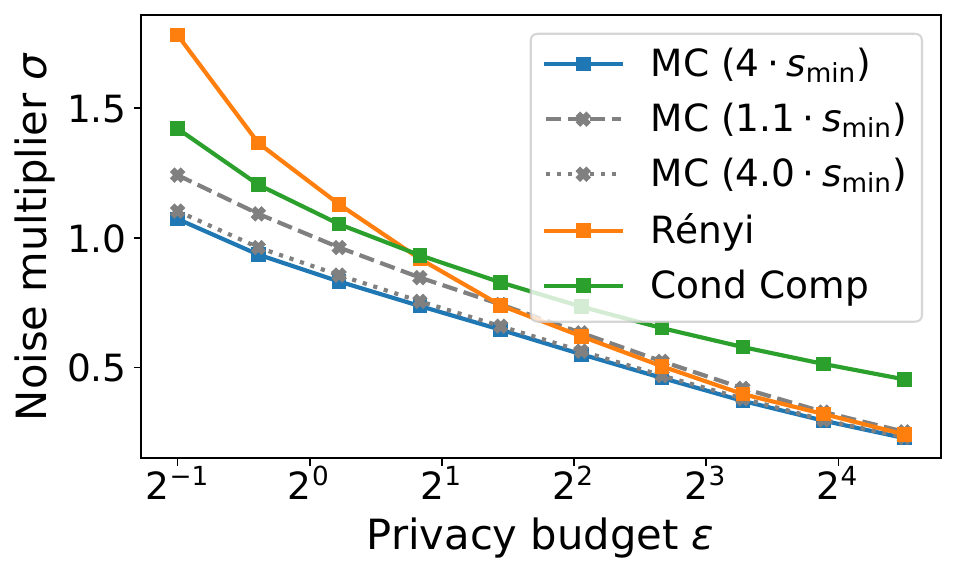}
        \subcaption{DP-SGD, $N\mathbin{=}100,\ k\mathbin{=}1$}
    \end{subfigure}\hfill
    \begin{subfigure}[t]{0.31\linewidth}
        \centering
        \includegraphics[width=\linewidth]{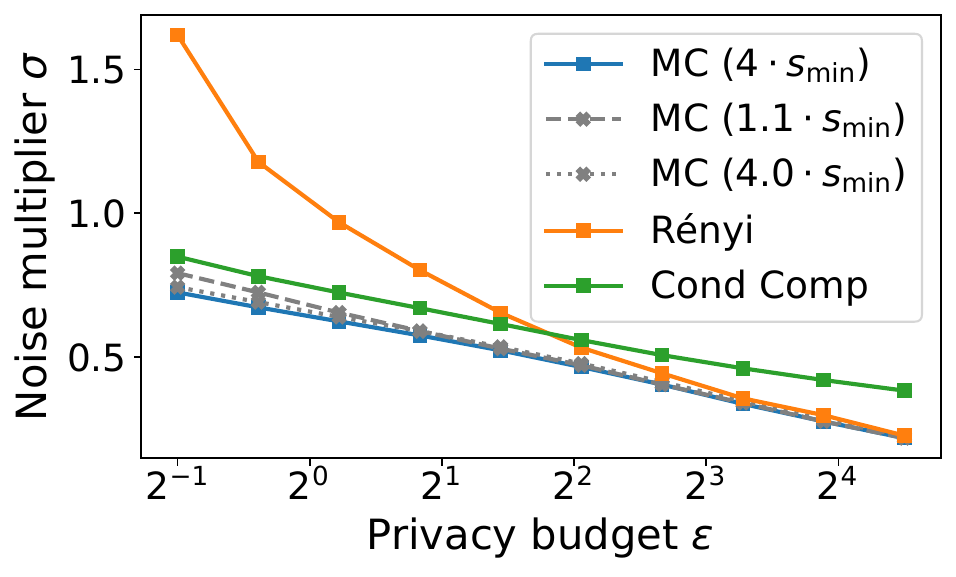}
        \subcaption{DP-SGD, $N\mathbin{=}1000,\ k\mathbin{=}1$}
    \end{subfigure}\hfill
    \begin{subfigure}[t]{0.31\linewidth}
        \centering
        \includegraphics[width=\linewidth]{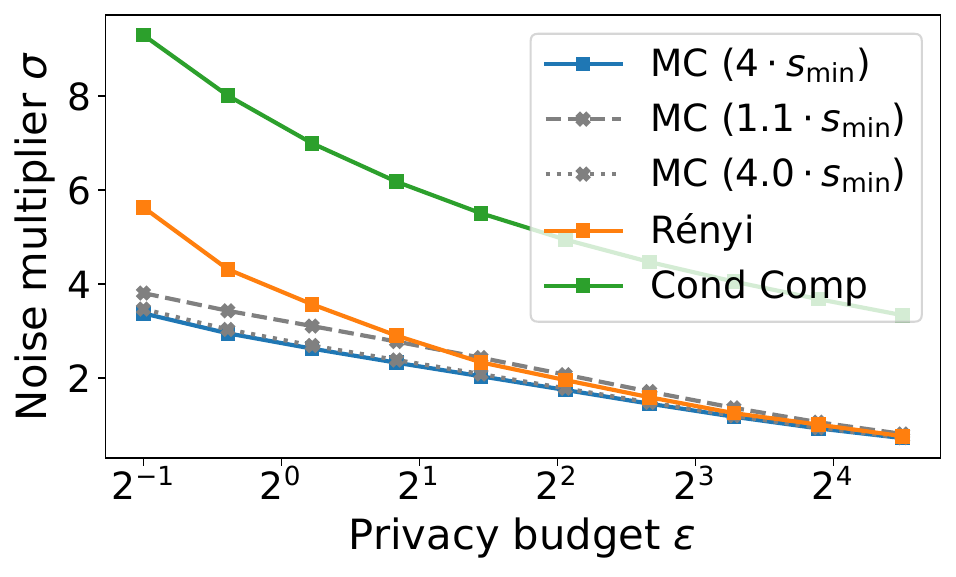}
        \subcaption{DP-SGD, $N\mathbin{=}1000,\ k\mathbin{=}10$}
    \end{subfigure}

    \begin{subfigure}[t]{0.31\linewidth}
        \centering
        \includegraphics[width=\linewidth]{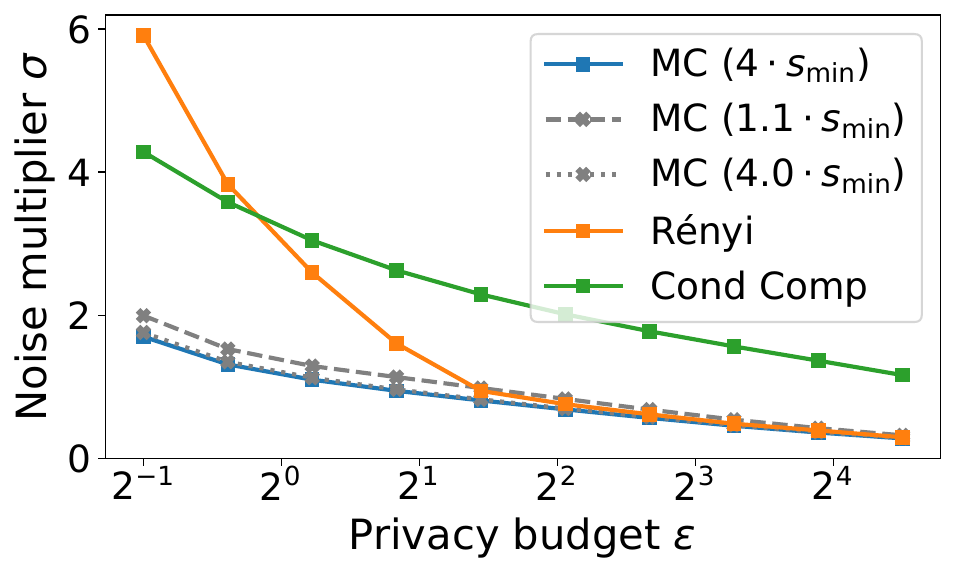}
        \subcaption{BSR, $N\mathbin{=}100,\ k\mathbin{=}1,\ p\mathbin{=}4$}
    \end{subfigure}\hfill
    \begin{subfigure}[t]{0.31\linewidth}
        \centering
        \includegraphics[width=\linewidth]{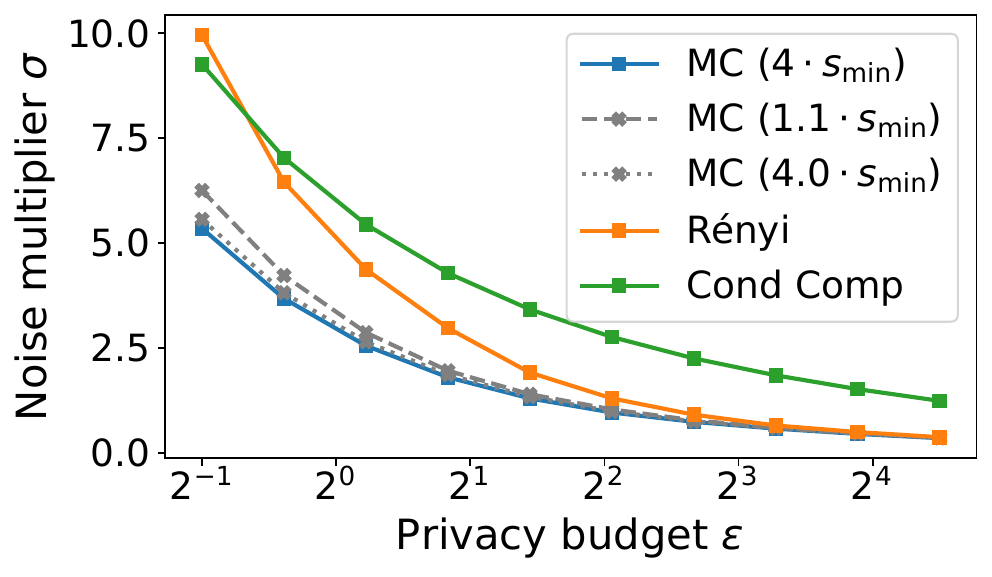}
        \subcaption{BSR, $N\mathbin{=}100,\ k\mathbin{=}1,\ p\mathbin{=}64$}
    \end{subfigure}\hfill
    \begin{subfigure}[t]{0.31\linewidth}
        \centering
        \includegraphics[width=\linewidth]{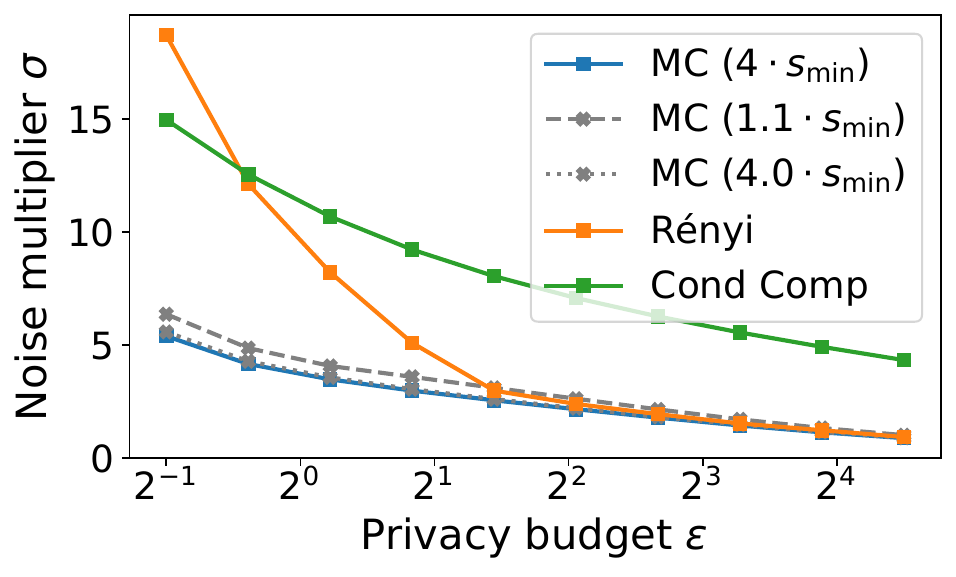}
        \subcaption{BSR, $N\mathbin{=}1000,\ k\mathbin{=}10,\ p\mathbin{=}4$}
    \end{subfigure}

    \caption{Main experimental~\cref{fig:rdp_mcmc_dpsgd}, re-drawn to evaluate both the raw Monte Carlo estimate (blue) in addition to high probability upper bounds (gray) for varying number of samples.
    }
    \vspace{-0.3cm}
    \label{fig:appendix_mc_whp_bounds}
\end{figure*}

In our main experiments, we directly used the Monte Carlo estimates of the privacy profile for calibration.
In practice, we instead have to compute probabilistic upper bounds, which we show in~\cref{fig:appendix_mc_whp_bounds}.
The number of samples is set as a multiple of the minimum number of samples $s_\mathrm{min}$ needed to obtain a bound that holds with high probability $1 - p$. Here,  we choose $p = \delta = 10^{-5}$, which is needed to construct an abstention-based mechanism (see Theorem 2.1 in~\cite{choquette2024near}).
As can be seen, the high probability bound is almost identical to the raw sample estimate when using enough samples.
In other words, the limitation is not the tightness of the tail bounds, but (1) that the resulting guarantees are qualitatively weaker than deterministic, sampling-free privacy parameters and (2) the cost of sampling (see runtime experiments in~\cref{fig:runtime_delta}).

\subsection{Model utility on CIFAR-10}\label{appendix:cifar_experiments}
\begin{wrapfigure}{r}{0.5\textwidth}
    \centering
    \includegraphics[width=0.4\textwidth]{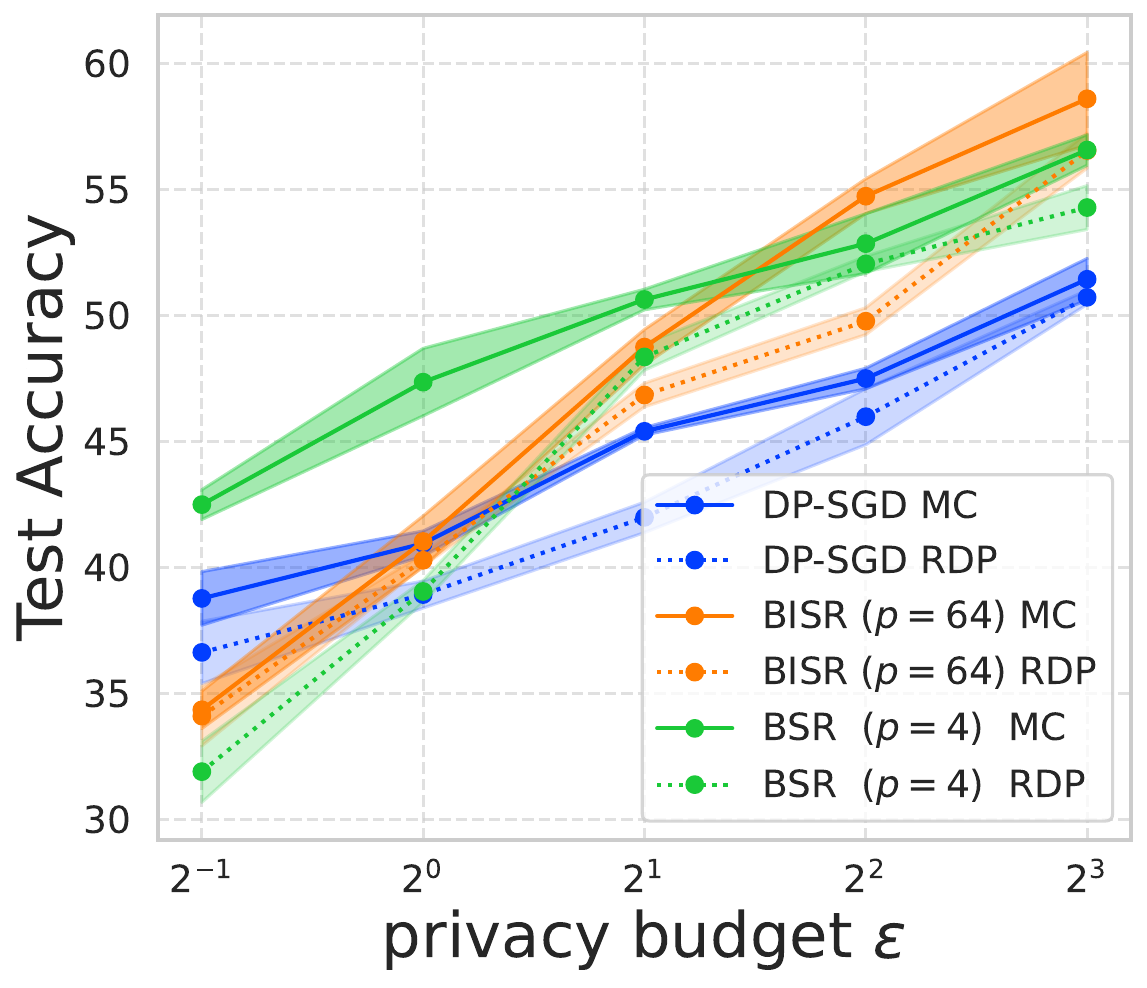}
    \caption{ConvNet model utility on CIFAR-10}
    \label{fig:cifar_experiment}
    \vskip1cm
\end{wrapfigure}
In~\cref{fig:cifar_experiment} we compare the model utility attained via our noise multipliers to those attained using the Monte Carlo accounting noise multipliers. We replicate the CIFAR-10 experimental setup from~\cite{kalinin2024banded}. We use $5$ random seeds and show mean and standard deviation of test accuracy.
Training is performed for $N=970$ steps divided into $k=10$ epochs, which corresponds to a batch size of $512$. Gradients are clipped to an $\ell_2$-norm of $8$ and we calibrate noise multipliers to $\delta=10^{-5}$ (for further details on the experimental setup, see~\cref{appendix:experimental_setup}.

As can be seen, there is a good match in utility for $\epsilon \geq 4$, which is consistent with our observations about the computed noise multipliers $\sigma$.
Furthermore, our sampling-free guarantees are strong enough to retain the benefit of using correlated (BSR, BISR) instead of uncorrelated noise (DP-SGD).

\subsection{Empirical runtime}\label{appendix:empirical_runtime}

\begin{figure}[h!]
    \centering
    \begin{subfigure}[t]{0.48\linewidth}
    \centering
    \includegraphics[width=1.0\linewidth]{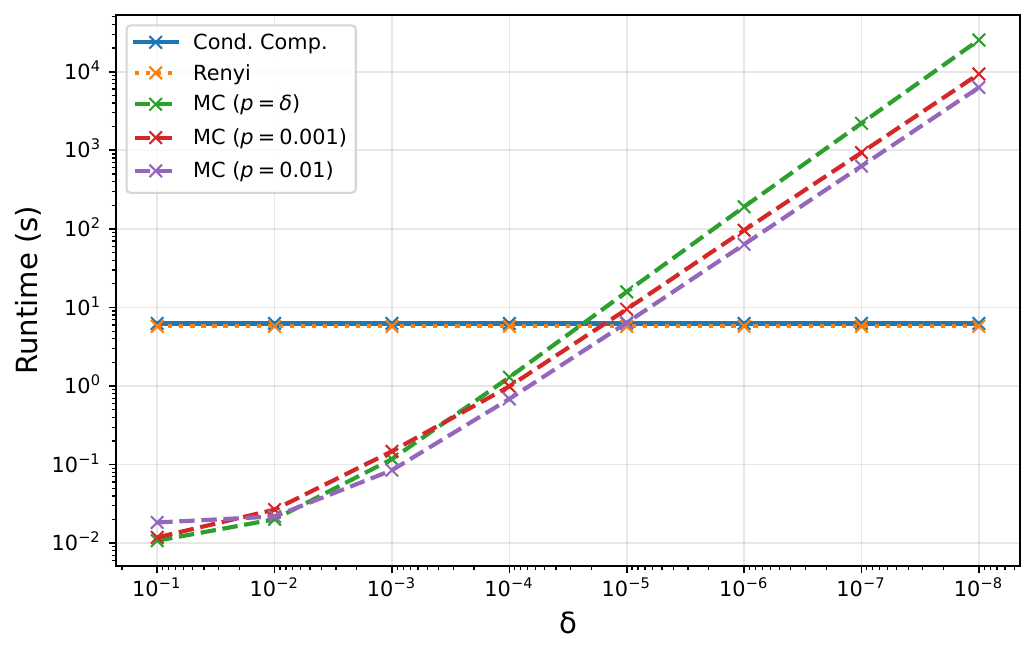}
    \caption{Varying privacy budget $\delta$}
    \label{fig:runtime_delta}
    \end{subfigure}
    \hfill
    \begin{subfigure}[t]{0.48\linewidth}
    \centering
    \includegraphics[width=1.0\linewidth]{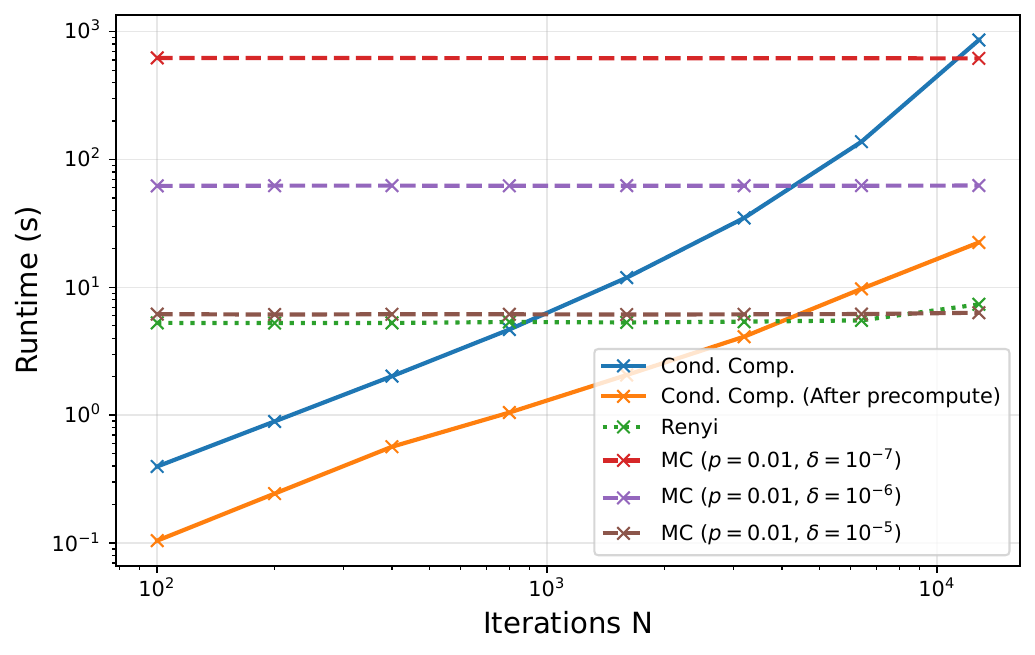}
    \caption{Varying number of iterations $N$}
    \label{fig:runtime_N}
    \end{subfigure}
    \\
    \begin{subfigure}[t]{0.48\linewidth}
    \centering
    \includegraphics[width=1.0\linewidth]{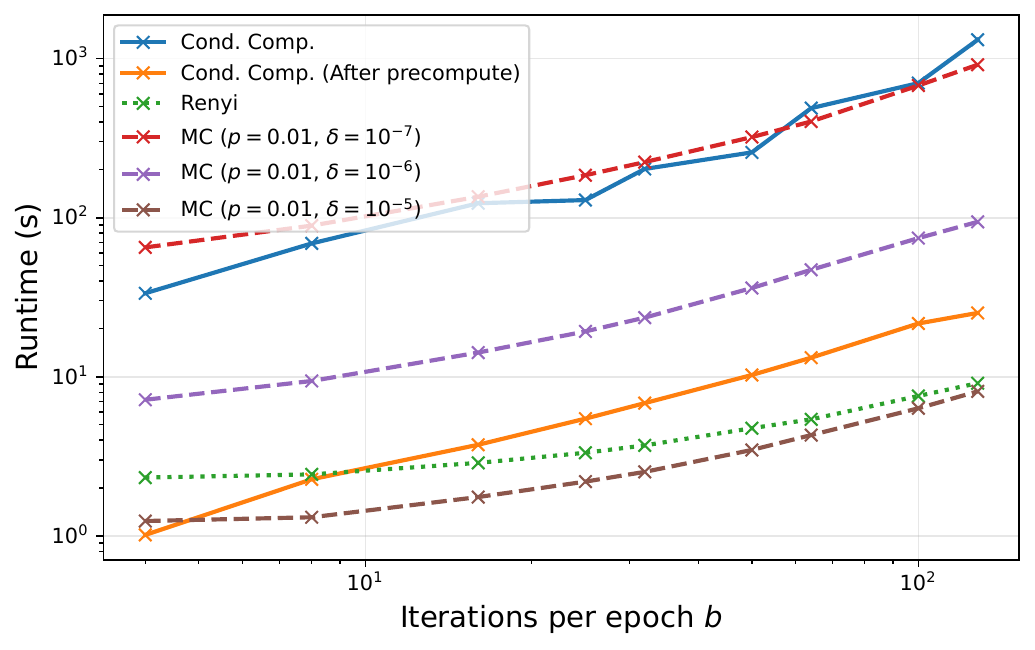}
    \caption{Varying batches per epoch $b$}
    \label{fig:runtime_b}
    \end{subfigure}
    \hfill
    \begin{subfigure}[t]{0.48\linewidth}
    \centering
    \includegraphics[width=1.0\linewidth]{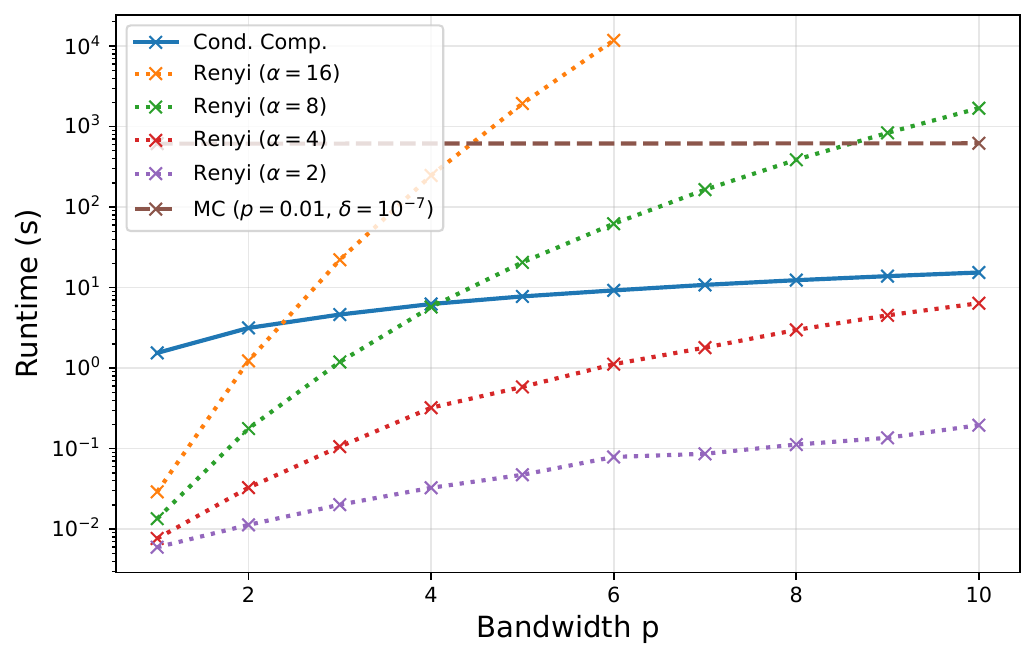}
    \caption{Varying bandwidth $p$}
    \label{fig:runtime_p}
    \end{subfigure}
    \caption{Wall-clock runtime under varying matrix mechanism and privacy accounting parameters.
    All observations are consistent with the claimed asymptotic runtimes.}
\end{figure}

In the main text, we already analyzed the asymptotic complexity of our proposed methods, as well the Monte Carlo baseline.
Here, we complement this analysis by measuring wall-clock runtime while varying all involved parameters,
namely privacy budget $\delta$, number of iterations $N$, batches per epoch $b$ and bandwidth $p$.
All experiments were performed on the same Intel Xeon Gold 6148 CPU (\SI{2.4}{\giga\hertz}, \SI{768}{\giga\hertz} RAM).
We measure the cost for a single run of each algorithm (or: run for all training steps in the case of conditional composition), i.e., we do not perform repeated runs to calibrate noise multiplier $\sigma$.

\textbf{Default parameters.}
While varying the previously mentioned parameters, we use the following default values for our accountants, which apply unless explicitly stated otherwise.
By default, we consider privacy parameters $\epsilon=8, \delta=10^{-5}$
and noise multiplier $\sigma=10$ (note that none of the accountants have an asymptotic runtime dependence on these parameters).
For the R\'enyi accountant, we use $\alpha=8$ by default and do not employ bandwidth truncation.
For the conditional composition accountant, we use the default AM-GM ternary privacy loss bound, i.e., we do not use additional variational distributions that would increase computational cost by a constant factor.
We further set ``bad event'' probability $\delta = \delta_E \mathbin{/} 2$.
With Monte Carlo accounting, we use the minimum number of samples $s_\mathrm{min}$ to attain a moderate significance / failure probability of $p=0.01$ while being able to verify $\delta < 1$.

\textbf{Varying privacy budget.}
In~\cref{fig:runtime_delta}, we vary the privacy budget $\delta$ for BSR factorization ($N=1000$ steps, $b=100$ batches per epoch, bandwidth $p=4$) using the minimum required number of samples for the Monte Carlo baseline's guarantees to hold with significance $p$.
Except for weak guarantees $\delta \geq 10^3$ which require less than $20000$ samples, the runtime of the Monte Carlo method increases inversely proportional to $\delta$. This confirms the importance of developing sampling-free privacy accounting methods for matrix mechanisms, especially when considering strict privacy budgets or large datasets, where one typically uses $\delta < \frac{1}{|D|}$.

\textbf{Varying number of iterations.}
In~\cref{fig:runtime_N}, we vary the number of iterations $N$ for BSR factorization ($b=100$ batches per epoch, bandwidth $p=4$).
For large $N > 10000$, the cost of evaluating the conditional composition accountant $\mathcal{O}(N^2 b^2)$ from scratch is similar to Monte Carlo accounting with $10^{-7} \leq \delta \leq 10^{-8}$.
However, subsequent evaluations (orange) for different noise multipliers $\sigma$ are significantly cheaper, which allows for amortizing the runtime when calibrating $\sigma$. 

\textbf{Varying batches per epoch.}
In~\cref{fig:runtime_b}, we vary the other parameter that conditional composition has a quadratic dependence on,
i.e., batches per epoch $b$ here shwon for BSR factorization ($N=12800$ iterations, bandwidth $p=4$).
As before, for this large $N = 12800$ the cost of evaluating the conditional composition accountant $\mathcal{O}(N^2 b^2)$ from scratch is similar to Monte Carlo accounting with $10^{-7} \leq \delta \leq 10^{-8}$.
However, after an initial evaluation, subsequent evaluations are less expensive than Monte Carlo accounting with $\delta=10^{-6}$.

\textbf{Varying bandwidth.}
Finally, in~\cref{fig:runtime_p} we vary bandwidth $p$ and R{\'e}nyi order $\alpha$ for BSR factorization ($N=1000$ iterations, 100 batches per epoch $p=4$). For R{\'e}nyi accounting, we truncate the  bandwidth of the Gram matrix to be identical to the bandwidth of the BSR noise correlation matrix.
As anticipated, at a fixed compute budget we need to either reduce the bandwidth to evaluate the R\'enyi accountant for large $\alpha$,
or limit $\alpha$, which will result in weaker guarantees in high privacy / small $\epsilon$ settings.

\subsubsection{Effect of bandwidth truncation on R\'enyi bounds}\label{appendix:effect_of_truncation}

\begin{figure}[h!]
    \centering

    \begin{subfigure}[t]{0.47\linewidth}
        \centering
        \includegraphics[trim={0 0 0 0.6cm},clip,width=\linewidth]{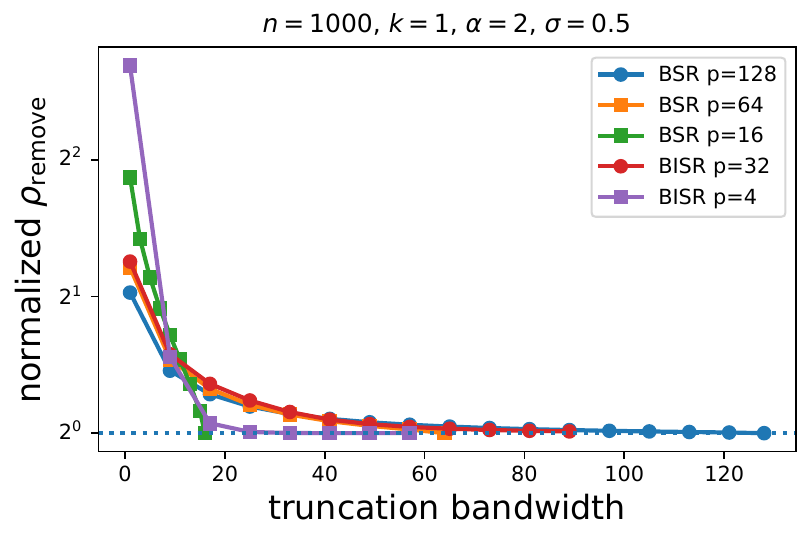}
        \subcaption{$N=1000, k=1, \alpha=2, \sigma=0.5$}
    \end{subfigure}\hfill
    \begin{subfigure}[t]{0.47\linewidth}
        \centering
        \includegraphics[trim={0 0 0 0.6cm},clip,width=\linewidth]{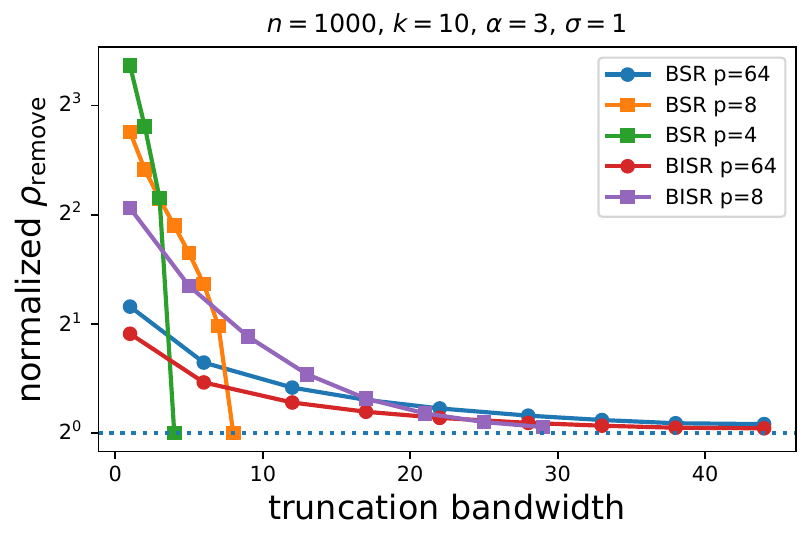}
        \subcaption{$N=1000, k=10, \alpha=3, \sigma=1$}
    \end{subfigure}\hfill
    \caption{
        Comparison of the ``remove'' direction R{\'e}nyi guarantee in the single-epoch ($k=1$) and multi-epoch ($k=10$) setting, 
        after truncation relative to the true guarantee without truncation.
        When the matrix is banded with a small bandwidth, it is better not to truncate it at all, whereas for a large bandwidth, one obtains an accurate approximation with a relatively small $p$.}
    \label{fig:bandwidth_truncation_tradeoff}
\end{figure}

\textbf{Effect on R\'enyi divergence}. In~\cref{fig:gram_subplots}, we visually demonstrated that matrix mechanisms, including those with high bandwidth, often have small effective bandwidth in the Gram matrix.
Our dynamic program exploits this by truncating the Gram matrix to speed up computation at the cost of an easy-to-compute additive term.
In~\cref{fig:bandwidth_truncation_tradeoff}, we compare the computed R\'enyi divergence across different factorizations, bandwidths, and Gram matrix bandwidth truncations. 
We specifically focus on small values of $\alpha$, which correspond to the low-privacy / large-$\epsilon$ regime in which R\'enyi accounting performs best.
In both the single- and multi-epoch setting, we observe the following:
For mechanisms with small bandwidth, truncating too aggressively may introduce a large error.
For mechanisms with large bandwidth (e.g., BSR with $p=64$ or BISR with arbitrary $p$), even truncating to a small fraction can still offer a good upper bound due to their limited effective bandwidth.

\textbf{Effect on correction term.} Recall from our discussion from~\cref{eq:rdp_correction_factor} that truncation itself reduces R\'enyi divergence, which has to be compensated by an additive correction term $\frac{\alpha \tau}{2 \sigma^2}$. For completeness, we evaluate the dependence of $\tau$ on truncation bandwidth $p$ for different mechanisms and $N=100$ steps in~\cref{table:truncation_correction}.

\begin{table}[h]
    \caption{Truncation correction term $\tau$ of the R\'enyi accountant for $N=100$, $k=1$ and varying truncation bandwidth $p$.
    Note that $\tau=0$ for $p > N \mathbin{/} 2$.}
    \label{table:truncation_correction}
    \centering
    \begin{tabular}{l|ccccccc}
        \toprule
         Mechanism & $p=1$ & $p=2$ & $p=4$ & $p=8$ & $p=16$ & $p=32$ & $p=50$\\
         \midrule
         BSR(4) & 0.805 & 0.531 & 0 & 0 & 0 & 0 & 0 \\
         BSR(8) & 1.060 & 0.824 & 0.547 & 0 & 0 & 0 & 0 \\
         BSR(64) & 1.750 & 1.535 & 1.312 & 1.081 & 0.837 & 0.559 & 0.320 \\
         BISR(4) & 0.981 & 0.755 & 0.456 & 0.159 & 0.019 & $2.87 \cdot 10^{-4}$ & $2.51 \cdot 10^{-6}$ \\
         BLT & 1.654 & 1.440 & 1.216 & 0.982 & 0.740 & 0.465 & 0.274 \\
         BandMF & 1.657 & 1.441 & 1.214 & 0.976 & 0.717 & 0.426 & 0.207 \\
        \bottomrule
    \end{tabular}
\end{table}

\textbf{Effect on noise multiplier.} The figure in~\cref{fig:gram_subplots} shows the effect of truncation on our R\'enyi divergence bounds, but these are only an intermediate step. In~\cref{table:truncation_noise_multiplier}, we additionally show the final noise multipliers $\sigma$ after calibrating for a specific privacy budget $(\epsilon,\delta)$. For a given $\epsilon$, we choose $\alpha_\mathrm{max}$ large enough so that the only looseness is due to truncation $p$ and conversion to $(\epsilon,\delta)$-DP. \cref{table:truncation_noise_multiplier_smaller} shows different mechanisms for small $\epsilon \in \{2,3,4\}$ and $p \in \{1,2,4\}$ under $\alpha_\mathrm{max} = 16$ while~\cref{table:truncation_noise_multiplier_larger} shows mechanisms for larger $\epsilon \in \{4,8,16\}$ and $p \in \{2,4,8\}$ under $\alpha_\mathrm{max} = 8$.

\begin{table}[h]
    \caption{Calibrated noise multipliers $\sigma$ for $\delta=10^{-5}$ attained by R\'enyi accounting under varying truncation bandwidth $p$.
    The maximal R\'enyi order $\alpha_\mathrm{max}$ is chosen large enough that non-tightness is only due to truncation
    and conversion to $(\epsilon,\delta)$-DP.}
    \label{table:truncation_noise_multiplier}
    \centering
    \begin{subtable}{0.45\textwidth}
        \caption{Small $\epsilon$ and $p$ under $\alpha_\mathrm{max} = 8$.}
        \label{table:truncation_noise_multiplier_smaller}
        \centering
        \begin{tabular}{lc|ccc}
        \toprule
         Mechanism & p & $\epsilon=2$ & $\epsilon=3$ & $\epsilon=4$\\
         \midrule
         BSR(4) & 1 & 1.937 & 1.349 & 1.045\\
         BSR(4) & 2 & 1.592 & 1.112 & 0.866\\
         BSR(4) & 4 & 1.073 & 0.878 & 0.768\\
         \midrule
         BSR(64) & 1 & 2.847 & 1.982 & 1.534\\
         BSR(64) & 2 & 2.672 & 1.860 & 1.439\\
         BSR(64) & 4 & 2.480 & 1.727 & 1.337\\
         \midrule
         BLT & 1 & 2.769 & 1.928 & 1.491\\
         BLT & 2 & 2.588 & 1.802 & 1.395\\
         BLT & 4 & 2.390 & 1.665 & 1.289\\
         \bottomrule
    \end{tabular}
    \end{subtable}
    \hfill
    \begin{subtable}{0.45\textwidth}
        \caption{Larger $\epsilon$ and $p$ under $\alpha_\mathrm{max} = 8$.}
        \label{table:truncation_noise_multiplier_larger}
        \centering
        \begin{tabular}{lr|ccc}
        \toprule
         Mechanism & p & $\epsilon=4$ & $\epsilon=8$ & $\epsilon=16$\\
         \midrule
         BSR(8) & 2 & 1.063 & 0.609 & 0.404\\
         BSR(8) & 4 & 0.903 & 0.587 & 0.404\\
         BSR(8) & 8 & 0.829 & 0.577 & 0.404\\
         \midrule
         BSR(64) & 2 & 1.439 & 0.795 & 0.474\\
         BSR(64) & 4 & 1.337 & 0.743 & 0.471\\
         BSR(64) & 8 & 1.230 & 0.712 & 0.471\\
         \midrule
         BLT & 2 & 1.395 & 0.770 & 0.463\\
         BLT & 4 & 1.289 & 0.719 & 0.461\\
         BLT & 8 & 1.177 & 0.693 & 0.461\\
         \bottomrule
    \end{tabular}
    \end{subtable}
\end{table}
\color{black}

\subsubsection{Comparison of R\'enyi ``remove'' and ``add'' direction.}\label{appendix:renyi_add_remove_slit_experiments}
\begin{figure*}[h!]
    \centering
    \begin{subfigure}[t]{0.31\linewidth}
        \centering
        \includegraphics[width=\linewidth]{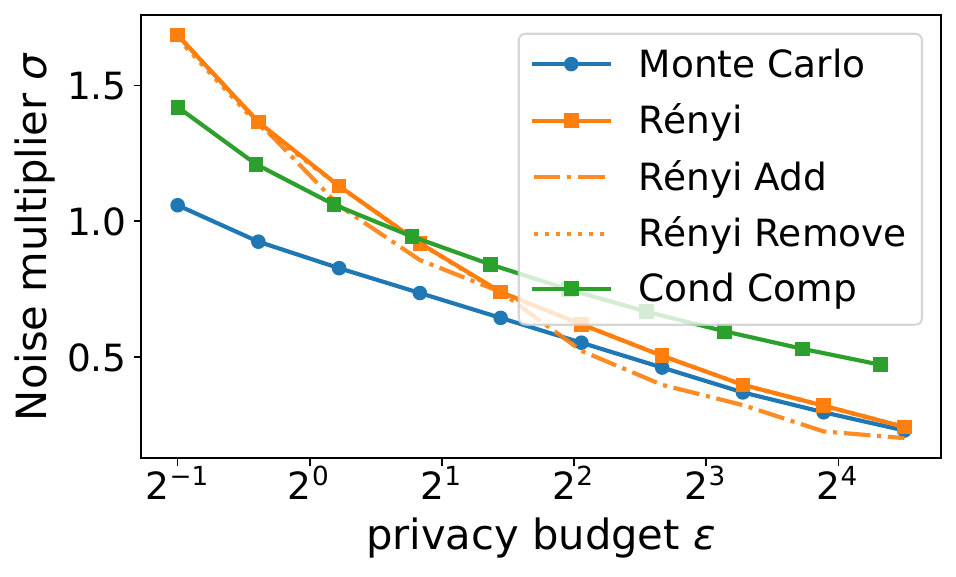}
        \subcaption{DP-SGD, $N\mathbin{=}100,\ k\mathbin{=}1$}
    \end{subfigure}\hfill
    \begin{subfigure}[t]{0.31\linewidth}
        \centering
        \includegraphics[width=\linewidth]{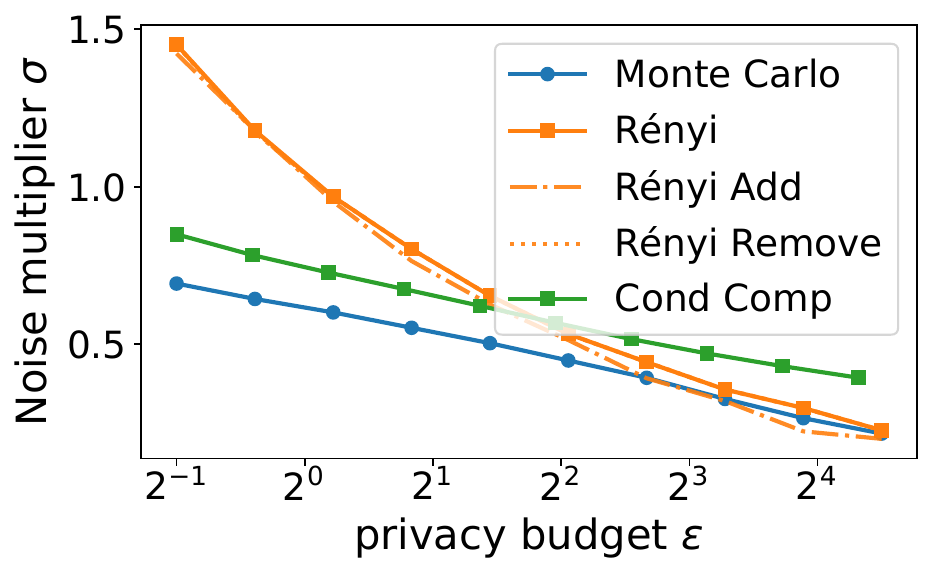}
        \subcaption{DP-SGD, $N\mathbin{=}1000,\ k\mathbin{=}1$}
    \end{subfigure}\hfill
    \begin{subfigure}[t]{0.31\linewidth}
        \centering
        \includegraphics[width=\linewidth]{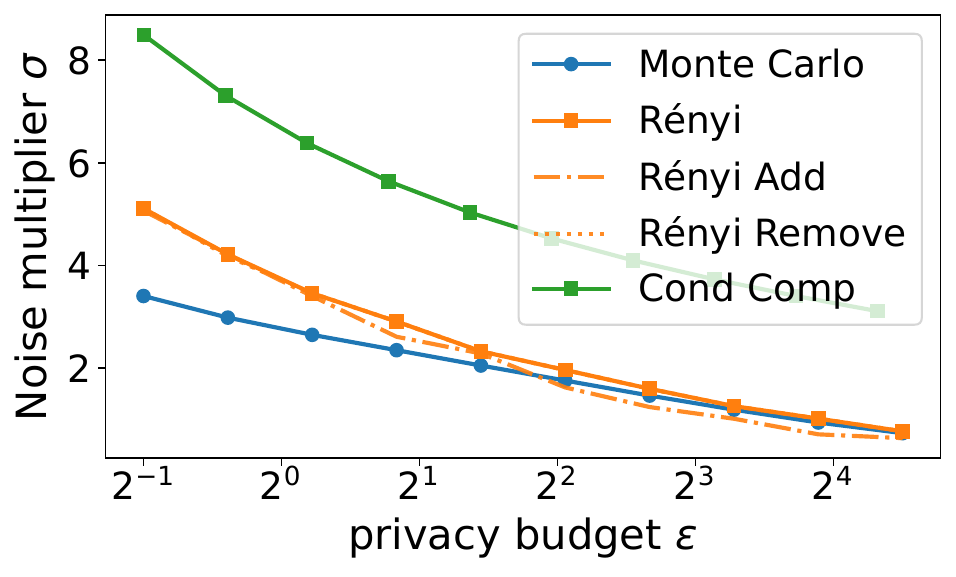}
        \subcaption{DP-SGD, $N\mathbin{=}1000,\ k\mathbin{=}10$}
    \end{subfigure}

    \begin{subfigure}[t]{0.31\linewidth}
        \centering
        \includegraphics[width=\linewidth]{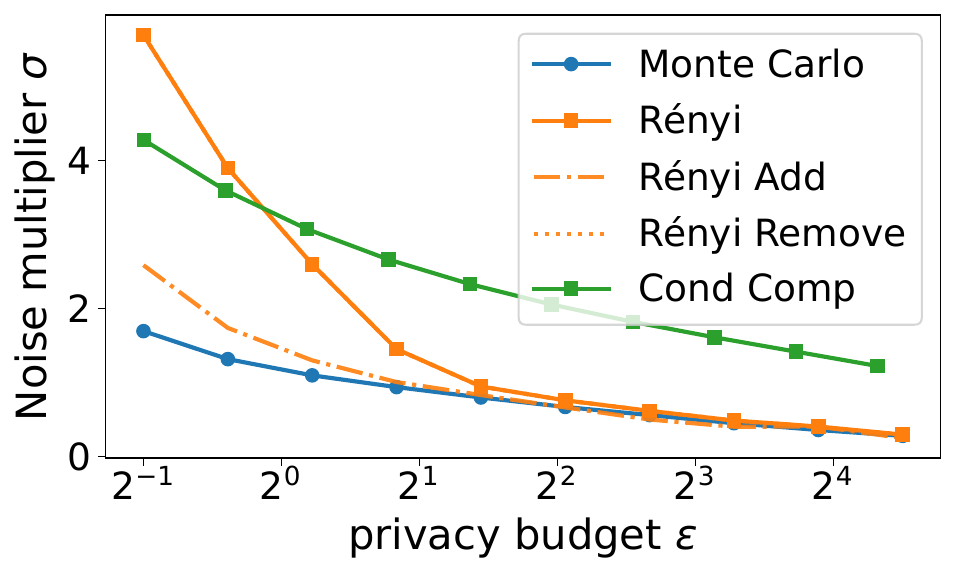}
        \subcaption{BSR, $N\mathbin{=}100,\ k\mathbin{=}1,\ p\mathbin{=}4$}
    \end{subfigure}\hfill
    \begin{subfigure}[t]{0.31\linewidth}
        \centering
        \includegraphics[width=\linewidth]{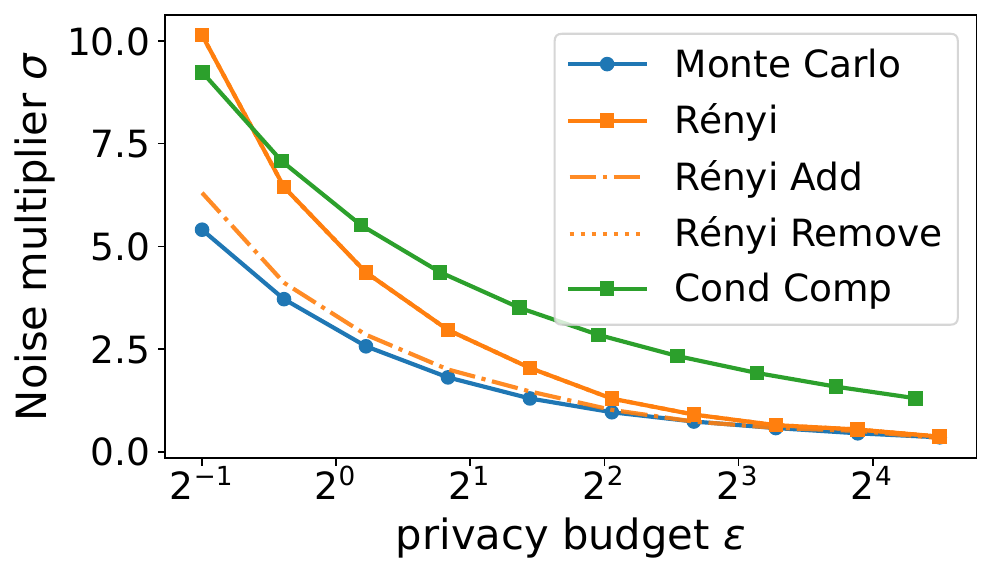}
        \subcaption{BSR, $N\mathbin{=}100,\ k\mathbin{=}1,\ p\mathbin{=}64$}
    \end{subfigure}\hfill
    \begin{subfigure}[t]{0.31\linewidth}
        \centering
        \includegraphics[width=\linewidth]{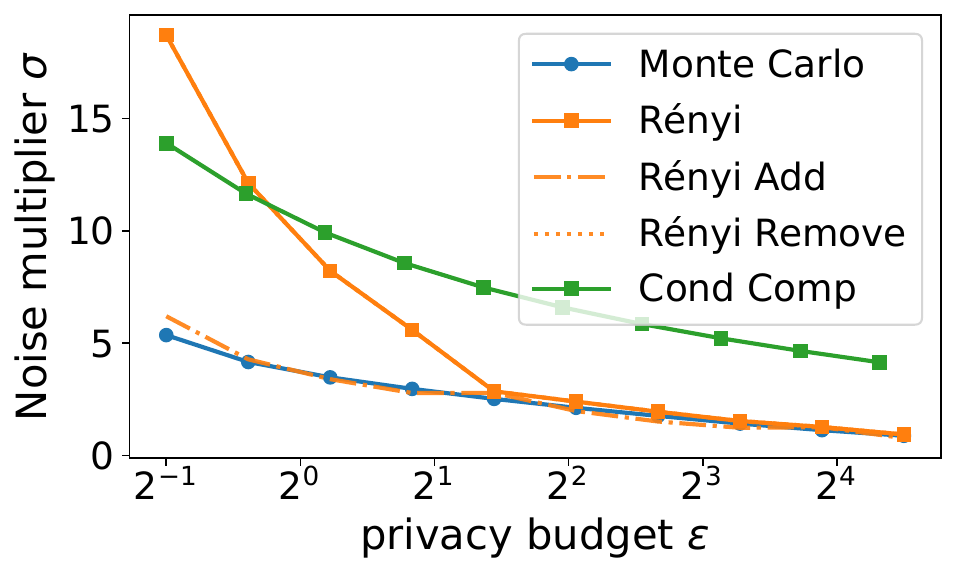}
        \subcaption{BSR, $N\mathbin{=}1000,\ k\mathbin{=}10,\ p\mathbin{=}4$}
    \end{subfigure}

    \caption{Main experimental~\cref{fig:rdp_mcmc_dpsgd}, re-drawn with separate ``add'' and ``remove'' directions for the R\'enyi accountant. The ``add'' bound is consistently smaller, i.e., the overall privacy guarantee is identical to the ``remove'' direction and the ``add'' slack has no effect on the overall privacy guarantee.}
    \vspace{-0.3cm}
    \label{fig:renyi_add_remove}
\end{figure*}

As describe~\cref{alg:renyi_dynamic_program_full}, our overall R\'enyi accountant uses a dynamic program for the ``remove'' direction
and an analytical bound for the ``add'' direction.
If the ``add'' bound were consistently larger, then the ``remove'' bound would never be active and there would be no benefit to running the dynamic program.
However, as shown in~\cref{fig:renyi_add_remove}, this is not the case. In fact, we consistently found the ``add'' direction bound to be smaller across parameters and matrix factorizations.
This matches up with prior work on R\'enyi accounting.
For example,~\citet{feldman2025privacy} claim for the special case of uncorrelated noise (DP-SGD) that ``numerical analysis seems to indicate the bound on the remove direction always dominates the one for the add direction''.
\citet{mironov2017renyi} formally prove for Poisson subsampling that the ``remove'' direction always dominates.

\subsection{Effect of bad event probability on conditional composition bounds}
In all other experiments, we choose bad event probability $\delta_E = \frac{\delta}{2}$ for  consistency with \cite{choquette2023privacy}. However, one can in principle use any $\delta_E \in (0,\delta)$.
In~\cref{table:bad_event_optimization}, we compare our previous results to the the calibrated noise multiplier $\sigma$ for an optimized $\delta \in \{0.9, 0.5, \mathrm{logspace}(-1,-3,5\}$ for DP-SGD and BISR(64). 
We find that the grid search over $\delta_E$ can offer slightly better $\sigma$, but that $\delta \mathbin{/} 2$ from prior work is a reasonable default.

\begin{table}[h]
    \caption{Comparison of calibrated noise multiplier $\sigma$ for conditional composition accountant 
    with ``bad'' event probability $\delta_E = \delta \mathbin{/} 2$ compared to a
    grid search over $\delta_E \in \{0.9, 0.5, \mathrm{logspace}(-1,-3,5\}$.
    Optimization can offer  better noise multipliers.}
    \label{table:bad_event_optimization}
    \centering
    \begin{tabular}{lr|cccccc}
        \toprule
         Mechanism & Optimized & $\epsilon = 0.50$ & $\epsilon = 0.75$ & $\epsilon = 1.14$ & $\epsilon = 1.71$ & $\epsilon = 2.58$ & $\epsilon = 3.88$\\
         \midrule
         DP-SGD & N & 1.420 & 1.210 & 1.061 & 0.943 & 0.841 & 0.749 \\
         DP-SGD & Y & 1.396 & 1.192 & 1.047 & 0.931 & 0.831 & 0.741 \\
         BISR(64) & N & 9.179 & 7.065 & 5.523 & 4.379 & 3.516 & 2.856 \\
         BISR(64) & Y & 9.047 & 6.995 & 5.490 & 4.367 & 3.516 & 2.856 \\
        \bottomrule
    \end{tabular}
\end{table}

\subsection{Comparison to DP-SGD-specific random allocation bounds}\label{appendix:feldman_comparison}

The focus of our work is on privacy accounting for correlated noise,
but are accountants are also applicable to the special case of uncorrelated noise (DP-SGD).
For completeness, we thus also compare to~\citet{feldman2025privacy, feldman2026efficient}.

\subsection{Comparison to Feldman \& Shenfeld (2025) method for uncorrelated noise}\label{appendix:old_feldman_experiments}
\begin{figure*}[h!]
    \centering
    \begin{subfigure}[t]{0.45\linewidth}
        \centering
        \includegraphics[width=\linewidth]{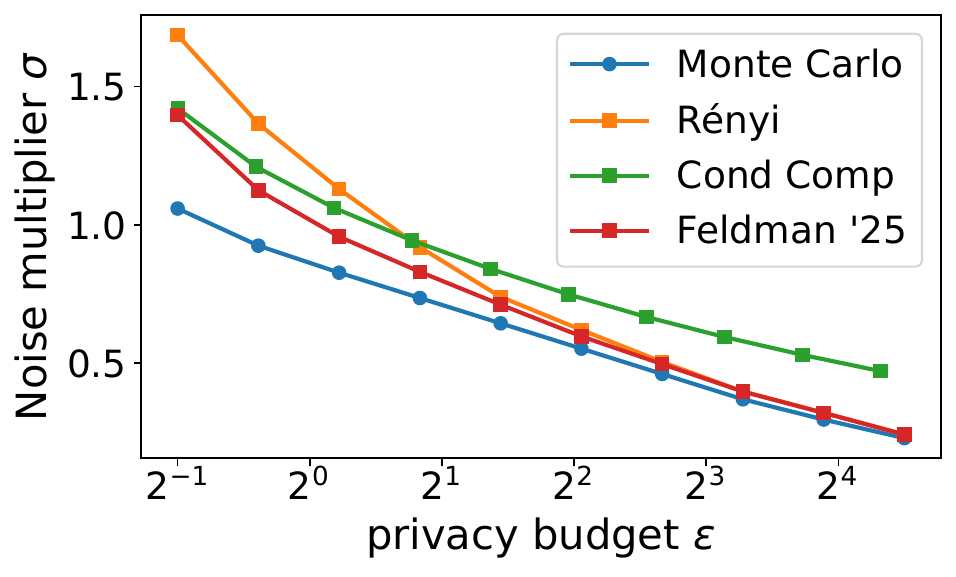}
        \subcaption{DP-SGD, $N=100,\ k=1$}
    \end{subfigure}
    \hfill
    \begin{subfigure}[t]{0.45\linewidth}
        \centering
        \includegraphics[width=\linewidth]{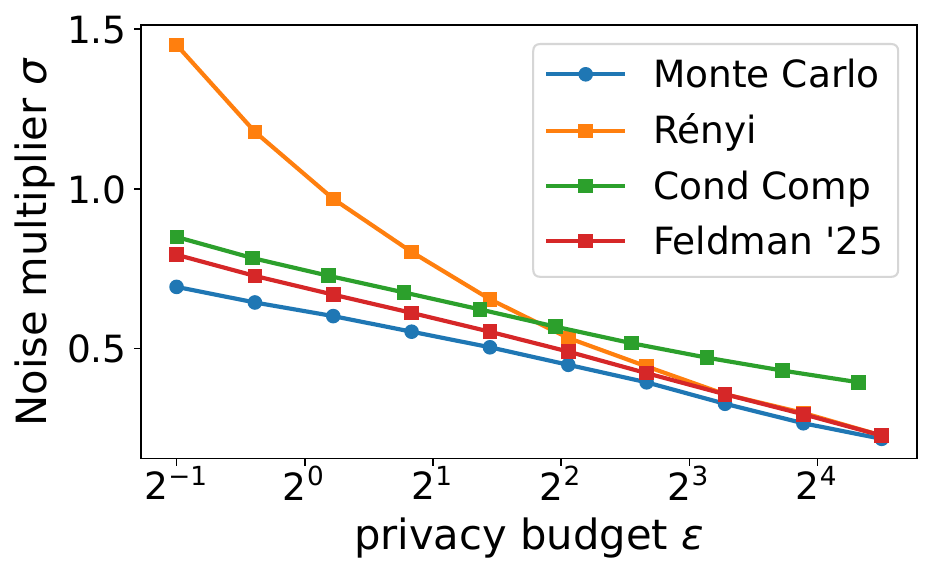}
        \subcaption{DP-SGD, $N=1000,\ k=1$}
    \end{subfigure}
    \caption{Main experimental~\cref{fig:rdp_mcmc_dpsgd_n_100_k_1,fig:rdp_mcmc_dpsgd_n_1000_k_1}, re-drawn to also include the DP-SGD specific guarantees of~\cite{feldman2025privacy} Their bounds are qualitatively similar (slack for small $\epsilon$, tight for large $\epsilon$), but tighter because they combine a wider range of methods specific to uncorrelated noise.}
    \label{fig:comparison_to_feldman_tightness}
    \vspace{-0.3cm}
\end{figure*}

Similar to us, \citet{feldman2025privacy} combine an exact R\'enyi accountant with other bounds that are tighter for small $\epsilon$.
However, as shown in~\cref{fig:comparison_to_feldman_tightness}, their procedure offers tight bounds for small to moderate privacy budgets $\epsilon$. This raises the question of what prevents our accountants from achieving similar tightness.

By definition, our exact R\'enyi accountant yields the same results as the ``direct analysis'' accountant from their Theorems 4.8 and 4.9 when applied to DP-SGD.
However, our novel dynamic program does so in a computationally much more efficient manner (see~\cref{appendix:renyi_runtime_vs_feldman} for empirical runtimes).

Similarly, our conditional composition accountant from~\cref{algorithm:conditional_composition} could be made identical to their ``truncated Poisson'' accountant for DP-SGD (their Theorem 4.1) by letting $\lambda_1 = \dots = \lambda_{b-1}$ and using their maximal Azuma Hoeffding bound to bound $\lambda_b$. However, as can be seen from their Appendix B.4, this bound relies on the same decomposition assumption as~\citep{feldman2026efficient}. As we already discussed  in~\cref{section:problem_statement,appendix:pld_inapplicability}, this decomposition into per-component PLRVs does not generalize to matrix mechanisms. 

Furthermore, the same PLRV decomposition lets them recursively refine their bound (their Theorem 4.2),
which again would not generalize  to matrix mechanisms.

Finally, they are able to apply a complete separate class of bounds (``Poisson decomposition''), which exploits that
they only ever need to distinguish two cases per step, just like in Poisson subsampling: Either the private record participates or not. In contrast, our general method has to account for the fact that earlier participation may also have an indirect effect through correlation. 

To summarize, since we address a more general problem, we must use more general bounds that are less tight when restricted to the special case of DP-SGD.

\subsection{Comparison to Feldman \& Shenfeld (2026) method for uncorrelated noise}

\begin{figure*}[h!]
    \centering
    \begin{subfigure}[t]{0.45\linewidth}
        \centering
        \includegraphics[width=\linewidth]{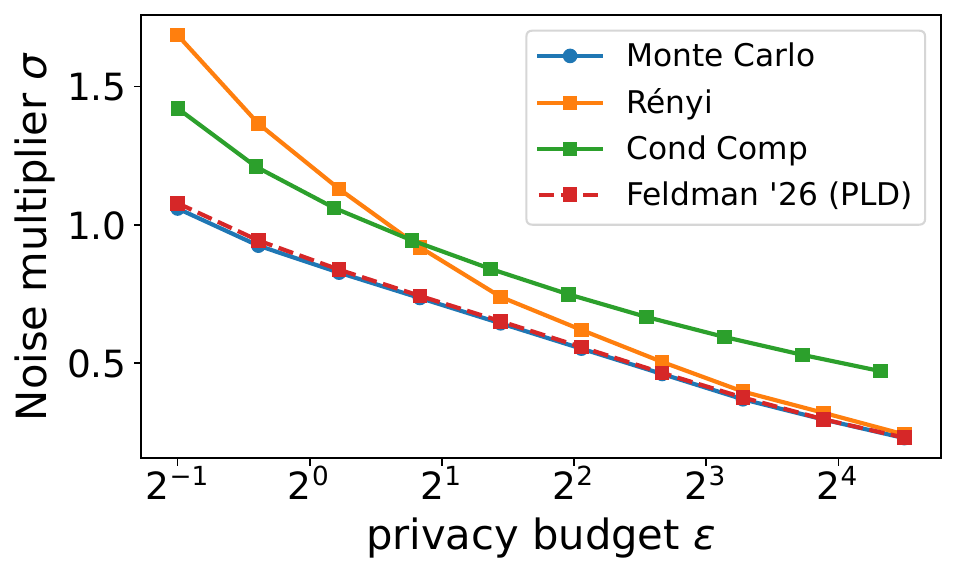}
        \subcaption{DP-SGD, $N=100,\ k=1$}
    \end{subfigure}
    \hfill
    \begin{subfigure}[t]{0.45\linewidth}
        \centering
        \includegraphics[width=\linewidth]{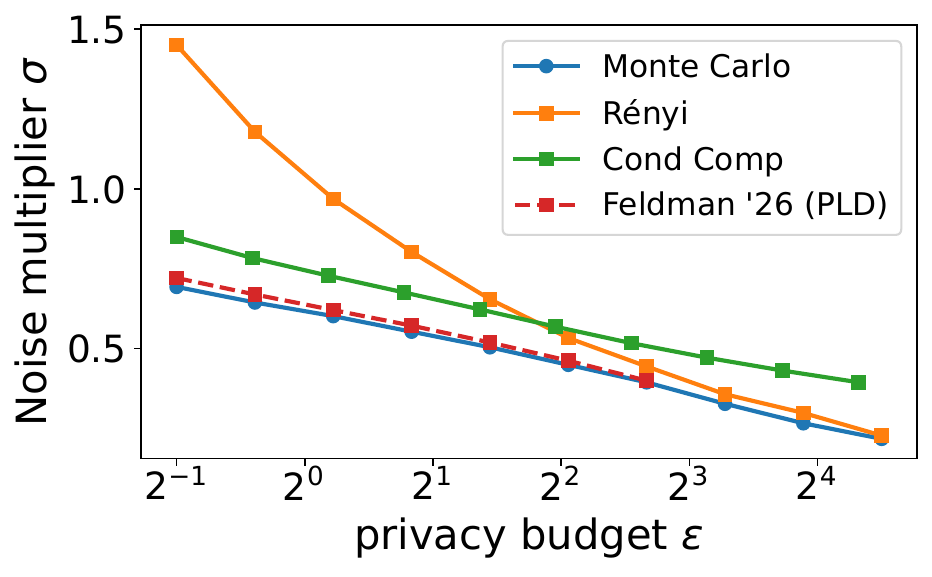}
        \subcaption{DP-SGD, $N=1000,\ k=1$}
    \end{subfigure}
    \caption{Main experimental~\cref{fig:rdp_mcmc_dpsgd_n_100_k_1,fig:rdp_mcmc_dpsgd_n_1000_k_1}, re-drawn to also include the DP-SGD specific guarantees of~\citet{feldman2026efficient}. Their bounds more almost exactly the Monte Carlo accountant.}
    \label{fig:comparison_to_feldman_tightness_pld}
    \vspace{-0.3cm}
\end{figure*}

In concurrent work,~\citet{feldman2026efficient} improve upon their previous paper through a method that foregoes
bounding privacy profile $\delta(\epsilon)$ or R\'enyi profile $\rho(\alpha)$ 
and instead operates in privacy loss distribution space.
As shown in~\cref{fig:comparison_to_feldman_tightness_pld}, this leads to improved bounds that almost exactly match those attained via Monte Carlo accounting (see~\cref{fig:comparison_to_feldman_tightness_pld}).

We believe that future work should work towards privacy loss distribution accounting for subsampled matrix mechanisms.
However, for reasons outlines in~\cref{section:problem_statement,appendix:pld_inapplicability}, we do not believe that this can be done through a simple generalization from their DP-SGD method.

\subsection{Comparison to Feldman \& Shenfeld (2025) R\'enyi accountant runtime}\label{appendix:renyi_runtime_vs_feldman}

\begin{figure}[H]
    \centering

    \begin{subfigure}[t]{0.48\linewidth}
        \centering
        \includegraphics[width=\linewidth]{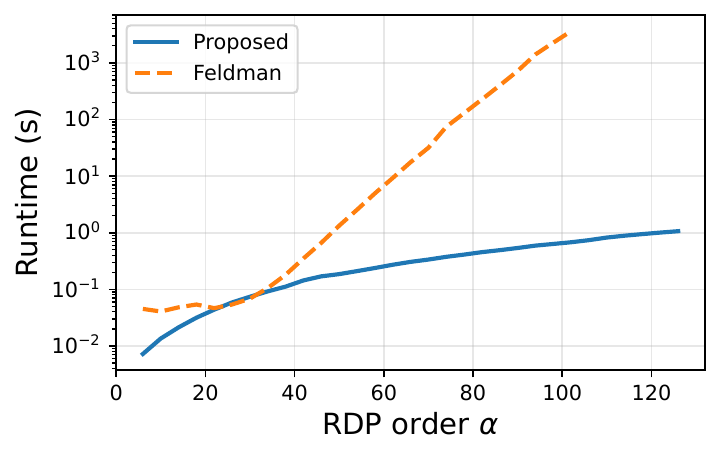}
        \subcaption{DP-SGD ($N=100, b=100$)}
    \end{subfigure}\hfill
    \begin{subfigure}[t]{0.48\linewidth}
        \centering
        \includegraphics[width=\linewidth]{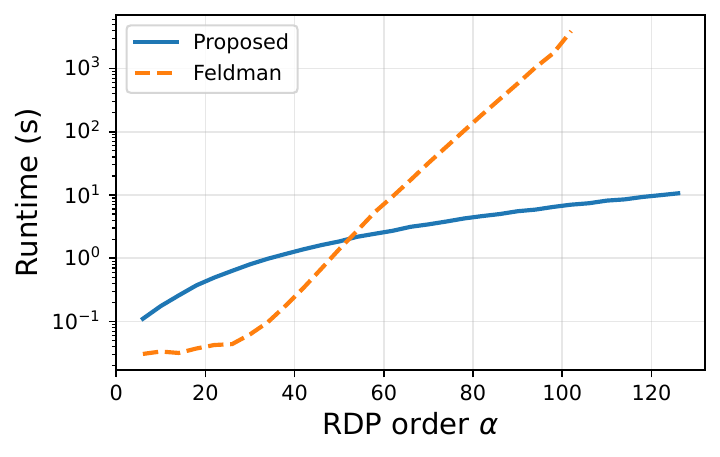}
        \subcaption{DP-SGD ($N=1000, b=1000$)}
    \end{subfigure}\hfill
    \caption{Wall-clock runtime comparison to Feldman et al. for the special case of uncorrelated noise / DP-SGD.
    Except for $\alpha \leq 30$, which benefits from vectorization, hasing, and just-in-time compilation,
    the baseline's runtime explodes explodes exponentially ($\mathcal{O}(2^\alpha)$) 
    while our polynomial-time dynamic program ($\mathcal{O}(b \alpha^2)$) scales to much larger $\alpha$}
    \label{fig:appendix_feldman_runtime}
\end{figure}

In~\cref{section:rdp_accountant} we claimed an improvement in runtime for exact R\'enyi accounting of uncorrelated-noise random allocation (DP-SGD) in the ``remove direction''.
Specifically, we claimed an improvement from exponential $\mathcal{O}(2^\alpha)$ to quadratic $\mathcal{O}(b \alpha^2)$
compared to 
the best known method from~\cite{feldman2025privacy}.
To confirm that this bound does not only hold in asymptotic limits,
we evaluate both accountants on the the same Intel Xeon Gold 6148 CPU (\SI{2.4}{\giga\hertz}, \SI{768}{\giga\hertz} RAM).
The empirical runtime measurements in~\cref{fig:appendix_feldman_runtime}
confirm that the baseline's cost explodes exponentially, whereas our method scales to large R\'enyi orders $\alpha$.

\clearpage

\section{Experimental details}\label{appendix:experimental_setup}

\subsection{Conditional composition accountant}
\textbf{Parameters for Algorithm 2.} In principle, given a target $\delta \in [0,1]$, we can choose an arbitrary ``bad'' event probability $\delta_E < \delta$
or even optimize over many different choices of $\delta_E < \delta$.
For simplicity, we follow~\cite{choquette2023privacy} and simply set $\delta_E = \delta \mathbin{/} 2$.
In addition to the default AM-GM bounds from~\cref{theorem:tail_amgm_add,theorem:tail_amgm_add} we evaluate the variational bound from~\cref{appendix:option_3_variational} which retains the same computational complexity.
We set its temperatures to $T = \{10^{-1}, 10^{-0.5}, 1, 10^{0.5}, 10^{1}\}$ based on our ablation from~\cref{fig:variational_bound_balation}.

\textbf{Evaluation of conditional composition bound.} After determining per-step dominating pairs $(P^{(1)}, Q^{(1)}), \dots, (P^{(N)}, Q^{(N)})$ via~\cref{algorithm:conditional_composition}, we compute their $N$-fold composition using the Google \texttt{dp\_accounting} library~\cite{dpaccountinglibrary}  specifically its PLD accountant applied to the $\texttt{MixtureGaussianPrivacyLoss}$.
We quantize the privacy loss distribution using the ``connect the dots'' algorithm from~\cite{doroshenko2022connect},
while adjusting the value discretization interval such that each quantized privacy loss pmf has a support of size $\leq 1000$
and at least one has a support of size exactly $1000$.
We otherwise use standard parameters, in particular a log-mass truncation bound of $-50$ for the privacy loss mechanism and a 
a tail mass truncation of $\SI{1e-15}{}$ for the quantized privacy loss pmfs.

\subsection{R\'enyi accountant}
In our experiments, we have numerical trade-offs arising in privacy accounting. The R\'enyi accountant takes the bandwidth as an input. If the matrix $\mC$ is $p$-banded for a small $p$, we can compute the divergence exactly. Otherwise, when the bandwidth is large or $\mC$ is not banded, we must use an \emph{effective} bandwidth: we pass a smaller $p$ and upper bound the contributions outside the $p$-cyclic band (see Algorithm~\ref{alg:renyi_dynamic_program}). This creates a trade-off between computational cost and the tightness of the bound.

This situation occurs, for example, for banded-inverse factorizations such as BISR and BandInvMF \citep{kalinin2025back}, where $\mC^{-1}$ is banded but $\mC$ is not, and for Buffered Linear Toeplitz (BLT) \citep{dvijotham2024efficient}, which is essentially dense in both $\mC$ and $\mC^{-1}$. Moreover, smaller privacy budgets $\varepsilon$ generally require larger R\'enyi orders $\alpha$, meaning we must use smaller effective bandwidth.
For our experiments, we thus report the optimal bound that can be attained for combinations of $(p,\alpha) \in \{2\} \times \{1,\dots,25\}
\cup \{4\} \times \{1,\dots,7\} \cup \{2\} \times \{1,\dots,4\}$.
The exception are ancillary experiments in~\cref{appendix:extra_experiments} in which we explicitly vary $\alpha$ across different ranges.

\subsection{Monte Carlo accountant.}
By default, we use $400000$ samples to compute raw Monte Carlo estimates for our experiments.
For our ancillary experiments in~\cref{appendix:monte_carlo_whp_bounds}, where we evaluate high probability upper bounds using the method from~\cite{choquette2024near}, we set the number of samples to different multiples of the minimum number of samples $s_\mathrm{min}$ required to obtain a bound better than $\delta=1$ (see details in~\cref{appendix:monte_carlo_whp_bounds}).

\subsection{DP-SGD random allocation accountants.}
While the focus of our work is on correlated-noise mechanisms, we still conduct a comparison to~\citet{feldman2025privacy} and~\citet{feldman2026efficient} for completeness (see~\cref{appendix:feldman_comparison}) using their respective reference implementations.

\citet{feldman2025privacy}: We use the official $\texttt{random-allocation}$ library (v1.0) available via pip.
We take the minimum over all available sub-accountants using the \texttt{allocation\_delta\_combined} method.
For comparability with our other R\'enyi accounting results, we set the maximum $\alpha$ to $25$.
We otherwise use default parameters. In particular, we use a discretization of $10^{-4}$ and a $\delta$-tolerance of $10^{-15}$.

\citet{feldman2026efficient}: We use the official \texttt{PLD\_subsampling} library (commit 952e91a36230e81bfa7c072e74dc89bb89750b0f).
We leave all parameters at their default values. In particular, we use a loss discretization of $0.1$, a tail truncation of $10^{-12}$, a maximum FFT grid of $10^6$, set $\texttt{max\_grid\_mult} = -1$ and employ the geometric convolution method.

\subsection{Optimized matrix factorizations}
The DP-SGD, BSR, and BISR factorization have closed-form expressions. In contrast, BandMF, BLT, and BandInvMF are minimized to best approximate the non-private workload matrix. 
For BandMF, BLT, and BandInvMF, we use the Google \texttt{jax-privacy} implementation \citep{jaxprivacy2022github} for this optimization.

\subsection{CIFAR-10 experiments}
For our CIFAR-10~\cite{krizhevsky2009learning} experiments, which are a standard task for evaluating the privacy--utility trade-off of matrix mechanisms (see, e.g.,~\cite{choquette2023multi,choquette2023privacy,choquette2024near}) we replicate the experimental setup from~\cite{kalinin2024banded}.

\textbf{Training mechanism.}
We use the standard train--test split, i.e., $50000$ training samples and $10000$ test samples.
Training is performed for $N=970$ steps divided into $k=10$ epochs, which corresponds to a batch size of $512$. Gradients are clipped to an $\ell_2$-norm of $8$ and we calibrate noise multipliers to $\delta=10^{-5}$.
We evaluate three different factorizations: DP-SGD, BSR ($p=4$), and BISR ($p=4$).

\textbf{Model architecture.} We use the following ConvNet architecture:
\begin{itemize}[noitemsep]
    \item $2 \times$ Conv2D(channels=32, kernel=(3, 3), strides=(1, 1), padding=’SAME’, activation=’relu’)
    \item MaxPool(kernel=(2, 2), strides=(2, 2))
    \item $2 \times$ Conv2D(channels=64, kernel=(3, 3), strides=(1, 1), padding=’SAME’), activation=’relu’)
    \item MaxPool(kernel=(2, 2), strides=(2, 2))
    \item $2 \times$ Conv2D(channels=128, kernel=(3, 3), strides=(1, 1), padding=’SAME’), activation=’relu’)
    \item $2 \times$ MaxPool(kernel=(2, 2), strides=(2, 2))
    \item Flatten()
    \item Dense(outputs=10)
\end{itemize}

\textbf{Other details.} We use $5$ random seeds and show mean and standard deviation of test accuracy.

\subsection{Compute resources.}
All experiments for R\'enyi accounting and CIFAR-10 training were performed on a basic Google collab instance.
Experiments involving conditional composition were performed using an Intel Xeon Gold 6148 CPU (\SI{2.4}{\giga\hertz}, \SI{768}{\giga\hertz} RAM).
All runtime measurements for~\cref{appendix:empirical_runtime} were performed using this Intel Xeon CPU.
The experiments did not involve multi-worker orchestration beyond python process tools for parallelizing
PLd accounting with the \texttt{dp\_accounting} library (see provided implementation of the conditional composition accountant.).

\subsection{Licenses of used assets}\label{appendix:licenses}
CIFAR-10 is available under MIT license. 
The \texttt{dp\_accounting} and \texttt{jax\_privacy} libraries are available under Apache-$2.0$ license.
The \texttt{random\_allocation} and \texttt{PLD\_accounting} libarries are available under MIT license.

\clearpage

\section{Full details of the Rényi accounting algorithm}\label{appendix:full_renyi_accountant}

\begin{algorithm}[h!]
\caption{Rényi accountant ``remove'' direction}
\label{alg:renyi_dynamic_program}
\begin{algorithmic}[1]

\REQUIRE Correlation matrix $\mathbf{C}$, bandwidth $p$, separation $b$, parameter $\alpha \in \mathbb{N}_{+}$.

\STATE $\mathbf{m}_i \leftarrow \sum_{j=0}^{k-1} |\mathbf{C}|_{:,\, i + jb}$

\STATE $\mathbf{G}_{i,j} \leftarrow \langle \mathbf{m}_i, \mathbf{m}_j \rangle \quad \forall i,j \in \{1,\dots,b\}$

\STATE $\tau \leftarrow \max_{\min(|i-j|, b- |i - j|) \ge p} \mathbf{G}_{i,j}$

\STATE $\mathbf{G}^{(p)}_{i,j} \leftarrow \max(\mathbf{G}_{i,j} - \tau, 0) \quad \forall \min(|i-j|, b- |i - j|) < p$

\STATE \# Loop over all possible counts for the first $p-1$ values
\STATE $\log S \leftarrow -\infty$
\FORALL{$l = (l_0, \dots, l_{p - 2})$ s.t. $\sum_{i = 0}^{p - 2} l_{i} = m_l \le \alpha$}
\STATE \# Precompute the prefix dynamics value
\STATE $\log w_l \leftarrow \sum\limits_{i < j}^{p - 2} \mathbf{G}^{(p)}_{i,j}\frac{l_i l_j}{\sigma^2} +\sum\limits_{i = 0}^{p - 2}[ \mathbf{G}^{(p)}_{i,i}\frac{l_i(l_i - 1)}{2\sigma^2}-\log (l_i!)]$
\STATE $\texttt{states} \leftarrow \{(m_l, l) \mapsto \log w_l\}$
\FOR{$k = p - 1$ to $b - 1$}
    
    \STATE $\texttt{new\_states} \leftarrow \emptyset$
    
    \FORALL{$(m,r) \in \texttt{states}$}
        \STATE $\log w \leftarrow \texttt{states}(m,r)$
        \STATE $L \leftarrow \text{length}(r)$
        \STATE $S_1 \leftarrow 2 \sum_{i=0}^{L-1} \mathbf{G}^{(p)}_{k,\,k-L+i } \cdot r_i$
        
        \FOR{$t = 0$ to $\alpha - m$}
            \STATE
            $\Delta \leftarrow
            \dfrac{\mathbf{G}^{(p)}_{k,k} \, t(t-1) + S_1 t}{2\sigma^2}
            - \log(t!)$
            
            \STATE $m' \leftarrow m + t$

            \IF{$p = 1$}
            \STATE $r' = ()$
            \ELSIF{$L < p - 1$}
                \STATE $r' \leftarrow (r_0,\dots,r_{L-1}, t)$
            \ELSE
                \STATE $r' \leftarrow (r_1,\dots,r_{L-1}, t)$
            \ENDIF
            
            \STATE $\log w_{\text{old}} \leftarrow \texttt{new\_states}(m',r')$
            \STATE
            $\log w_{\text{new}}
            \leftarrow
            \operatorname{LSE}(\log w, \log w_{\text{old}} + \Delta)$
            \STATE
            $\texttt{new\_states}(m',r')
            \leftarrow
            \log w_{\text{new}}$
        \ENDFOR
    \ENDFOR

    \STATE $\texttt{states} \leftarrow \texttt{new\_states}$
\ENDFOR
\STATE $\log S_l \leftarrow -\infty$

\FORALL{$(m,r) \in \texttt{states}$}
    \IF{$m = \alpha$}
        \STATE $\log S_{l,r} \leftarrow \sum\limits_{i,j=0}^{p - 2} \frac{\mathbf{G}^{(p)}_{i, b-1 -j}l_i r_{L - 1 - j}}{\sigma^2}$
        \STATE $\log S_l \leftarrow \operatorname{LSE}\bigg(\log S_l,\; \texttt{states}(m,r), \log S_{l,r}\bigg)$
    \ENDIF
\ENDFOR
\STATE $\log S \leftarrow \operatorname{LSE}(\log S,\; \log S_l)$

\ENDFOR

\STATE
$\rho \leftarrow
\dfrac{\log S + \log(\alpha!) - \alpha \log b}{\alpha - 1}
+ \dfrac{\tau \alpha}{2\sigma^2}$
\STATE \textbf{return} $\rho$
\end{algorithmic}
\end{algorithm}

\section{Proofs for Section 3}\label{appendix:renyi_proofs}
\RenyiDivergenceBound*
\begin{proof}

By Lemma~\ref{lem:dominating_pair} $(\hat{P},\hat{Q})$ is a dominating pair for the ``add'' adjacency, where the density of $\hat{P}$ is given by
\[
p(\mathbf{x}) = \frac{1}{b}\sum_{j=1}^{b} \phi(\mathbf{x}; \mathbf{m}_j, \sigma^2 I),
\]
and the density of $\hat{Q}$ is
\[
q(\mathbf{x}) = \phi(\mathbf{x}; \mathbf{0}, \sigma^2 I),
\]
where $\phi(\mathbf{x}; \mu, \Sigma)$ denotes the density of the multivariate Gaussian distribution.

Thus, for $\alpha>1$, the Rényi divergence in the add direction is
\begin{equation}
    R_\alpha(\hat{P} || \hat{Q}) = \frac{1}{\alpha - 1} \log
    \int_{\sR^N} \frac{p(\vx)^\alpha}{q(\vx)^{\alpha-1}} \ \mathrm{d} \vx = \frac{1}{\alpha - 1} \log\mathbb{E}_{\vx \sim Q}\frac{p(\vx)^\alpha}{q(\vx)^{\alpha}}
\end{equation}
Via multinomial rule and quadratic expansion, we have
\begin{align*}
     p(\vx)^\alpha &= \left(
        \frac{1}{b}\sum_{j=1}^{b} \phi(\mathbf{x}; \mathbf{m}_j, \sigma^2 I))
    \right)^\alpha
    =
    \frac{1}{b^\alpha}
    \sum_{r_1 + \cdots + r_b = \alpha}
    \binom{\alpha}{r_1,\dots,r_b} 
    \left(\prod_{j=1}^b \phi(\mathbf{x}; \mathbf{m}_j, \sigma^2 I)^{r_j}
    \right)
    \\
    &=  \frac{1}{b^\alpha}
    \sum_{r_1 + \cdots + r_b = \alpha}
    \binom{\alpha}{r_1,\dots,r_b} \frac{1}{(2\pi\sigma^2)^{N\alpha / 2}}\exp\left(-\frac{1}{2\sigma^2}\sum\limits_{j = 1}^{b}r_j (\mathbf{x} - \vm_j )^T(\mathbf{x} - \vm_j)\right)
    \\
    &=
    \frac{1}{b^\alpha}\frac{1}{(2\pi\sigma^2)^{N\alpha / 2}}\exp\left(
        -\frac{\alpha}{2 \sigma^2} \vx^T \vx
    \right)
    \sum_{r_1 + \cdots + r_b = \alpha}
    \binom{\alpha}{r_1,\dots,r_b}\exp\left(
         \frac{1}{\sigma^2} \vx^T \left[\sum\limits_{j = 1}^{b} r_j \vm_j\right]
    \right)\\
    &\hspace{8.7cm} \times \exp \left(-\frac{1}{2\sigma^2}\sum\limits_{j = 1}^{b}r_j\vm_j^T\vm_j\right)
\end{align*}

The $q(\vx)^\alpha$ is simply given by:

\begin{equation*}
    q(\vx)^\alpha = \frac{1}{(2\pi\sigma^2)^{N\alpha / 2}}\exp\left(
        -\frac{\alpha}{2 \sigma^2} \vx^T \vx
    \right).
\end{equation*}

Then the R\'enyi divergence is computed as:
\begin{align*}
     R_\alpha(\hat{P} || \hat{Q}) = \frac{1}{1 - \alpha} \log \Bigg(\frac{1}{b^\alpha} \sum_{r_1 + \cdots + r_b = \alpha}
    \binom{\alpha}{r_1,\dots,r_b}&\exp \left(-\frac{1}{2\sigma^2}\sum\limits_{j = 1}^{b}r_j\vm_j^T\vm_j\right)\\
    &\times \mathbb{E}_{\vx \sim Q}\exp\Bigg(
         \frac{1}{\sigma^2} \vx^T \bigg[\sum\limits_{j = 1}^{b} r_j \vm_j\bigg]
    \Bigg)\Biggr)
\end{align*}

Using the moment generating function of a multivariate Gaussian, for $\mathbf{x}\sim \mathcal{N}(\mathbf{0}, \sigma^2 I)$ and any $\mathbf{t}\in\mathbb{R}^N$,
\[
\mathbb{E}\big[\exp(\mathbf{t}^\top \mathbf{x})\big]
=
\exp\!\left(\frac{\sigma^2}{2}\|\mathbf{t}\|_2^2\right),
\]
and taking $\mathbf{t} = \frac{1}{\sigma^2}\sum_{j=1}^{b}r_j\mathbf{m}_j$, we obtain
\begin{equation}
\label{eq:R_alpha_P_Q}
\begin{aligned}
    R_\alpha(\hat{P} || \hat{Q})
    = \frac{1}{1 - \alpha} \log \Biggl(\frac{1}{b^\alpha} \sum_{r_1 + \cdots + r_b = \alpha}
    \binom{\alpha}{r_1,\dots,r_b}&\exp \left(-\frac{1}{2\sigma^2}\sum\limits_{j = 1}^{b}r_j\vm_j^T\vm_j \right)\\
    &\times \exp\Bigg(
         \frac{1}{2\sigma^2} \bigg\|\sum\limits_{j = 1}^{b} r_j \vm_j\bigg\|^2_2
    \Bigg)\Biggr).
\end{aligned}
\end{equation}

We change the summation from partitions of $\alpha$ to all possible $b^{\alpha}$ tuples of $\alpha$ indices in $[1, b]$:
\begin{align*}
    R_\alpha(\hat{P} || \hat{Q})
    &= \frac{1}{1 - \alpha} \log \left(\frac{1}{b^\alpha} \sum_{(r_1, \cdots, r_b) \in [1, b]^\alpha}
    \exp \left(-\frac{1}{2\sigma^2}\sum\limits_{j = 1}^{\alpha}\vm_j^T\vm_j \right)
    \exp\left(
         \frac{1}{2\sigma^2} \left\|\sum\limits_{j = 1}^{\alpha} \vm_j\right\|^2_2
    \right)\right)\\
    &= \frac{1}{1 - \alpha} \log \left(\frac{1}{b^\alpha} \sum_{(r_1, \cdots, r_b) \in [1, b]^\alpha}
    \exp \left(\frac{1}{2\sigma^2}\sum\limits_{j_1 \ne j_2}^{\alpha}\vm_{r_{j_1}}^T\vm_{r_{j_2}} \right)\right)\\
    &=\frac{1}{1 - \alpha} \log \left(\frac{1}{b^\alpha} \sum_{(r_1, \cdots, r_b) \in [1, b]^\alpha}
    \exp \left(\frac{1}{2\sigma^2}\sum\limits_{j_1 \ne j_2}^{\alpha}\mathbf{G}_{r_{j_1}, r_{j_2}} \right)\right),
\end{align*}
where $\mathbf{G}$ denotes the Gram matrix of vectors $\vm_j$, concluding the proof.
\end{proof}

\RenyiDynamicProgram*
\begin{proof}
    
To compute the divergence efficiently using dynamic programming, we use the expression for $R_{\alpha}(\hat{P}||\hat{Q})$ from equation~\eqref{eq:R_alpha_P_Q}, substituting the Gram matrix $\mathbf{G}$:

\begin{equation}
\begin{aligned}
\label{eq:r_alpha_p_q_appendix}
     R_\alpha(\hat{P} || \hat{Q}) &= \log   \sum_{r_1 + \dots  + r_b = \alpha}
    \binom{\alpha}{r_1, r_2 \dots r_{b}}
    \exp\left(
        \frac{1 }{2 \sigma^2} \sum_{j=1}^{b}r_{j}(r_{j} - 1) \mathbf{G}_{j, j} + \frac{1}{\sigma^2}\sum_{j_1 < j_2}^{b} \mathbf{G}_{j_1, j_2}r_{j_1}r_{j_2}
    \right)\\
    & = \log   \sum_{r_1 + \dots  + r_b = \alpha}
    \frac{1}{r_1! \dots r_{b}!}
    \exp\left(
        \frac{1 }{2 \sigma^2} \sum_{j=1}^{b}r_{j}(r_{j} - 1) \mathbf{G}_{j, j} + \frac{1}{\sigma^2}\sum_{j_1 < j_2}^{b} \mathbf{G}_{j_1, j_2}r_{j_1}r_{j_2}
    \right)\\
    &\qquad + \log (\alpha!) .
\end{aligned}
\end{equation}

Next, we factor out the dependence on $r_b$, using that the matrix $\mathbf{G}$ is $p$ cyclic-banded:

\begin{equation}
\begin{aligned}
     &\log \sum\limits_{r_b = 0}^{\alpha}\Bigg[\frac{1}{r_{b}!} \exp\left(\frac{1}{2\sigma^2}(r_{b} - 1)r_b \mathbf{G}_{b, b} + \frac{1}{\sigma^2}\sum\limits_{j = b - p + 1}^{b - 1}r_{b} r_{j}\mathbf{G}_{j, b} + \frac{1}{\sigma^2}\sum\limits_{j = 1}^{\min(p - 1, b - p)}r_{b} r_{j}\mathbf{G}_{j, b}\right) \\
     &\quad \times \sum_{r_1 + \dots  + r_{b-1} = \alpha - r_b}
    \frac{1}{r_1! \dots r_{b - 1}!}
    \exp\left(
        \frac{1}{2 \sigma^2} \sum_{j=1}^{b - 1}r_{j}(r_{j} - 1) \mathbf{G}_{j, j} + \frac{1}{\sigma^2}\sum_{j_1 < j_2}^{b -1} \mathbf{G}_{j_1, j_2}r_{j_1}r_{j_2}
    \right)\Bigg] \\
    &\quad+ \log (\alpha!)
\end{aligned}
\end{equation}

This suggests a dynamic programming formulation. Define

\begin{equation}
    F(k, m, \mathbf{l}, \mathbf{r}) := \log \sum\limits_{\substack{r_1 + \dots + r_{k} = m\\ r_1, \dots r_k \ge 0\\ (r_{k - p +1}, \dots, r_{k}) = \mathbf{r}\\
    (r_1, \dots, r_{p - 1}) = \mathbf{l}\\}}\frac{1}{r_1! \dots r_{k}!}
    \exp\bigg(
        \frac{1 }{2 \sigma^2} \sum_{j=1}^{k}r_{j}(r_{j} - 1) \mathbf{G}_{j, j} + \frac{1}{\sigma^2}\sum_{j_1 < j_2}^{k} \mathbf{G}_{j_1, j_2}r_{j_1}r_{j_2}
    \bigg),
\end{equation}
where $\mathbf{l}$ is the prefix of length $p-1$ (empty if we run DP-SGD with $p=1$), and $\mathbf{r}$ is the suffix of length $p-1$ starting at position $k-p+1$.

If we compute all $O(b\alpha^{2p - 1})$ values of $F$, then the divergence is: 
\begin{equation}
\begin{aligned}
    R_\alpha(\hat{P} || \hat{Q}) &= \frac{1}{\alpha - 1} \log \sum\limits_{\|\mathbf{l}\cup \mathbf{r} \|_1 \le \alpha} \exp(F(B, \alpha, \mathbf{l}, \mathbf{r}))\exp \left(\frac{1}{\sigma^2}\sum\limits_{j_1 = 1}^{p - 1}\sum\limits_{j_2 =b - p + 2}^{b} \mathbf{G}_{j_1, j_2} l_{j_1} r_{j_2} \right)\\
    &\qquad+ \frac{\log (\alpha!)}{\alpha - 1} - \frac{\alpha \log b}{\alpha - 1}.
\end{aligned}
\end{equation}

The interaction for the indices $1\le j_1 < j_2 \le b$ are accounted in the dynamic function $F$, and to compute the whole sum we would have to account for the cyclic interaction between the last $p - 1$ indices and the first $p - 1$ indices. Note that we must fix the prefix $\mathbf{l}$; otherwise, when computing the cyclic interaction, we would obtain conflicting values for the first $p-1$ entries. Although this increases the computational complexity, we view it as an unavoidable cost of cyclic interaction.

To compute the dynamic values $F$, we use dynamic programming in a forward direction. We go in a cycle over all possible prefix values $\mathbf{l} = (r_1, \dots, r_{p - 1})$. Given a state $ (k, m, \mathbf{l}, \mathbf{r})$,
we update all successor states $(k + 1, m + r_{k+1}, \mathbf{l}, (r_{k - p + 3}, \dots, r_{k+1}))$ via
\begin{equation}
\begin{aligned}
    &F(k + 1, m + r_{k + 1}, \mathbf{l}, (r_{k - p + 3}, \dots, r_{k +1}))\\
    &\qquad\mathrel{{+}{=}} F(k, m,  \mathbf{l}, (r_{k - p + 2}, \dots r_{k})) - \log (r_{k + 1}!) + \frac{1}{2\sigma^2}(r_{k + 1} - 1)r_{k + 1} \mathbf{G}_{k + 1, k + 1}\\
    &\qquad\qquad+ \frac{1}{\sigma^2}\sum\limits_{j_1 = k - p + 2}^{k}r_{k + 1} r_{j_1}\mathbf{G}_{j_1, k + 1}.
\end{aligned}
\end{equation}

Each of $O(b \cdot \alpha \cdot \alpha^{p - 1} \cdot \alpha^{p - 1})$ states can be computed in $O(p)$ time, with $O(\alpha)$ forward transitions, resulting in overall runtime $O(bp \alpha^{2p})$. Regarding the memory requirements, we loop over the prefix values $\mathbf{l}$ and therefore do not need to store values for previously processed prefixes. Likewise, since we increase $k$ iteratively, we do not need to retain values from earlier $k$. Instead, we store all current suffixes (there are $O(\alpha^{p-1})$ of them), each with at most $O(\alpha)$ distinct partial sums, for a total of $O(\alpha^{p})$ memory. In addition, we must store the $p$-cyclic banded matrix $\mathbf{G}$, which requires $O(bp)$ memory, resulting in an overall requirement of $O(\alpha^{p} + bp)$ memory.

Algorithm~\ref{alg:renyi_dynamic_program} computes the sum in~\eqref{eq:r_alpha_p_q_appendix} via dynamic programming under the assumption that the Gram matrix $\mathbf{G}$ is $p$ cyclically banded, and otherwise adds the nonzero term $\frac{\tau\alpha}{2\sigma^2}$ to obtain an upper bound. Concretely, given a matrix $\mathbf{C}$, it first computes the Gram matrix $\mathbf{G}$, then forms the truncated $p$ cyclically banded matrix $\mathbf{G}^{(p)}$ together with the error term $\tau$.

The indices in the sum form a cyclic interaction. To compute the sum dynamically, we first break the cycle by looping over all prefixes $\mathbf{l}=(l_0,\dots,l_{p-2})$ whose total sum is less than the target sum $\alpha$. For each fixed prefix, we precompute the log sum of exponentials corresponding to the interactions among the first $p-1$ indices. We then expand the range of indices up to $b-1$ iteratively by running a forward dynamic program over $k=p-1,\dots,b-1$. At step $k$, the set of dynamic states is stored in a dictionary, where each state is identified by (i) the current suffix $\mathbf{r}$ of the previous $p-1$ values and (ii) the current partial sum $m$ of indices up to position $k-1$. From a state $(\mathbf{r},m)$ we propose a next value $t\in\{0,\dots,\alpha-m\}$, compute the contribution corresponding to interactions involving the new index $k$, and update the suffix, the partial sum, and the accumulated log value. For numerical stability, we aggregate contributions using the log-sum-exp (LSE) operation. The resulting states are stored in a temporary dictionary $\mathrm{new\_states}$, and $\mathrm{states}$ is updated after all admissible forward values have been considered. This completes the computation of all interactions between indices $0\le i<j\le b-1$ except for the cyclic terms. Finally, we close the cycle by adding the interactions between the final suffix and the fixed prefix in a last loop. Summing the resulting log values over all valid choices of the prefix $\mathbf{l}$ completes the computation.

\end{proof}

\RemoveRenyiBound*
\begin{proof}
If $(\hat{P},\hat{Q})$ is a dominating pair for the ``remove'' adjacency, then Lemma~29 of \citet{zhu2022optimal} implies that $(\hat{Q},\hat{P})$ is a dominating pair for the ``add'' adjacency.
The density of $\hat{P}$ is
\[
p(\mathbf{x}) = \frac{1}{b}\sum_{j=1}^{b} \phi(\mathbf{x}; \mathbf{m}_j, \sigma^2 I),
\]
and the density of $\hat{Q}$ is
\[
q(\mathbf{x}) = \phi(\mathbf{x}; \mathbf{0}, \sigma^2 I),
\]
where $\phi(\mathbf{x}; \mu, \Sigma)$ denotes the density of the multivariate Gaussian distribution.

Thus, for $\alpha>1$, the Rényi divergence in the remove direction is
\begin{equation}
    \mathrm{R}_{\alpha}(Q\|P)
    =
    \frac{1}{\alpha - 1}\log \int_{\mathbb{R}^{b}} \frac{q(\mathbf{x})^{\alpha}}{p(\mathbf{x})^{\alpha - 1}}\,d\mathbf{x}.
\end{equation}
The mixture density $p(\mathbf{x})$ appears in the denominator, which makes the integral difficult to compute in closed form. We therefore upper bound it.

By the arithmetic mean--geometric mean inequality (equivalently, Jensen's inequality for the exponential function),
\begin{align}
    p(\mathbf{x})
    &=
    \frac{1}{b}\sum_{j=1}^{b} \phi(\mathbf{x}; \mathbf{m}_j, \sigma^2 I)
    \;\ge\;
    \left(\prod_{j=1}^{b}\phi(\mathbf{x}; \mathbf{m}_j, \sigma^2 I)\right)^{1/b}
    \\
    &=
    \frac{1}{(2\pi \sigma^2)^{N/2}}
    \exp\!\left(
    -\frac{1}{2b\sigma^2}\sum_{j = 1}^{b}(\mathbf{x}- \mathbf{m}_j)^\top(\mathbf{x}- \mathbf{m}_j)
    \right),
\end{align}
and
\begin{equation}
    q(\mathbf{x})
    =
    \phi(\mathbf{x}; \mathbf{0}, \sigma^2 I)
    =
    \frac{1}{(2\pi \sigma^2)^{N/2}}
    \exp\!\left(-\frac{1}{2\sigma^2} \mathbf{x}^\top\mathbf{x}\right).
\end{equation}

Substituting the lower bound on $p(\mathbf{x})$ into the expression for $\mathrm{R}_{\alpha}(\hat{Q}\|\hat{P})$ yields
\begin{align}
    \mathrm{R}_{\alpha}(\hat{Q}\|\hat{P})
    &\le
    \frac{1}{\alpha - 1}
    \log\int_{\mathbb{R}^N}
    q(\mathbf{x})
    \exp\!\left(
    \frac{\alpha - 1}{2b\sigma^2}
    \sum_{j = 1}^b \bigl(-2\mathbf{x}^\top \mathbf{m}_j + \mathbf{m}_j^\top \mathbf{m}_j\bigr)
    \right)\,d\mathbf{x}
    \\
    &=
    \frac{1}{2b\sigma^2}\sum_{j = 1}^{b} \mathbf{m}_j^\top \mathbf{m}_j
    +
    \frac{1}{\alpha - 1}
    \log
    \mathbb{E}_{\mathbf{x}\sim Q}
    \exp\!\left\{
    -\frac{\alpha - 1}{b\sigma^2}
    \left(\sum_{j = 1}^{b} \mathbf{m}_j \right)^\top \mathbf{x}
    \right\}.
\end{align}

Using the moment generating function of a multivariate Gaussian, for $\mathbf{x}\sim \mathcal{N}(\mathbf{0}, \sigma^2 I)$ and any $\mathbf{t}\in\mathbb{R}^N$,
\[
\mathbb{E}\big[\exp(\mathbf{t}^\top \mathbf{x})\big]
=
\exp\!\left(\frac{\sigma^2}{2}\|\mathbf{t}\|_2^2\right),
\]
and taking $\mathbf{t} = -\frac{\alpha - 1}{b\sigma^2}\sum_{j=1}^{b}\mathbf{m}_j$, we obtain
\begin{align}
    \mathrm{R}_{\alpha}(\hat{Q}\|\hat{P})
    &\le
    \frac{1}{2b\sigma^2}\sum_{j = 1}^{b} \mathbf{m}_j^\top \mathbf{m}_j
    +
    \frac{1}{\alpha - 1}
    \cdot
    \frac{\sigma^2}{2}
    \left\|
    \frac{\alpha - 1}{b\sigma^2}
    \sum_{j = 1}^{b}\mathbf{m}_j
    \right\|_2^2
    \\
    &=
    \frac{1}{2b\sigma^2}\sum_{j = 1}^{b} \mathbf{m}_j^\top \mathbf{m}_j
    +
    \frac{\alpha - 1}{2b^2\sigma^2}
    \left\|
    \sum_{j = 1}^{b}\mathbf{m}_j
    \right\|_2^2.
\end{align}

Finally, let $G$ be the Gram matrix with entries $\mathbf{G}_{i,j} = \langle \mathbf{m}_i, \mathbf{m}_j\rangle$. Then
\[
\sum_{j=1}^{b}\mathbf{m}_j^\top\mathbf{m}_j = \sum_{j=1}^{b} \mathbf{G}_{j,j},
\qquad
\left\|\sum_{j=1}^{b}\mathbf{m}_j\right\|_2^2
=
\sum_{i=1}^{b}\sum_{j=1}^{b} \mathbf{G}_{i,j},
\]
which proves the claim.
\end{proof}

\begin{lemma}
\label{lem:sum-sharp-p-complete}
The sum
\begin{equation}
\label{eq:hard-sum}
\sum_{(r_1,\dots,r_\alpha)\in [b]^\alpha}
\exp\left(
  \sum_{\substack{j_1,j_2=1 \\ j_1\neq j_2}}^\alpha
  \frac{\mathbf{G}_{r_{j_1},r_{j_2}}}{2\sigma^2}
\right)
\end{equation}
appearing in Lemma~\ref{lem:renyi-bound} is $\#P$-complete to compute for a
general matrix $\mathbf{G}\in \mathbb{R}^{b\times b}$.
\end{lemma}

\begin{proof}
Suppose, for the sake of contradiction, that the sum in~\eqref{eq:hard-sum}
can be computed in polynomial time for an arbitrary matrix $\mathbf{G}$.

Let $H$ be a graph on $b$ vertices. We construct a matrix $\mathbf{G}$ from $H$
as follows:
\[
\mathbf{G}_{i,j}
=
\begin{cases}
0, & \text{if } i\neq j \text{ and } (i,j)\in E(H),\\
-\infty, & \text{otherwise}.
\end{cases}
\]
Then, for any tuple $(r_1,\dots,r_\alpha)\in[b]^\alpha$, the corresponding term
in~\eqref{eq:hard-sum} is equal to $1$ if the vertices
$r_1,\dots,r_\alpha$ are distinct and form an $\alpha$-clique in $H$, and is
equal to $0$ otherwise. Hence the sum in~\eqref{eq:hard-sum} counts the number
of ordered $\alpha$-cliques in $H$.

Since each unordered $\alpha$-clique contributes exactly $\alpha!$ ordered
tuples, computing~\eqref{eq:hard-sum} would allow us to count the number of
$\alpha$-cliques in $H$ in polynomial time. This is a $\#P$-complete problem.
Therefore, computing the sum in~\eqref{eq:hard-sum} is $\#P$-complete.
\end{proof}
\clearpage

\section{Proofs for Section 4}\label{appendix:proofs_conditional_composition}

In the following, we first confirm that the conditional composition~\cref{lemma:conditional_composition} immediately follows as a special case of the general conditional composition framework from~\cite{choquette2023privacy}. In the next sections, we then provide the omitted proofs for our instantiation of the conditional composition framework.

The general conditional composition result is:
\begin{theorem}[Theorem 3.1 in~\cite{choquette2023privacy}]\label{thm:conditioningtrick}
Let $\mathcal{M}_1: \mathcal{D} \rightarrow \mathcal{X}_1, \mathcal{M}_2: \mathcal{X}_1 \times \mathcal{D} \rightarrow \mathcal{X}_2, \mathcal{M}_3: \mathcal{X}_1 \times \mathcal{X}_2 \times \mathcal{D} \rightarrow \mathcal{X}_3, \ldots \mathcal{M}_N$ be a sequence of adaptive mechanisms, where each $\mathcal{M}_n$ takes a dataset in $\mathcal{D}$ and the output of mechanisms $\mathcal{M}_1, \ldots, \mathcal{M}_{n-1}$ as input. Let $\mathcal{M}$ be the mechanism that outputs $(x_1 = \mathcal{M}_1(D), x_2 = \mathcal{M}_2(x_1, D), \ldots, x_N = \mathcal{M}_N(x_1, \ldots, x_{N-1}, D))$. Fix any two adjacent datasets $D, D'$. 

Suppose there exists ``bad events'' $E_1 \subseteq \mathcal{X}_1, E_2 \subseteq \mathcal{X}_1 \times \mathcal{X}_2, \ldots... E_{N-1} \subseteq \mathcal{X}_1 \times \mathcal{X}_2 \times \ldots \times \mathcal{X}_{N-1}$ such that \[\Pr_{x \sim \mathcal{M}(D)}\left[\exists n: (x_1, x_2, \ldots x_n) \in E_n\right] \leq \delta\] and pairs of distributions $(P_1, Q_1), (P_2, Q_2), \ldots (P_N, Q_N)$ such that the PLD of $\mathcal{M}_1(D)$ and $\mathcal{M}_1(D')$ is dominated by the PLD of $P^{(1)}, Q^{(1)}$ and for any $n \geq 1$ and ``good'' output $(x_1, x_2, \ldots x_n) \notin E_n$, the PLD of $\mathcal{M}_{n+1}(x_1, \ldots, x_n, D)$ and $\mathcal{M}_{n+1}(x_1, \ldots, x_n, D')$ is dominated by the PLD of $P^{(n+1)}, Q^{(n+1)}$. Then for all $\epsilon$:

\[H_\epsilon(\mathcal{M}(D), \mathcal{M}(D')) \leq H_\epsilon\left(P^{(1)} \times P^{(2)} \times \ldots \times P^{(N)}, Q^{(1)} \times Q^{(2)} \times \ldots \times Q^{(N)}\right) + \delta.\]
\end{theorem}
In our case, we do not have a mechanism, but only its dominating pair $(P, Q)$ from~\cref{lem:dominating_pair}.
However, for any $D, D'$, we can define the dummy mechanism $\mathcal{M}(D) = P$ and $\mathcal{D}' = Q$ and $M_{n}(x_1,\dots,x_n, D) = P_{x_n \mid x_{1:n-1}}$ and $M_{n}(x_1,\dots,x_n, D') = Q_{x_n \mid x_{1:n-1}}$. The special case below then follows immediately.
\conditionalcomposition*

\subsection{Dominance criterion}\label{appendix:proofs_cond_comp_algorithm}

Our dominance criterion that relates stochastic dominance between mixture-of-Gaussians mechanisms to a pointwise comparison between reverse hazard functions is derived from the following result:
\begin{restatable}{lemma}{mogweightmonotonicity}[Corollary 4.4 in~\cite{choquette2023privacy}]
    \label{lemma:stochastic_dominance_to_dp_dominance}
    Define distributions $P = \sum_{i=1}^b \mathcal{N}(\mu_i, \sigma) p(z=i)$
    and $Q = \mathcal{N}(0, \sigma)$
    with  $0 \leq \mu_b \leq \dots \leq \mu_b$.
    Let $P' = \sum_{i=1}^b \mathcal{N}(\mu_i, \sigma) p'(z=i)$.
    Then $(P, Q) \prec (P', Q)$ and $(Q, P) \prec (Q, P')$ whenever
    $p(z)$ is stochastically dominated by $p'(z)$, i.e.,
    $P(z \geq i) \leq P'(z \geq i)$ for all $i \in [b]$.
\end{restatable}

Next, we show that if the reverse hazard function associated with mixture weights $p(z)$ is pointwise smaller than
that associated with $p(z')$, then we have stochastic dominance, i.e., the criterion in~\cref{lemma:stochastic_dominance_to_dp_dominance} is fulfilled:

\begin{restatable}{lemma}{reversehazarddominance}\label{lemma:reverse_hazard_to_stochastic_dominance}
    Given $p, p' : [b] \rightarrow [0,1] $, 
    define the corresponding reverse hazard functions at $i$ as  $\lambda_i = P(z = i \mid z \leq i)$ and $\lambda'_i = P'(z = i \mid z \leq i)$.
    If $\lambda_i \leq \lambda_i'$ for all $i \in [b]$, then  $p(z)$ is stochastically dominated by $p'$.
\end{restatable}
\begin{proof}
    Let $F_i = P(z \leq i)$ denote the CDF of $z$. By definition, $\lambda_i = P(z=i \mid z \leq i) = P(z=i)/F_i$.
    Furthermore, $F_{i} = F_{i+1} - P(z=i+1)$.
    It thus follows that $F_{i} = F_{i+1} - P(z=i+1) = F_{i+1}(1 - \lambda_{i+1})$. Since $F_i=1$ for any $i \geq b$, recursive application of this equality yields $F_i = \prod_{j=i+1}^b (1 - \lambda_j)$.
    The condition $\lambda_j \leq \lambda'_j$ implies $(1 - \lambda_j) \geq (1 - \lambda'_j)$ for all $j$. Consequently, $F_i \geq F'_i$ for all $i$, i.e., $p(z)$ is stochastically dominated by $p'(z)$.
\end{proof}

Our result then immediately follows from combining the previous two lemmata:

\reversehazardtodominance*
\begin{proof}
    Instantiate~\cref{lemma:reverse_hazard_to_stochastic_dominance}
    with 
    $\lambda_i = P(z=i \mid z \leq i, \vy_{1:n-1})$
    and
    $\lambda'_i = P^{(n)}(z=i \mid z \leq i )$.
    Having $\forall i[b]: \lambda_i \leq \lambda'_i$ then implies
    stochastic dominance via~\cref{lemma:reverse_hazard_to_stochastic_dominance},
    i.e.
    $\forall i \in [b]: P(z \geq i \mid P(z = i \mid \vy_{1:n-1}) \leq P(z \geq i)$.
    Thus, the conditions from~\cref{lemma:stochastic_dominance_to_dp_dominance} are met,
    i.e., this stochastic dominance translated to dominance in the differential privacy sense.
\end{proof}

\subsection{Analytical bounds}
First, we prove that the reverse hazard function underlying our dominance criterion
is equivalent to a ternary form of privacy loss between Gaussian mixtures.
Then, we prove our analytical tail bounds on this privacy loss.

\reversehazardloss*
\begin{proof}
    Let $p(\vy_{1:n-1} \mid z=j) = \mathcal{N}(\vy_{1:n-1} \mid \bm{m}_{j,:n-1}, \sigma^2 \eye)$ denote the likelihood of the earlier outcome $\vy_{1:n-1}$ conditioned on mixture component $j$.
    Let $p(z=j) = \frac{1}{b}$ denote the uniform prior that arises from random allocation with $b$ batches per epoch.
    Using the definition of conditional probability and Bayes' rule, the conditional reverse hazard function is:
    \begin{align*}
        &\Pr(z=i \mid z \leq i, \vy_{1:n-1})\\
        = &\frac{p(\vy_{1:n-1} \mid z=i)  p(z=i)}{\sum_{j=1}^i p(\vy_{1:n-1} \mid z=j)  p(z=j)}\\
        =& \frac{p(\vy_{1:n-1} \mid z=i)}{\sum_{j=1}^i p(\vy_{1:n-1} \mid z=j)}\\
        = & \left(1 + \frac{\sum_{j=1}^{i-1} p(\vy_{1:n-1} \mid z=j)}{p(\vy_{1:n-1} \mid z=i)}\right)^{-1}.
    \end{align*}
    Recall that $L_i(\vy_{1:n-1}) = \log \left( \frac{\frac{1}{i-1}\sum_{j=1}^{i-1} p(\vy_{1:n-1} \mid z=j)}{p(\vy_{1:n-1} \mid z=i)} \right)$.
    Multiplying and dividing the second summand by $i-1$ yields:
    \begin{equation}
        \frac{\sum_{j=1}^{i-1} p(\vy_{1:n-1} \mid z=j)}{p(\vy_{1:n-1} \mid z=i)}
        = (i-1) \exp(L_i(\vy_{1:n-1}))
        = \exp(\log(i-1) + L_i(\vy_{1:n-1})).
    \end{equation}
    Applying the definition of the sigmoid function $s(x) = (1+e^{-x})^{-1}$, we obtain:
    \begin{equation}
        \Pr(z=i \mid z \leq i, \vy_{1:n-1}) = s\left(-\big(\log(i-1) + L_i(\vy_{1:n-1})\big)\right),
    \end{equation}
    which completes the proof.
\end{proof}

\tailamgmadd*
\begin{proof}
    This result directly follows as a special case of the general variational
    tail bound~\cref{thm:add_tail_bound}
    by defining $\vmu_i = \vm_{i,1:n-1}$ for all $i \in [b]$
    and using variational distribution $\psi \sim \mathrm{Uniform}(\{1,\dots,i-1\})$.
    In particular, this causes the term $\mathrm{KL}(\psi || \pi)$ to vanish.

    Alternatively, this result could be proven by (1) applying the arithmetic-mean geometric-mean inequality
    to the logarithm in~\cref{eq:gaussian_mixture_privacy_loss}
    to obtain a 
    pointwise lower bound,
    (2) noticing that the privacy loss between any two Gaussian mixture components is an affine function of $\vy_{:n-1}$,
    and finally
    (3) use the standard formula for the distribution of affine transformations of a multivariate Gaussians.
\end{proof}

\begin{theorem}\label{theorem:tail_amgm_remove}
    Let $\vy \sim R$ with $R = Q$ (``add')
    as defined in~\cref{lem:dominating_pair}.
    Define standard normal CDF $\Phi : \sR \rightarrow [0,1]$
    and $\psi = \mathrm{Uniform}([i-1])$
    Then $\Pr[L_i(\vy_{1:n-1}) \leq \tau_i] \leq \frac{1}{b} \sum_{k=1}^b \Phi\left(\frac{\tau_i - \nu_k}{\xi_k}\right)$ with 
    \begin{align*}
        \nu_k &= \frac{1}{\sigma^2} \vm_{k,:n-1}^T (\mathbb{E}_{j \sim \psi}[\vm_{j:n-1}] - \vm_{i,:n-1}) + \frac{1}{2\sigma^2} \left( \|\vm_{i,:n-1}\|_2^2 - \mathbb{E}_{j \sim \psi}[\|\vm_{j, :n-1}\|_2^2] \right), \\
        \xi_k &= \frac{1}{\sigma} \| \vm_{i,:n-1} - \mathbb{E}_{j \sim \psi}[\vm_{j,:n-1}] \|_2.
    \end{align*}
\end{theorem}
\begin{proof}
    This result directly follows as a special case of the general variational
    tail bound~\cref{thm:remove_tail_bound}
    by defining $\vmu_i = \vm_{i,1:n-1}$ for all $i \in [b]$
    and using variational distribution $\psi \sim \mathrm{Uniform}(\{1,\dots,i-1\})$.
    In particular, this causes the term $\mathrm{KL}(\psi || \pi)$ to vanish.
\end{proof}

\begin{remark}
    Note that since the ``remove'' bound is a weighted sum of Gaussian CDFs, we can no longer find $\tau_i$ for a desired significance analytically
via the inverse CDF.
Due to monotonicity of standard normal CDF $\Psi$, we can nevertheless find $\tau_i$ with arbitrary precision via bisection. 

\end{remark}

\subsection{Correctness of Algorithm 2}

While~\cref{section:bad_outcome_from_dominating_pairs} in its entirety already served as a proof sketch,
we now provide a formal proof for the correctness of~\cref{algorithm:conditional_composition}:
\condcompbothcorrect*
\begin{proof}
    \cref{algorithm:conditional_composition} allocates a failure probability $\beta_i$ to each component $i$ such that $\sum_{i} \beta_i = \beta$.
    For each $i$, the subroutine \texttt{whp\_lower} returns a threshold $\tau_i$ such that $\Pr[L_i(\vy_{1:n-1}) < \tau_i] \leq \beta_i$, where the probability is over the randomness of the privacy loss random variable $L_{\tilde{P},\tilde{Q},\tilde{R}}$ generated by the algorithm's reference distribution $\tilde{R}$.
    In the ``remove'' case, $\tilde{R}$ is the joint distribution $P_{\vy_{1:n-1}}$ of the first $n-1$ outcomes of $\vy \sim P$.
    In the ``add'' case, $\tilde{R}$ is the joint distribution $Q_{\vy_{1:n-1}}$ of the first $n-1$ outcomes of $\vy \sim Q$.

    Define the failure event $\overline{A} = \bigcup_i \{ L_i(\vy_{1:n-1}) < \tau_i \}$
    with $L_i$ defined in~\cref{lemma:reverse_hazard_loss}. By a union bound, $\Pr(\overline{A}) \leq \sum \beta_i = \beta$.
    
    Conditioned on the success event $A = \overline{\overline{A}}$, we have $L_i(\vy_{1:n-1}) \geq \tau_i$ for all $i$.
    Consequently, by the monotonicity of the sigmoid function and the derivation in \cref{lemma:reverse_hazard_loss}, the true conditional reverse hazard satisfies:
    \begin{equation}
        \Pr(z=i \mid z \leq i, \vy_{1:n-1}) = s(-\log(i-1) - L_i) \leq s(-\log(i-1) - \tau_i) = \overline{\lambda_i}.
    \end{equation}
    Since the means are sorted non-decreasingly (Line 3 of~\cref{algorithm:conditional_composition}), the conditions of \cref{lemma:reverse_hazard_to_stochastic_dominance} (larger reverse hazard implies dominance in the differential privacy sense) are satisfied. Thus, in the ``remove'' case the returned pair $(P^{(n)}, Q^{(n)})$ dominates the conditional distributions $P_{y_n | \tilde{y}_{1:n-1}}$
    and $Q_{y_n | \tilde{y}_{1:n-1}}$ with probability at least $1-\beta$ under $P$.
    In the ``add'' case, the returned pair $(Q^{(n)}, P^{(n)})$ dominates the conditional distributions  $Q_{y_n | \tilde{y}_{1:n-1}}$
    and $P_{y_n | \tilde{y}_{1:n-1}}$ with probability at least $1-\beta$ under $Q$.
\end{proof}

\subsection{Runtime complexity of Algorithm 2}\label{appendix:condcomp_runtime_proof}
In the following, we explain in more detail why executing~\cref{algorithm:conditional_composition} once for every iteration $n \in \{1,\dots,N\}$ leads to an overall complexity of $\mathcal{O}(N^2 b^2)$. Then, we explain why subsequent evaluations for different noise multipliers $\sigma \in sR_+$ can be done in  $\mathcal{O}(N b^2)$ (``remove'' direction) or $\mathcal{O}(N b)$ (``add'' direction).

\textbf{Initial call complexity.}
When executing~\cref{appendix:proofs_cond_comp_algorithm} for a specific iteration $n$,
we first sort the entries of $\vm_{:,n} \in \sR^b$, i.e., we have an initial cost of $\mathcal{O}(b \log b)$.
Then, we calculate one privacy loss tail bound per mixture component $i \in \{1,\dots,b\}$ using bisection with $\kappa$ steps on the analytical bound from~\cref{theorem:tail_amgm_add,theorem:tail_amgm_remove}.
The main computational cost of evaluating these bounds are the two expectations
$\mathbb{E}_{j \sim \psi^{(i)}}[\vm_{j:n-1}]$ and $\mathbb{E}_{j \sim \psi^{(i)}}[\|\vm_{j, :n-1}\|_2^2]$ 
with $\psi^{(i)} = \mathrm{Uniform}([i-1])$. Since they are weighted prefix sums of each other, all of them can be calculated jointly for all $i \in [b]$ at an overall cost of $\mathcal{O}(N b)$.
In the ``add'' case, we can then evaluate $\nu$ and $\xi$ once at a cost of $\mathcal{O}(N)$ and
then do bisection on 
$\Phi\left(\frac{\tau_i -\nu}{\xi}\right)$ at a constant cost $\mathcal{O}(\kappa)$.
In the ``remove'' case, we can evaluate all $\left(\nu_k\right)_{k=1}^b$ and $\left(\xi_k\right)_{k=1}^b$ once at a cost of $\mathcal{O}(N \cdot b)$
and then di bisection  $\frac{1}{b}\sum_{k=1}^b \Phi\left(\frac{\tau_i -\nu_k}{\xi_k}\right)$ at a cost of $\mathcal{O}(b \kappa)$.
Thus, the cost for a specific iteration $n$ is:
\begin{itemize}
    \item Add: $\mathcal{O}(b \log b) + N b + N b + b \kappa) = \mathcal{O}(N b)$,
    \item Remove:  $\mathcal{O}(b \log b) + N b + N b^2 + b^2 \kappa) = \mathcal{O}(N b^2)$.
\end{itemize}
since the number of participation patterns $b$ is never larger than the number of training steps $N$.
Multipliying by the overall number of iterations than yields $\mathcal{O}(N^2 b)$ and $\mathcal{O}(N^2 b^2)$ for the ``add'' and ``remove'' direction.

\textbf{Subsequent call complexity}
As can be seen from our previous discussion, the computational cost for each iteration $n \in [N]$ and component $i \in [b]$
is dominated by the cost of calculating 
$\nu$ and $\xi$ (``add'' case) or $\left(\nu_k\right)_{k=1}^b$ and $\left(\xi_k\right)_{k=1}^b$ (``remove direction'')
because it involves vector operations along the $n$-dimension.
However, these terms only depend on noise multiplier $\sigma$ via multiplication.
Thus, after we have computed them once, we can simply scale the existing values for any new $\sigma'$.
This reduces the runtime for a specific iteration $n$ to:
\begin{itemize}
    \item Add: $\mathcal{O}(b \log b + b \kappa) = \mathcal{O}(b \log b)$
    \item Remove: $\mathcal{O}(b \log b + b^2 \kappa) = \mathcal{O}(b^2)$
\end{itemize}
Thus, the cost for evaluation across all iterations is $\mathcal{O}(N b \log b)$ and $\mathcal{O}(N b^2)$, respectively.

\clearpage

\section{Variational tail bounds on ternary privacy loss distribution}\label{appendix:proofs_variational}

In the following, we show that we can attain a wider (and, as shown in~\cref{appendix:choice_of_variational_family} potentially tighter) range of privacy loss distribution tail bounds than the AM-GM that can be used to derive our previous reuslts.
The proof strategy is as follows: We apply Jensen's inequality / a pointwise variational bound to the privacy loss,
observe that the resultant function is an affine transformation of $\vy \sim P$ or $\vy \sim Q$, and then determine the distribution of this Gaussian or Gaussian mixture after transformation.

\begin{restatable}{lemma}{variationalgeneric}\label{lemma:variational_generic}
    Let $\tilde{P} = \sum_{j=1}^{i-1} \mathcal{N}(\vmu_j, \sigma^2 \eye) \cdot \pi_i$
    with weights $\pi_i = \frac{1}{i-1}$ and $\tilde{Q} = \mathcal{N}(\vmu_i, \sigma^2 \eye)$.
    Let $L_i(\vx) = \log( \frac{\dd \tilde{P}}{\dd \tilde{Q}})(\vx)$ be their privacy loss at $\vx$.
    Define per-component privacy losses $
    L_{i,j}(\vx) = \log\left(\frac{ \mathcal{N}(\vx \mid \vmu_j, \sigma^2 \eye)}{\mathcal{N}(\vx \mid \vmu_i, \sigma^2 \eye)}\right)$. 
    Let $\psi$ be any categorical distribution  on $[i-1]$. Then, $L_i(\vx) \geq \underline{L_i}(\vx)$ with
    \begin{equation}\label{eq:variational_bound}
        \underline{L_i}(\vx) = \mathbb{E}_{j \sim \psi}[L_{i,j}] - \mathrm{KL}(\psi || \pi).
    \end{equation}
\end{restatable}
\begin{proof}
    Let $p$ be the density of $P$ and $q$ be the density of $Q$.
    By definition of the privacy loss $L_i$ and the product rule for logarithms, we have
    \begin{equation*}
        L_i(\vx) = \log \left( \frac{p(\vx)}{q(\vx)} \right) =
        \log \left( \sum_{j=1}^{i-1} \mathcal{N}(\vx \mid \vmu_j, \sigma^2 \eye) \pi_i \right)
        - \log(\mathcal{N}(\vx \mid \vmu_i, \sigma^2 \eye).
    \end{equation*}
    The first term is simply the log-likelihood of a Gaussian mixture model.
    It is, quasi by definition, lower-bounded by the evidence lower bound (for the formal proof, see, e.g., Section 2.2 of~\cite{blei2017variational}).
    We thus have
    \begin{align*}
        L_i(\vx)
        & \geq \mathbb{E}_{j \sim \psi}\left[\log \left( \mathcal{N}(\vx \mid \vmu_j, \sigma^2 \eye) \right)\right] - \mathrm{KL}( \psi || \pi) - \log(\mathcal{N}(\vx \mid \vmu_i, \sigma^2 \eye)
        \\
        &
        =
         \mathbb{E}_{j \sim \psi}\left[\log \left( \mathcal{N}(\vx \mid \vmu_j, \sigma^2 \eye) \right)  - \log(\mathcal{N}(\vx \mid \vmu_i, \sigma^2 \eye) \right] - \mathrm{KL}( \psi || \pi)
        \\
        &
        =
         \mathbb{E}_{j \sim \psi}\left[ L_{i,j}(\vx)\right] - \mathrm{KL}( \psi || \pi)
         \\
         &
         = \underline{L_i}(\vx),
    \end{align*}
    where the first equality is due to linearity of expectation.
\end{proof}

Next, we use this bound on the privacy loss $L_i$ to derive a lower tail bound on $L_i(\vx)$ with $\vx \sim \tilde{R}$ where $\tilde{R}$ is an arbitrary Gaussian mixture.
The concrete instantiations for the ``remove'' and ``add'' $\tilde{R}$ from~\cref{algorithm:conditional_composition} will then follow as special cases.
\begin{lemma}\label{lemma:variational_distribution_general}
    Let $\vx \sim \tilde{R}$ with $\tilde{R} = \sum_{k=1}^K \mathcal{N}(\vrho_k, \sigma^2 \eye) \frac{1}{K}$.
    Let $\Psi = \{\psi^{(1)}, \dots, \psi^{(K)}\}$ be a collection of variational distributions on $[i-1]$.
    Consider the random variable $Y$ generated by sampling a component index $k \sim \mathrm{Uniform}([K])$, sampling $\vx \sim \mathcal{N}(\vrho_k, \sigma^2 \eye)$, and computing the component-specific lower bound $\underline{L}_i(\vx; \psi^{(k)})$ as defined in~\cref{eq:variational_bound}.
    Then $Y \sim \sum_{k=1}^K \mathcal{N}(\nu_k, \xi_k) \frac{1}{K}$ with component-specific means $\nu_k \in \sR$
    and standard deviations $\xi_k \in \sR_+$ defined by:
    \begin{align*}
        \nu_k &= \frac{1}{\sigma^2} \vrho_k^T (\mathbb{E}_{j \sim \psi^{(k)}}[\vmu_j] - \vmu_i) + \frac{1}{2\sigma^2} \left( \|\vmu_i\|_2^2 - \mathbb{E}_{j \sim \psi^{(k)}}[\|\vmu_j\|_2^2] \right) - \mathrm{KL}(\psi^{(k)} \mid\mid \pi), \\
        \xi_k &= \frac{1}{\sigma} \| \vmu_i - \mathbb{E}_{j \sim \psi^{(k)}}[\vmu_j] \|_2.
    \end{align*}
\end{lemma}
\begin{proof}
    We first derive a bound for each component $k$. Recall that the per-component privacy loss is $L_{i,j}(\vx) = \frac{1}{\sigma^2} \vx^T (\vmu_j - \vmu_i) + \frac{1}{2\sigma^2} \left( \|\vmu_i\|_2^2 - \|\vmu_j\|_2^2 \right)$.
    For a specific mixture component $k$, we use the variational distribution $\psi^{(k)}$. By linearity of expectation, the lower bound $\underline{L}_i(\vx; \psi^{(k)}) = \mathbb{E}_{j \sim \psi^{(k)}} [L_{i,j}(\vx)] - \mathrm{KL}(\psi^{(k)} \mid\mid \pi)$ from~\cref{lemma:variational_generic} can be expressed as an affine transformation:
    \begin{equation*}
        \underline{L}_i(\vx; \psi^{(k)}) = \vx^T \underbrace{\left[ \frac{1}{\sigma^2} (\mathbb{E}_{j \sim \psi^{(k)}}[\vmu_j] - \vmu_i) \right]}_{\mathbf{a}_k} + \underbrace{\frac{1}{2\sigma^2} \left( \|\vmu_i\|_2^2 - \mathbb{E}_{j \sim \psi^{(k)}}[\|\vmu_j\|_2^2] \right) - \mathrm{KL}(\psi^{(k)} \mid\mid \pi)}_{c_k}.
    \end{equation*}
    Conditioned on the latent component $z=k$, the vector $\vx$ follows $\mathcal{N}(\vrho_k, \sigma^2 \eye)$. Thus, the conditional distribution of the lower bound is univariate Gaussian with mean $\nu_k$ and variance $\xi_k^2$.
    The mean can be calculated via:
    \begin{align*}
        \nu_k &= \mathbb{E}_{\vx \sim \mathcal{N}(\vrho_k, \sigma^2 \eye)} [\vx^T \mathbf{a}_k + c_k] = \vrho_k^T \mathbf{a}_k + c_k \\
        &= \frac{1}{\sigma^2} \vrho_k^T (\mathbb{E}_{j \sim \psi^{(k)}}[\vmu_j] - \vmu_i) + \frac{1}{2\sigma^2} \left( \|\vmu_i\|_2^2 - \mathbb{E}_{j \sim \psi^{(k)}}[\|\vmu_j\|_2^2] \right) - \mathrm{KL}(\psi^{(k)} \mid\mid \pi).
    \end{align*}
    The variance $\xi_k^2$ can be calculated via:
    \begin{align*}
        \xi_k^2 &= \mathrm{Var}(\vx^T \mathbf{a}_k + c_k \mid \vx \sim \mathcal{N}(\vrho_k, \sigma^2 \eye)) = \mathbf{a}_k^T (\sigma^2 \eye) \mathbf{a}_k \\
        &= \sigma^2 \left\| \frac{1}{\sigma^2} (\mathbb{E}_{j \sim \psi^{(k)}}[\vmu_j] - \vmu_i) \right\|_2^2 \\
        &= \frac{1}{\sigma^2} \| \mathbb{E}_{j \sim \psi^{(k)}}[\vmu_j] - \vmu_i \|_2^2.
    \end{align*}
    Taking the square root yields $\xi_k = \frac{1}{\sigma} \| \vmu_i - \mathbb{E}_{j \sim \psi^{(k)}}[\vmu_j] \|_2$. Since the mixture component $k$ is chosen uniformly at random with probability $1/K$, the marginal distribution is the mixture $\frac{1}{K} \sum_{k=1}^K \mathcal{N}(\nu_k, \xi_k^2)$.
\end{proof}
Applying this result to the ``remove'' case, we have:
\begin{theorem}\label{thm:remove_tail_bound}
    Let $\vx \sim \tilde{R}$ with $\tilde{R} = \frac{1}{b} \sum_{k=1}^b \mathcal{N}(\vmu_k, \sigma^2 \eye)$.
    Let $\Psi = \{\psi^{(1)}, \dots, \psi^{(b)}\}$ be a collection of variational distributions on $[i-1]$.
    For any threshold $\tau \in \sR$, the left-tail probability of the privacy loss random variable 
    $L_i(\vx)$ with $L_i$ defined in~\cref{lemma:reverse_hazard_loss} is bounded by:
    \begin{equation}
        \Pr_{\vx \sim \tilde{R}}[L_i(\vx) \leq \tau] \leq \frac{1}{b} \sum_{k=1}^b \Phi\left( \frac{\tau - \nu_k}{\xi_k} \right),
    \end{equation}
    where $\Phi$ is the cumulative distribution function of the standard normal distribution, and per-component parameters are:
    \begin{align*}
        \nu_k &= \frac{1}{\sigma^2} \vmu_k^T (\mathbb{E}_{j \sim \psi^{(k)}}[\vmu_j] - \vmu_i) + \frac{1}{2\sigma^2} \left( \|\vmu_i\|_2^2 - \mathbb{E}_{j \sim \psi^{(k)}}[\|\vmu_j\|_2^2] \right) - \mathrm{KL}(\psi^{(k)} \mid\mid \pi), \\
        \xi_k &= \frac{1}{\sigma} \| \vmu_i - \mathbb{E}_{j \sim \psi^{(k)}}[\vmu_j] \|_2.
    \end{align*}
\end{theorem}
\begin{proof}
    By the Law of Total Probability, $\Pr_{\vx \sim \tilde{R}}[L_i(\vx) \leq \tau] = \frac{1}{b} \sum_{k=1}^b \Pr_{\vx \sim \mathcal{N}(\vmu_k, \sigma^2 \eye)} [L_i(\vx) \leq \tau]$.
    Since $L_i(\vx) \geq \underline{L}_i(\vx; \psi^{(k)})$, we have the bound
    $\Pr[L_i(\vx) \leq \tau] \leq \Pr[\underline{L}_i(\vx; \psi^{(k)}) \leq \tau]$ for each component $k$.
    Applying \cref{lemma:variational_distribution_general} with $\vrho_k = \vmu_k$, the variable $\underline{L}_i(\vx; \psi^{(k)})$ conditioned on component $k$ follows $\mathcal{N}(\nu_k, \xi_k)$.
    Consequently, each term in the sum is exactly $\Phi\left( \frac{\tau - \nu_k}{\xi_k} \right)$.
\end{proof}
For the ``add'' case, we observe that the lower bound on our privacy loss is a single, univariate Gaussian:
\begin{restatable}{lemma}{variationalboundadd}\label{theorem:variational_bound_add}
        Let $\vx \sim \tilde{R}$ with $\tilde{R} = \mathcal{N}(\vzero, \sigma^2 \eye)$.
        Then $\underline{L_i}(\vx) \sim \mathcal{N}(\nu, \xi)$ with mean and standard deviation
        \begin{align*}
            \nu &= \left(||\vmu_i||_2^2 - \mathbb{E}_{j \sim \psi}[||\vmu_j||_2^2]\right) \mathbin{/} (2 \sigma^2) - \mathrm{KL}(\psi || \pi), \\
            \xi &=   \left\lvert\left\lvert \vmu_i - \mathbb{E}_{j \sim \psi}[\vmu_j] \right\rvert\right\rvert_2 \mathbin{/} \sigma .
        \end{align*}
\end{restatable}
\begin{proof}
    This immediately follows as a special case of~\cref{lemma:variational_distribution_general} with $\vrho_k = \vzero$ and $K=1$.
\end{proof}
Naturally, this result directly implies the following analytical tail bound for the ``add'' case:
\begin{theorem}\label{thm:add_tail_bound}
    Let $\vx \sim \tilde{R}$ with $\tilde{R} = \mathcal{N}(\vzero, \sigma^2 \eye)$.
    Let $\psi$ be an  arbitrary variational distribution on $[i-1]$.
    For any threshold $\tau \in \sR$, the left-tail probability of the privacy loss random variable 
    $L_i(\vx)$ with $L_i$ defined in~\cref{lemma:reverse_hazard_loss} is bounded by:
    \begin{equation}
        \Pr_{\vx \sim \tilde{R}}[L_i(\vx) \leq \tau] \leq \Phi\left( \frac{\tau - \nu}{\xi} \right),
    \end{equation}
    where $\Phi$ is the cumulative distribution function of the standard normal distribution, with parameters
    \begin{align*}
        \nu &= \frac{1}{2\sigma^2} \left( \|\vmu_i\|_2^2 - \mathbb{E}_{j \sim \psi}[\|\vmu_j\|_2^2] \right) - \mathrm{KL}(\psi \mid\mid \pi), \\
        \xi &= \frac{1}{\sigma} \| \vmu_i - \mathbb{E}_{j \sim \psi}[\vmu_j] \|_2.
    \end{align*}
\end{theorem}
\begin{remark}[Handling Identical Components]
    In structured matrix mechanisms, multiple batches often share the same mixture mean, i.e., $\vmu_j = \vmu_i$. For these components, the privacy loss $L_{i,j}(\vx)$ is identically zero for all $\vx$. Including them in the variational mean $\nu$ is suboptimal as it pulls the lower bound toward zero. We handle this by setting the variational weights $\psi_j = 0$ for all such indices. This effectively pulls the zero-loss components out of the expectation and into the KL divergence term, where they contribute a constant offset to the mean $\nu$.
\end{remark}

\clearpage

\section{Evaluation of variational bounds}\label{appendix:choice_of_variational_family}
In the following, we conduct an initial exploration of whether and how the variational privacy loss tail bounds from~\cref{appendix:proofs_variational} may improve upon the basic AM-GM tail bounds from~\cref{appendix:proofs_conditional_composition}.
In particular, we consider different natural variational families and empirically compare them in terms of tightness of their reverse hazard bounds and bounds on the final privacy profile after conditional composition.

Recall that we need to find a lower tail bound on the ternary privacy loss  random variable $L_{\tilde{P}, \tilde{Q}, \tilde{R}}$,
which is defined as $L_i(\vx)$ with $L_i(\vx) = \log\left(\frac{\dd \tilde{P}}{\dd \tilde{Q}}\right)(\vx)$ and $\vx \sim \tilde{R}$,
where $\tilde{P}$ is a uniform Gaussian mixture with means $\vmu_1,\dots,\vmu_{i-1}$,
$\tilde{Q}$ is a single Gaussian with mean $\vmu_i$,
and $\tilde{R}$ is either:
\begin{itemize}
    \item $\sum_{k=i}^{b} \mathcal{N}(\mmu_k, \sigma^2 \eye_{n-1}) \cdot \frac{1}{b}$ (``remove'' case),
    \item or $\mathcal{N}(\vzero, \sigma^2 \eye_{n-1})$ (``add'' case), which is equivalent to the trivial mixture
    $\sum_{k=i}^{b} \mathcal{N}(\vzero, \sigma^2 \eye_{n-1}) \cdot \frac{1}{b}$.
\end{itemize}
Let $\vrho_k \in \sR^{n-1}$ denote an arbitrary component mean of $\tilde{R}$ (either $\vmu_k$ or the all-zeros vector $\vzero$).
From~\cref{lemma:variational_distribution_general,thm:remove_tail_bound,thm:add_tail_bound}, we know that the tail probability conditioned on component $k$ is upper-bounded via
\begin{equation}\label{eq:single_component_tail_bound}
    \Pr_{\vx \sim \mathcal{N}(\vrho_k, \sigma^2 \eye)}[L_i(\vx) \leq \tau] \leq \Phi\left( \frac{\tau - \nu_k}{\xi_k} \right)
\end{equation}
with scalar mean $\nu_k \in \sR$ and scalar standard deviation $\xi_k \in \sR_+$
with
\begin{align}
    \nu_k &= \frac{1}{\sigma^2} \vmu_k^T (\mathbb{E}_{j \sim \psi^{(k)}}[\vmu_j] - \vmu_i) + \frac{1}{2\sigma^2} \left( \|\vmu_i\|_2^2 - \mathbb{E}_{j \sim \psi^{(k)}}[\|\vmu_j\|_2^2] \right) - \mathrm{KL}(\psi^{(k)} \mid\mid \pi),
    \label{eq:variational_per_component_mean}
    \\
    \xi_k &= \frac{1}{\sigma} \| \vmu_i - \mathbb{E}_{j \sim \psi^{(k)}}[\vmu_j] \|_2,
    \label{eq:variational_per_component_std}
    \\
    \pi &= \mathrm{Uniform}(\{1,\dots,i-1\}),
\end{align}
for any variational distribution $\psi^{(k)}$ over mixture indices $\{1,\dots, i-1\}$.
Therefore, we can define an entire variational family $\Psi^{(k)}$ and choose a $\psi^{(k)} \in \Psi^{(k)}$ that leads to the tightest tail bound.

In principle, we could let $\Psi^{(k)}$ be the set of all distributions supported on $\{1,\dots, i-1\}$ and explicitly optimize $\psi^{(k)}$ (e.g., via gradient descent, which would be an interesting direction for future work).
However, for $k \in \sN$ epochs and $b \in \sN$ batches per epoch and thus $N = k \cdot b$, we would need to solve
$N \cdot b^2$ such optimization problems (one for each step $n$, each privacy loss $L_i$ with $i \in [b]$, and  each component $k \in [b]$ of $\tilde{R})$. 
We therefore consider the following heuristically chosen variational families instead:

\subsection{Option 1: Maximizing the expectation based on KL divergence}
A heuristic for tightening the bound in~\cref{eq:single_component_tail_bound}
is by maximizing the mean $\nu_k$.
In turn, a heuristic for doing so is to eliminate the  KL-divergence term in~\cref{eq:variational_per_component_mean}
by letting
\begin{equation*}
    \psi^{(k)} = \pi = \mathrm{Uniform}(\{1,\dots, i-1\}.
\end{equation*}
This corresponds to taking the geometric mean of privacy losses, which we already used in deriving our R\'enyi divergence bound from~\cref{lem:remove-renyi-bound}.

\subsection{Option 2: Maximizing the expectation based on component distance}
Another heuristic for maximizing the mean is to apply quadratic expansion, i.e., 
\begin{equation*}
    \nu_k = - \frac{1}{2 \sigma^2} \mathbb{E}_{j \sim \psi^{(k)}}\left[||\vrho_k - \vmu_j||_2^2\right]
    + \frac{||\vrho_k - \vmu_i||_2^2}{2\sigma^2} - \mathrm{KL}(\psi^{(k)} || \pi),
\end{equation*}
and notice that the second term does not depend on variational distribution $\psi^{(k)}$.
Instead of minimizing the KL-divergence, we could instead define $\psi^{(k)}$ to minimize the distance in the first term, i.e., $\psi^{(k)}(z=j) = 1 \iff j = \mathrm{argmin} ||\rho_k - \vmu_j||_2^2$.
To smoothly interpolate between this solution and minimization of the KL divergence, we can define
the variational family as
$\Psi^{(k)} = \{\psi(t, \vrho_k, (\vmu_1,\dots,\vmu_{i-1}\}) \mid t \in T\}$ with
\begin{equation*}
    \psi(t, \vrho_k, (\vmu_1,\dots,\vmu_{i-1})\}
    = \mathrm{softmax}\left( \left\{- \frac{||\vrho_k - \vmu_j||_2^2}{t } \middle| 1 \leq j < i\right\} \right)
\end{equation*}
for some fixed set of non-negative temperatures $T \subset \sR_+$.

\subsection{Option 3: Minimizing the standard deviation based on component distance}\label{appendix:option_3_variational}
Another heuristic for tightening the tail bound in~\cref{eq:single_component_tail_bound}
is by minimizing the standard deviation $\nu_k$, which makes the argument of Gaussian CDF $\Phi$ smaller for any $\tau \leq \nu_k$, i.e., any tail probability smaller than $0.5$.
Analogous to the previous approach, we can define
$\Psi^{(k)} = \{\psi(t, \vmu_i, (\vmu_1,\dots,\vmu_{i-1})\} \mid t \in T\}$ with
\begin{equation*}
    \psi(t, \vmu_i, (\vmu_1,\dots,\vmu_{i-1})\}
    = \mathrm{softmax}\left( \left\{- \frac{||\vmu_i - \vmu_j||_2^2}{t } \middle| 1 \leq j < i\right\} \right)
\end{equation*}
for some fixed set of non-negative temperatures $T \subset \sR_+$.
Notice that, unlike the previous approach, this variational family is independent of the component mean $\vrho_k$, i.e.,
we use the same variational family across all components of $\tilde{R}$.

\subsection{Our choice.}
In our experiments, we opt for Option 3, since it only requires evaluation of a linear number of distances $||\vmu_i - \vmu_j||$,
and this computational cost is shared across all $K$ components of our tail bound.
We additionally include $\tau = \infty$, which corresponds to the KL-minimizing uniform distribution from Option 1.
For our experiments, we specifically use temperatures $T = \{\infty\} \cup \{10^{-1}, 10^{-0.5}, 1, 10^{0.5}, 10^{1}\}$, which we found to improve tightness of our accountant compared to Option 1 ($T = \{\infty\}$) alone, i.e., the geometric mean.
Furthermore, we empirically did not find Option 2 to offer any increased tightness.
\cref{fig:experiment_variational_family,fig:variational_bound_balation} demonstrates this through an empirical comparison across different matrix factorizations and noise levels in the multi-epoch setting,
visualizing both the reverse hazard function and the resultant privacy profiles.
\cref{table:variational_bound} additionally shows the effect on calibrated noise multipliers $\sigma$, where the variational family results in slightly reduced noise and the relative improvement grows with privacy budget $\epsilon$.

\begin{table}[h]
    \caption{Effect of using the proposed variational family with temperature range $T \in \{\infty\} \cup \mathrm{np.logspace}(-2,2,9)$ instead of only using the geometric mean (AM-GM) for calibrating noise multipliers $\sigma$ under privacy budget $\delta=10^{-5}$
    and varying $\epsilon$.}
    \label{table:variational_bound}
    \centering
    \begin{tabular}{lr|cccccc}
        \toprule
         Mechanism & Bound & $\epsilon = 0.5$ & $\epsilon = 1$ & $\epsilon = 2$ & $\epsilon = 4$ & $\epsilon = 8$ & $\epsilon = 16$ \\
         \midrule
         BSR(4) & AM-GM & 9.286 & 7.026 & 5.526 & 4.463 & 3.655 & 2.991 \\
         BSR(4) & VI(-2,2,9) & 9.271 & 7.009 & 5.507 & 4.440 & 3.625 & 2.960 \\
         BISR(4) & AM-GM & 11.764 & 8.650 & 6.611 & 5.183 & 4.113 & 3.267 \\
         BISR(4) & VI(-2,2,9) & 11.732 & 8.614 & 6.575 & 5.141 & 4.057 & 3.208 \\
         \bottomrule
    \end{tabular}
\end{table}

\clearpage

\begin{figure}[h]
    \centering

    \begin{subfigure}[t]{0.45\linewidth}
        \centering
        \includegraphics[width=\linewidth]{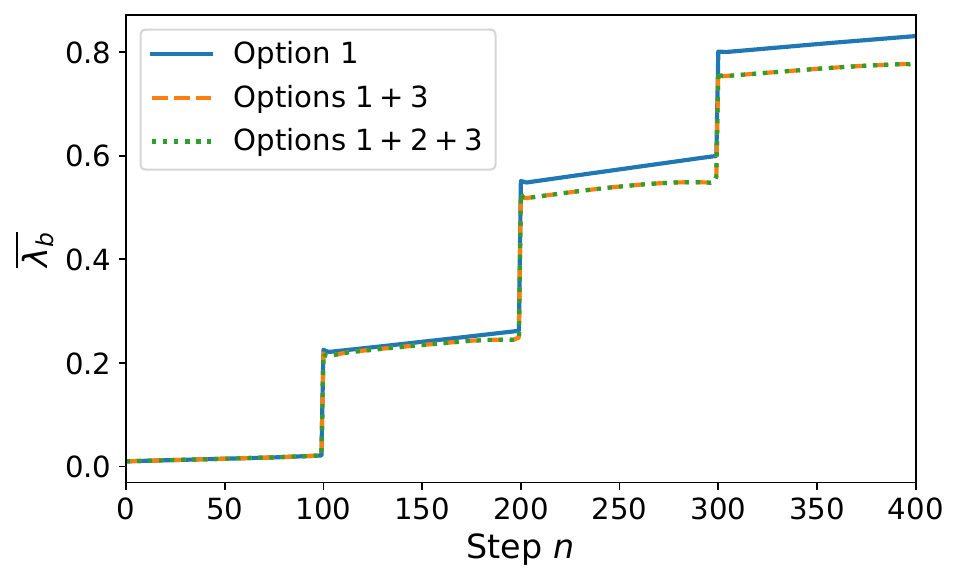}
        \subcaption{BSR, $N=400,\ k=4,\ p=4,\ \sigma=2$}
    \end{subfigure}\hfill
    \begin{subfigure}[t]{0.45\linewidth}
        \centering
        \includegraphics[width=\linewidth]{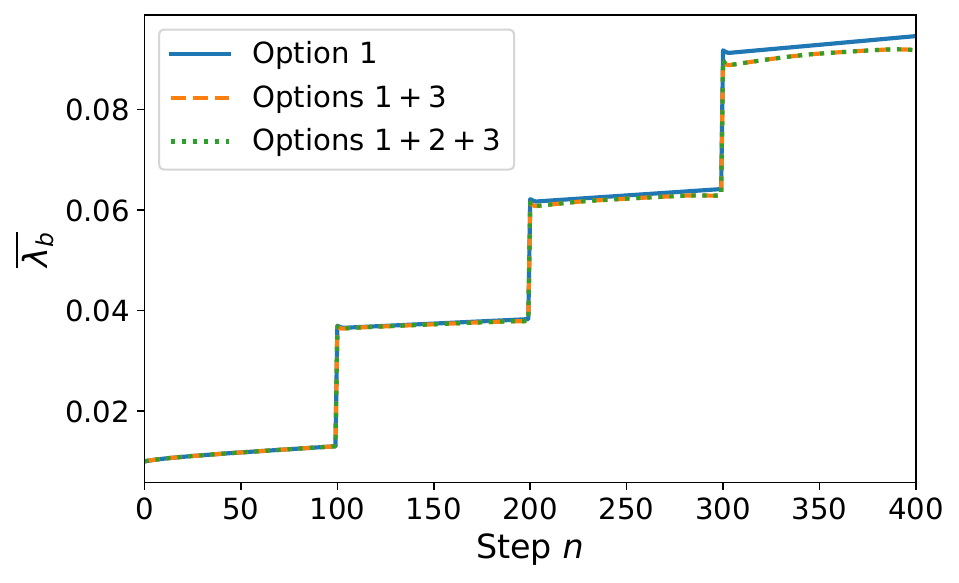}
        \subcaption{BSR, $N=400,\ k=4,\ p=4,\ \sigma=5$}
    \end{subfigure}\hfill
    
    \begin{subfigure}[t]{0.45\linewidth}
        \centering
        \includegraphics[width=\linewidth]{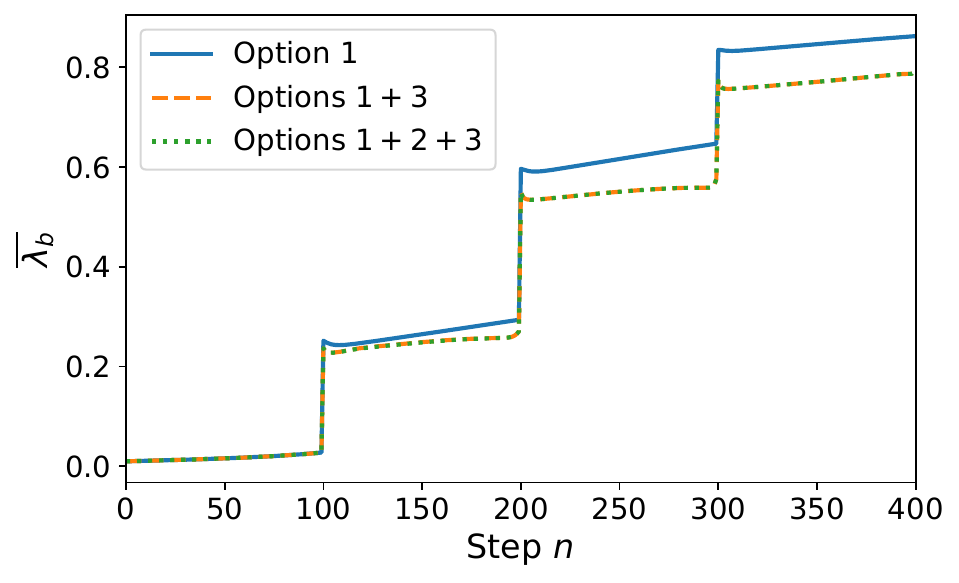}
        \subcaption{BISR, $N=400,\ k=4,\ p=4,\ \sigma=2$}
    \end{subfigure}\hfill
    \begin{subfigure}[t]{0.45\linewidth}
        \centering
        \includegraphics[width=\linewidth]{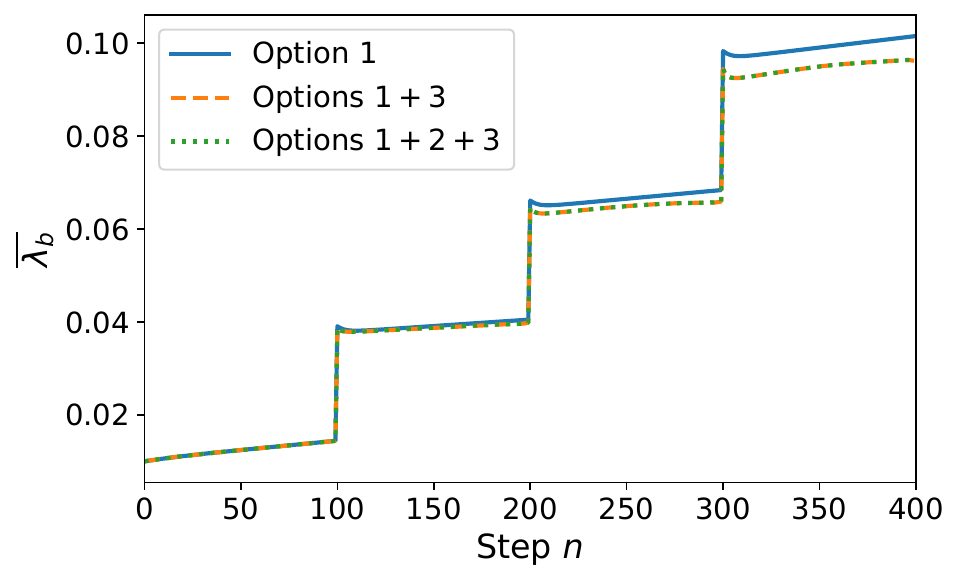}
        \subcaption{BISR, $N=400,\ k=4,\ p=4,\ \sigma=5$}
    \end{subfigure}\hfill
    \caption{Comparison of the largest mixture component's posterior probability $\overline{\lambda_b}$  (smaller is better)
    between variational families 
    for different factorizations and noise multipliers $\sigma$ under $\delta_E = 10^{-5}$ and the ``remove'' relation.
    In later epochs, minimizing variance of the tail bound (``Options 1+3'') improves upon only taking the geometric mean 
    via a uniform distribution (``Option 1'').
    The effect is more pronounced in settings where  the noise multiplier is small and sensitivity is large ($\sigma=2$, BISR),
    compared to settings where the noise multiplier is large and the overall sensitivity is small ($\sigma=5$, BSR).
    Maximizing the mean (``Option 2'') does not further improve the posterior bound.
    }
    \label{fig:experiment_variational_family}
\end{figure}

\clearpage

\begin{figure*}[ht]
    \centering
    \begin{subfigure}[t]{0.45\linewidth}
        \centering
        \includegraphics[width=\linewidth]{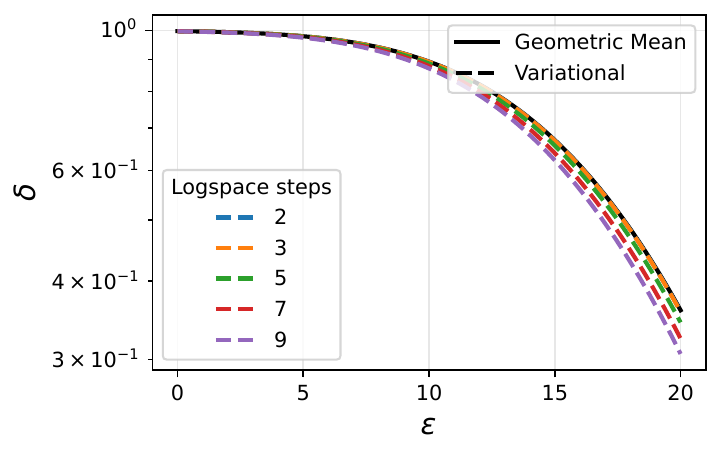}
        \subcaption{BSR ($N=100, k=4, p=4$), $\sigma=2$}
    \end{subfigure}\hfill
    \begin{subfigure}[t]{0.45\linewidth}
        \centering
        \includegraphics[width=\linewidth]{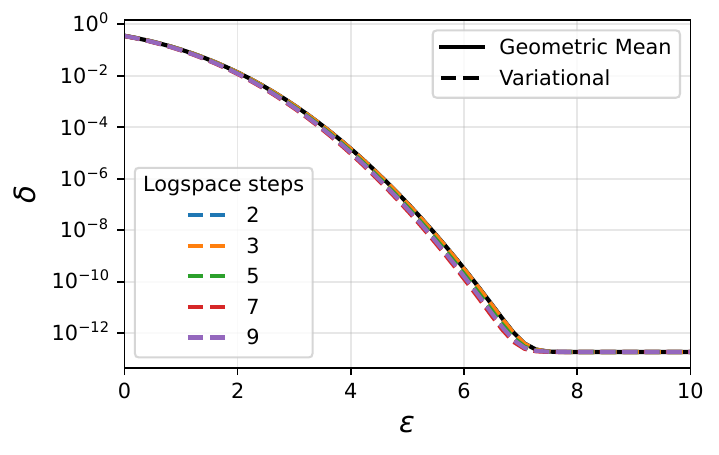}
        \subcaption{BSR ($N=100, k=4, p=4$), $\sigma=4$}
    \end{subfigure}
    \\
    \begin{subfigure}[t]{0.45\linewidth}
        \centering
        \includegraphics[width=\linewidth]{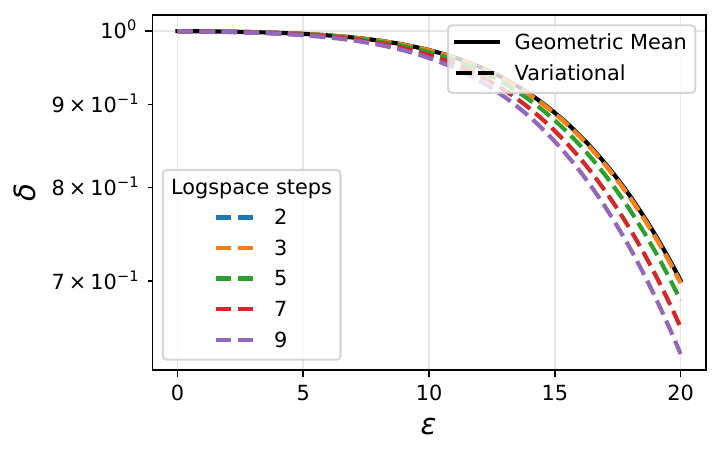}
        \subcaption{BISR ($N=100, k=4, p=4$), $\sigma=2$}
    \end{subfigure}\hfill
    \begin{subfigure}[t]{0.45\linewidth}
        \centering
        \includegraphics[width=\linewidth]{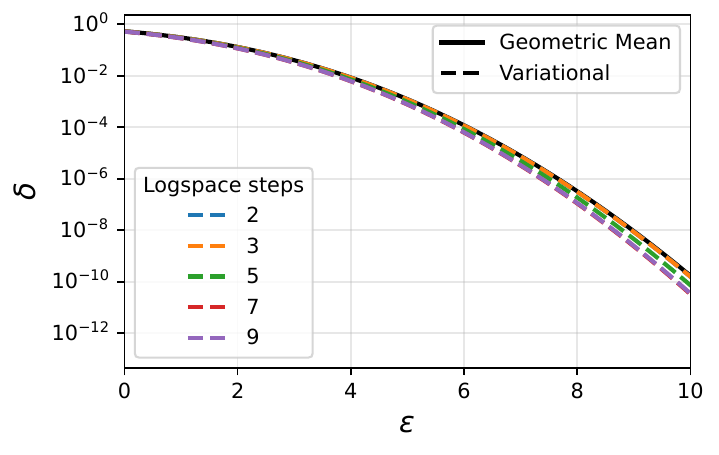}
        \subcaption{BISR ($N=100, k=4, p=4$), $\sigma=2$}
    \end{subfigure}
    \caption{Effect of using the proposed variational family with temperature range $T \in \{\infty\} \cup \mathrm{np.logspace}(-2,2,\mathrm{steps})$ instead of only using the geometric mean ($T = \infty$).
    Using multiple temperatures leads to an improvement in privacy parameter $\delta$, especially for small noise multiplier $\sigma$ / large privacy parameter $\epsilon$.
    }
    \label{fig:variational_bound_balation}
\end{figure*}

\subsection{Remark on Jumps in Reverse Hazard in Multi-Epoch Setting}\label{appendix:reverse_hazard_jumps}
At first, the jumps in our reverse hazard upper bound, which can be observed in~\cref{fig:experiment_variational_family} may seem abnormal or like an implementation error.
However, this is not the case. Instead, these jumps explain why we see a degradation in bound tightness in the multi-epoch setting, especially for DP-SGD. 

Consider DP-SGD with $k$ epochs, $b$ batches per epoch, and $N = k \cdot b$ iterations overall. From the definition of our dominating pair in~\cref{lem:dominating_pair}, we know that the mixture means $\mm \in \sR^{k \cdot N}$ are simply
\begin{equation*}
    \mm = \begin{bmatrix}
        \eye_{b}\\
        \eye_{b}\\
        \cdots\\
        \eye_{b}
    \end{bmatrix},
\end{equation*}
i.e., a vertical stack of identity matrices, each with shape $b \times b$.
Consider an arbitrary step $n \in [N]$.
Recall from~\cref{algorithm:conditional_composition} that, to find our bound $\overline{\lambda_b}$ on the reverse hazard of the largest mixture component, we derive a lower tail bound on the ternary privacy loss $L_{\tilde{P}, \tilde{Q},\tilde{R}} = \log\left(\frac{\dd \tilde{P}}{\tilde{Q}}\right)(\vx)$ with $\vx \sim \tilde{R}$. Here, $\tilde{P}$ is a uniform mixture with means $\vmu_1, \dots, \vmu_{b-1} \in \{0,1\}^{n-1}$ taken from $\mm_{:, 1:n-1}$.
The distribution $\tilde{Q}$ is a single Gaussian with mean $\vmu_b \in \{0,1\}^{n-1}$.
In epoch $e = (n \mod b) + 1$, the single mean $\vmu_b$ of $\tilde{Q}$ has $\max\{e-1, 0\}$ non-zero entries.
Similarly, each mean $\vmu_j$ of $\tilde{P}$ has either $e$ or $\max\{e-1,0\}$ non-zero entries in distinct dimensions.
Thus, the pairwise distance between $\vmu_i$ and any $\vmu_j$ fulfills
\begin{equation*}
    ||\vmu_j - \vmu_i||_2 \in \{\sqrt{2e-1}, \sqrt{2e-2}\}.
\end{equation*}
Thus, the privacy loss increases from epoch to epoch.
In~\cref{fig:posterior_jumps}, we empirically verify this effect by taking $\SI{100000}{}$ samples from $L_{\tilde{P}, \tilde{Q}, \tilde{R}}$
for DP-SGD at the end and start of each epoch for $k=4$, $b=100$, $\sigma=5$.
We observe that the sample histogram of the privacy loss distribution significantly widens when transitioning from one epoch to the next.

\begin{figure}[H]
    \centering

    \begin{subfigure}[t]{0.3\linewidth}
        \centering
        \includegraphics[width=\linewidth]{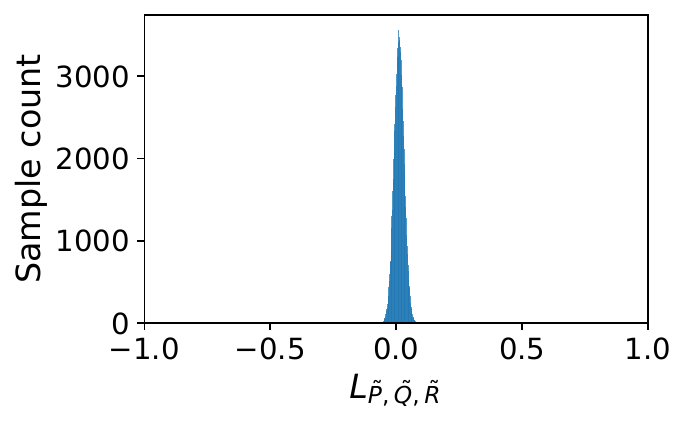}
        \subcaption{Step $100$}
    \end{subfigure}\hfill
    \begin{subfigure}[t]{0.3\linewidth}
        \centering
        \includegraphics[width=\linewidth]{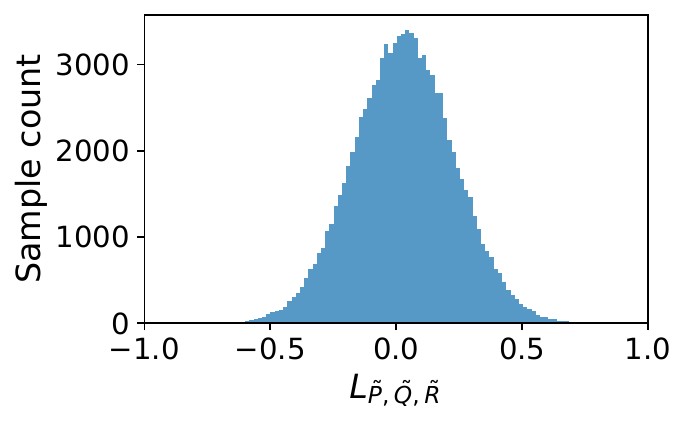}
        \subcaption{Step $200$}
    \end{subfigure}\hfill
    \begin{subfigure}[t]{0.3\linewidth}
        \centering
        \includegraphics[width=\linewidth]{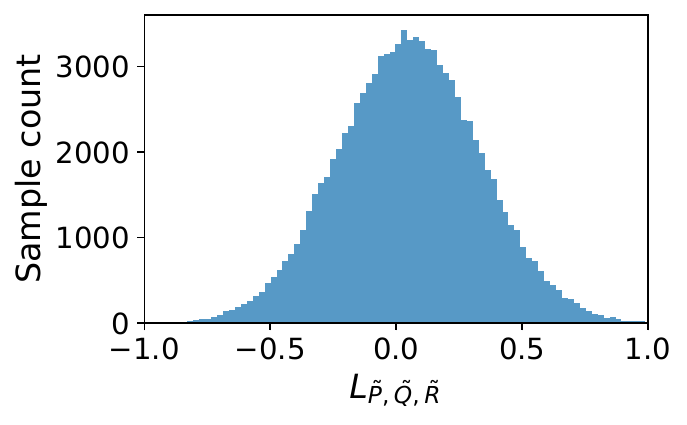}
        \subcaption{Step $300$}
    \end{subfigure}
    
    \begin{subfigure}[t]{0.3\linewidth}
        \centering
        \includegraphics[width=\linewidth]{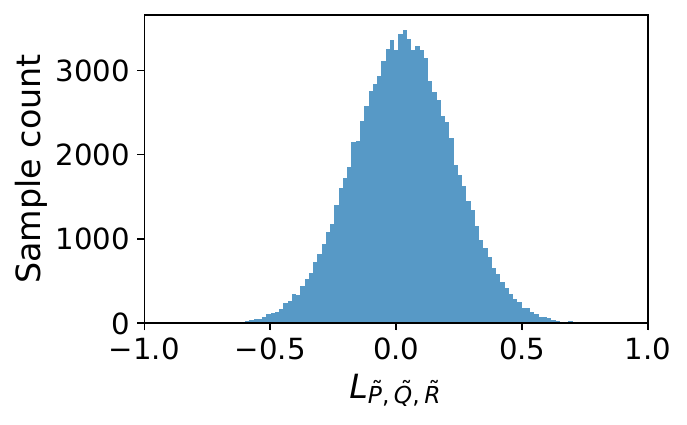}
        \subcaption{Step $101$}
    \end{subfigure}\hfill
    \begin{subfigure}[t]{0.3\linewidth}
        \centering
        \includegraphics[width=\linewidth]{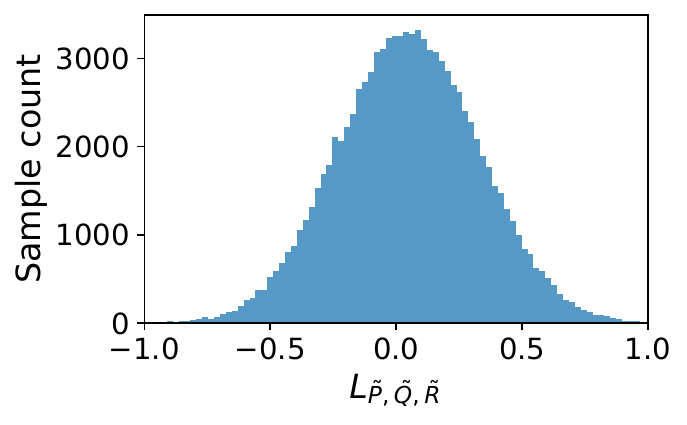}
        \subcaption{Step $201$}
    \end{subfigure}\hfill
    \begin{subfigure}[t]{0.3\linewidth}
        \centering
        \includegraphics[width=\linewidth]{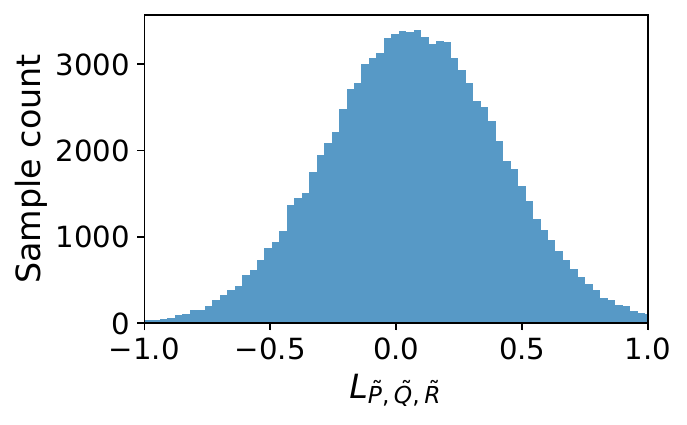}
        \subcaption{Step $301$}
    \end{subfigure}

    \caption{Sample histogram of the ternary privacy loss $L_{\tilde{P},\tilde{Q},\tilde{R}}$, whose lower tail is used to find reverse hazard upper bound $\overline{\lambda_b}$, shown at transitions between epochs.
    Samples are taken for DP-SGD with $k=4$ epochs, $b=100$ batches per epoch, and noise multiplier $\sigma=5$ for the ``remove'' relation.
    The histogram widens between epochs, which explains observed jumps in the reverse hazard bound.}
    \label{fig:posterior_jumps}
\end{figure}

\clearpage

\section{Add and Remove Direction Dominating Pairs}\label{appendix:add_remove_flip}
Recall that we assume the following dominating pair:
\fulldominatingpair*
That is:
\begin{itemize}[nosep]
    \item In the ``remove'' case, $P$ is a Gaussian mixture and $Q$ is a single multivariate Gaussian.
    \item In the ``add'' case $P$ is a single Gaussian and $Q$ is a Gaussian mixture.
\end{itemize}
In the original Lemma 3.2 of~\cite{choquette2024near}, the directions are reversed. However, this appears to be a typing error.
Their proof for the ``add'' case actually argues correctly that the first argument of the hockey-stick divergence should be a zero-mean  Gaussian and the second argument should be a Gaussian mixture (note that in their proof $\vz$ is a zero-mean Gaussian):
``We give the proof for the \underline{add adjacency} [...]
Because each user is assigned to their batch independently, we can assume without loss of generality that contributions from all users other than the differing user are always $0$. In more detail, by post-processing, we can assume that we release the contributions to the input matrix of all examples except the differing user's. Let $\vx$ be these contributions, and $\vx'$ be $\vx$ plus the contribution of the differing user. Then distinguishing $\mC \vx + \vz$ and $\mC \vx' + \vz$ is equivalent to \underline{distinguishing $\vz$ and $\mC (\vx' - \vx) + \vz$} [ ...]``.

Further note that the same pattern of $P$ being a mixture under ``remove'' and $Q$ being a mixture under ``add'' arises in various other subsampling schemes, such as Poisson subsampling or subsampling without replacement (see, e.g., Theorem 11 in~\cite{zhu2022optimal}).

\clearpage

\section{Inapplicability of privacy loss distribution accounting}\label{appendix:pld_inapplicability}

In the following, we explain in slightly more detail why existing privacy loss distribution (PLD) accountants cannot be applied to matrix mechanisms under random allocation. We first consider the single-epoch setting ($k=1$), before discussing the multi-epoch setting ($k > 1$), which introduces further complications.

\subsection{Single-epoch setting}
In the single-epoch setting, the number of participation patterns $b$ equals the number of steps $N$ and each record participates to exactly one step. 
Thus, in~\cref{eq:dominating_pair} the number of mixture components $\mu_i = \mathcal{N}(\vm_i, \sigma^2 \eye_N)$ equals the number of steps $N$
and the mixture means are indicator vectors $\vm_i = \ve_i$.
Note that, in this section, subscripts $i$ of $\mu$ denote mixture components, whereas subscripts $n$ of $P$ and $Q$ denote steps.

The \emph{approach} underlying PLD accounting is that, once we know the distribution of privacy loss random variable $L = \log \left(\frac{\dd P}{ \dd Q}(\vy)\right)$ with $\mathbf{y} \sim P$ for dominating pair $(P,Q)$, the privacy profile $\delta(\epsilon) = H_{e^\epsilon}(P||Q)$ can be evaluated via a simple univariate integral.

PLD accounting as an \emph{efficient algorithmic procedure} becomes feasible when PLRV $L_{P,Q}$ can be expressed as a simple function of ``smaller'' PLRVs $L_1,\dots,L_M$ whose distribution is known and easy to evaluate.

The standard use case~\cite{sommer2019privacyloss,koskela2020computing,zhu2022optimal} is that dominating pairs for $N$-fold adaptive composition
factorize as $P = \bigtimes_{n=1}^N P_n$ and $Q = \bigtimes_{n=1}^N Q_n$, which implies $L = \sum_{n=1}^N L_{P_n, Q_n}$.
However, as we already explained in~\cref{section:problem_statement}, our distribution 
$P$ is not factorizable for random allocation, even with uncorrelated noise ($\mC = \eye_N$).
This is because $P$ is a mixture distribution $P = \frac{1}{b}  \sum_{i=1}^b \mu_i$.

For uncorrelated noise, \citet{feldman2026efficient} show that one can nevertheless break down $L_{P,Q}$ into $N$ simple PLRVs (see their Theorem 4.4).
Specifically, they exploit that each $\mu_i$ only differs from $Q$ in a single dimension, specifically $\mu_i = \bigtimes_{n=1}^{i-1} S_n \times R \times \bigtimes_{n=i+1}^N S_n$  and $Q = \bigtimes_{n=1}^N S_n$, where $R$ and all $S_n$ are univariate Gaussians. Thus, the privacy loss simplifies to
\begin{equation*}
    \log\left(\frac{\dd P}{\dd Q}(\vy)\right)
    =
    \log\left(
    \frac{1}{b}
    \sum_{i=1}^b \frac{\dd \mu_i}{\dd Q}(\vy) \right)
    = \log\left(
    \frac{1}{N}
    \sum_{n=1}^N \exp\left( \log \left( \frac{\dd R}{\dd S}(y_n) \right) \right)\right).
\end{equation*}
This is a relatively simple analytical function of $N$ per-step privacy losses. These $N$ privacy losses become independent PLRVs when sampling $\mathbf{y} \sim \mu_i$ conditioned on any mixture component $i$.

In contrast, with matrix mechanisms, each component $\mu_i$ has a different mean $\vm_i$ (recall~\cref{eq:dominating_pair})
with multiple non-zero entries.
Thus, the per-component likelihood ratios $\frac{\dd \mu_i}{\dd Q}(\vy)$ are functions of $\vm_i^T \vy$.
Thus, and because the means are generally not orthogonal, each summand of 
$\log\left(
    \frac{1}{b}
    \sum_{i=1}^b \frac{\dd \mu_i}{\dd Q}(\vy) \right)$
depends on multiple elements of random vector $\vy$, which breaks the independence assumption.

To conclude: While random allocation breaks decomposition into per-step privacy losses, noise correlation further breaks decomposition into per-component privacy losses.

\subsection{Multi-epoch setting}
Going beyond $k=1$ poses further challenges that make the PLD accountant of~\citet{feldman2026efficient} inapplicable.
First, it should be noted that there are actually two variants of balls-and-bins sampling or random allocation:
The first one~\cite{chua2024balls,choquette2024near}, which we use in~\cref{definition:bnb}, repeats the same participation pattern in each of the $k$ epochs.
We adopted this definition for consistency with prior work on matrix mechanism accounting by~\citet{choquette2024near}.
They likely chose it because the number of mixture components is linear in the number of batches per epoch $b$,
which facilitates Monte Carlo accounting.
In contrast,~\citet{feldman2025privacy,feldman2026efficient} and~\citet{dong2025leveraging} let each record choose $k$ steps to participate in uniformly at random, resulting in a combinatorally growing number of possible participation patterns.

However, even if one were to apply matrix mechanisms to their notion of random allocation, the proof strategy from~\citet{feldman2026efficient} would break down. Concretely, their first step is a reduction from $k>1$ to $k=1$  (see Lemma 3.2 in~\citet{feldman2025privacy} and Lemma 2.8 in~\citet{feldman2026efficient}).
This reduction relies on randomly splitting the steps $\{1,\dots,N\}$ into $N \mathbin{/} k$ subsets and applying the subsampled mechanism independently to each subset.
However, this is fundamentally incompatible with matrix mechanisms:
They are specifically designed to map a fixed number of inputs $N$ (e.g., batches) to a fixed number of outputs $N$ (e.g., batch gradients)
by sampling from a joint distribution with optimized strategy matrix $\mC \in \sR^{N \times N}$.
A similar reduction could potentially be possible if one were to impose a block-diagonal structure on  $\mC$. However, this is not standard practice.

\clearpage

\section{Broader impact}\label{appendix:broader_impact}
Our work is primarily targeted at mitigating negative societal impact of machine learning by providing formal privacy protection.
However, it shares various potential risks, including risks for misuse, with other works.
Firstly, differential privacy is an inherently probabilistic notion of privacy, but may be misrepresented as providing qualitatively identical results to deterministic cryptographic methods (e.g., encryption).
Furthermore, differentially private machine learning in practice is usually conducted in settings with $\delta < |S|^{-1}$ but $\epsilon > 1$~\cite{ponomareva2023dp}.
Here, differential privacy allows for relatively comparison of privacy between models, but does not provide strong privacy protection in absolute sense~\cite{ponomareva2023dp}. Nevertheless, the claim that a model was trained with some form of differential privacy may be used to incorrectly privacy-wash publicly released methods.
This is compounded by the fact that numerical privacy parameters $(\epsilon,\delta)$ are not easy to interpret by non-experts.
Another issue is that, even if an algorithm is formally proven to fulfill certain privacy guarantees, 
it's practical implementations may violate them due to implementation errors or numerical issues.
This may affect matrix mechanisms in particular, since they add additional complication compared to uncorrelated noise (DP-SGD) model training.
A potential way to mitigate this risk is through privacy auditing methods (see~\cite{namatevs2025privacy} for an overview). But these may themselves be affected by implementation issues.

\clearpage


\section*{NeurIPS Paper Checklist}

\begin{enumerate}

\item {\bf Claims}
    \item[] Question: Do the main claims made in the abstract and introduction accurately reflect the paper's contributions and scope?
    \item[] Answer: \answerYes{} 
    \item[] Justification: Yes. The claim that our methods are sampling-free is inherent to their design. 
    The claim that our R\'enyi accountant uses dynamic programming for efficient computation is supported by~\cref{alg:renyi_dynamic_program} and the runtime discussion in~\cref{section:rdp_accountant}. The claim that conditional composition complements it by offering stronger guarantees in the low-$\epsilon$ regime is supported by experimental results in~\cref{section:experiments} and~\cref{appendix:extra_experiments}.
    The claim that we support arbitrary banded and non-banded matrices is inherent to the design of our algorithms.
    The claim that we conduct experiments across a wide range of matrix mechanisms is supported by experiments in~\cref{section:experiments} and~\cref{appendix:extra_experiments}, where we consider DP-SGD, BSR, BISR, BLT, BandMF, and BandInvMF factorizations.
    Our claim that we improve upon the computational complexity of the best known R\'enyi accountant for random allocation and uncorrelated noise is formally proven as part of~\cref{lem:renyi-dynamic_program}.
    \item[] Guidelines:
    \begin{itemize}
        \item The answer \answerNA{} means that the abstract and introduction do not include the claims made in the paper.
        \item The abstract and/or introduction should clearly state the claims made, including the contributions made in the paper and important assumptions and limitations. A \answerNo{} or \answerNA{} answer to this question will not be perceived well by the reviewers. 
        \item The claims made should match theoretical and experimental results, and reflect how much the results can be expected to generalize to other settings. 
        \item It is fine to include aspirational goals as motivation as long as it is clear that these goals are not attained by the paper. 
    \end{itemize}

\item {\bf Limitations}
    \item[] Question: Does the paper discuss the limitations of the work performed by the authors?
    \item[] Answer: \answerYes{}
    \item[] Justification: See ``limitations'' paragraph at end of~\cref{section:conclusion}.
    \item[] Guidelines:
    \begin{itemize}
        \item The answer \answerNA{} means that the paper has no limitation while the answer \answerNo{} means that the paper has limitations, but those are not discussed in the paper. 
        \item The authors are encouraged to create a separate ``Limitations'' section in their paper.
        \item The paper should point out any strong assumptions and how robust the results are to violations of these assumptions (e.g., independence assumptions, noiseless settings, model well-specification, asymptotic approximations only holding locally). The authors should reflect on how these assumptions might be violated in practice and what the implications would be.
        \item The authors should reflect on the scope of the claims made, e.g., if the approach was only tested on a few datasets or with a few runs. In general, empirical results often depend on implicit assumptions, which should be articulated.
        \item The authors should reflect on the factors that influence the performance of the approach. For example, a facial recognition algorithm may perform poorly when image resolution is low or images are taken in low lighting. Or a speech-to-text system might not be used reliably to provide closed captions for online lectures because it fails to handle technical jargon.
        \item The authors should discuss the computational efficiency of the proposed algorithms and how they scale with dataset size.
        \item If applicable, the authors should discuss possible limitations of their approach to address problems of privacy and fairness.
        \item While the authors might fear that complete honesty about limitations might be used by reviewers as grounds for rejection, a worse outcome might be that reviewers discover limitations that aren't acknowledged in the paper. The authors should use their best judgment and recognize that individual actions in favor of transparency play an important role in developing norms that preserve the integrity of the community. Reviewers will be specifically instructed to not penalize honesty concerning limitations.
    \end{itemize}

\item {\bf Theory assumptions and proofs}
    \item[] Question: For each theoretical result, does the paper provide the full set of assumptions and a complete (and correct) proof?
    \item[] Answer: \answerYes{} 
    \item[] Justification: We formally prove all claims while explicitly stating all involved assumptions in~\cref{appendix:renyi_proofs,appendix:proofs_conditional_composition,appendix:proofs_variational}.
    \item[] Guidelines:
    \begin{itemize}
        \item The answer \answerNA{} means that the paper does not include theoretical results. 
        \item All the theorems, formulas, and proofs in the paper should be numbered and cross-referenced.
        \item All assumptions should be clearly stated or referenced in the statement of any theorems.
        \item The proofs can either appear in the main paper or the supplemental material, but if they appear in the supplemental material, the authors are encouraged to provide a short proof sketch to provide intuition. 
        \item Inversely, any informal proof provided in the core of the paper should be complemented by formal proofs provided in appendix or supplemental material.
        \item Theorems and Lemmas that the proof relies upon should be properly referenced. 
    \end{itemize}

    \item {\bf Experimental result reproducibility}
    \item[] Question: Does the paper fully disclose all the information needed to reproduce the main experimental results of the paper to the extent that it affects the main claims and/or conclusions of the paper (regardless of whether the code and data are provided or not)?
    \item[] Answer: \answerYes{}{} 
    \item[] Justification: The used algorithms are clearly described in~\cref{alg:renyi_dynamic_program_full,alg:renyi_dynamic_program,alg:renyi_dynamic_program_remove,algorithm:conditional_composition}.
    The involved parameters for our experiments are specified in~\cref{appendix:experimental_setup}.
    Code was shared via the supplementary material during the review period, and will be made publicly available at \href{https://github.com/morganstanley/MSML/tree/main/papers/Matrix_Mechanism_Accounting}{github.com/morganstanley/MSML/tree/main/papers/Matrix\_Mechanism\_Accounting}.
    \item[] Guidelines:
    \begin{itemize}
        \item The answer \answerNA{} means that the paper does not include experiments.
        \item If the paper includes experiments, a \answerNo{} answer to this question will not be perceived well by the reviewers: Making the paper reproducible is important, regardless of whether the code and data are provided or not.
        \item If the contribution is a dataset and\slash or model, the authors should describe the steps taken to make their results reproducible or verifiable. 
        \item Depending on the contribution, reproducibility can be accomplished in various ways. For example, if the contribution is a novel architecture, describing the architecture fully might suffice, or if the contribution is a specific model and empirical evaluation, it may be necessary to either make it possible for others to replicate the model with the same dataset, or provide access to the model. In general. releasing code and data is often one good way to accomplish this, but reproducibility can also be provided via detailed instructions for how to replicate the results, access to a hosted model (e.g., in the case of a large language model), releasing of a model checkpoint, or other means that are appropriate to the research performed.
        \item While NeurIPS does not require releasing code, the conference does require all submissions to provide some reasonable avenue for reproducibility, which may depend on the nature of the contribution. For example
        \begin{enumerate}
            \item If the contribution is primarily a new algorithm, the paper should make it clear how to reproduce that algorithm.
            \item If the contribution is primarily a new model architecture, the paper should describe the architecture clearly and fully.
            \item If the contribution is a new model (e.g., a large language model), then there should either be a way to access this model for reproducing the results or a way to reproduce the model (e.g., with an open-source dataset or instructions for how to construct the dataset).
            \item We recognize that reproducibility may be tricky in some cases, in which case authors are welcome to describe the particular way they provide for reproducibility. In the case of closed-source models, it may be that access to the model is limited in some way (e.g., to registered users), but it should be possible for other researchers to have some path to reproducing or verifying the results.
        \end{enumerate}
    \end{itemize}

\item {\bf Open access to data and code}
    \item[] Question: Does the paper provide open access to the data and code, with sufficient instructions to faithfully reproduce the main experimental results, as described in supplemental material?
    \item[] Answer: \answerYes{}{} 
    \item[] Justification: Code was shared via the supplementary material during the review period, and will be made publicly available at \href{https://github.com/morganstanley/MSML/tree/main/papers/Matrix_Mechanism_Accounting}{github.com/morganstanley/MSML/tree/main/papers/Matrix\_Mechanism\_Accounting}. The experimental setup for our main experimental results is fully specified in~\cref{appendix:experimental_setup}.
    Further details for reproducing ancillary results, in which we vary various parameters, are provided with the respective experiments in~\cref{appendix:extra_experiments}.
    The only used data is the CIFAR-10 dataset, which is publicly available.
    \item[] Guidelines:
    \begin{itemize}
        \item The answer \answerNA{} means that paper does not include experiments requiring code.
        \item Please see the NeurIPS code and data submission guidelines (\url{https://neurips.cc/public/guides/CodeSubmissionPolicy}) for more details.
        \item While we encourage the release of code and data, we understand that this might not be possible, so \answerNo{} is an acceptable answer. Papers cannot be rejected simply for not including code, unless this is central to the contribution (e.g., for a new open-source benchmark).
        \item The instructions should contain the exact command and environment needed to run to reproduce the results. See the NeurIPS code and data submission guidelines (\url{https://neurips.cc/public/guides/CodeSubmissionPolicy}) for more details.
        \item The authors should provide instructions on data access and preparation, including how to access the raw data, preprocessed data, intermediate data, and generated data, etc.
        \item The authors should provide scripts to reproduce all experimental results for the new proposed method and baselines. If only a subset of experiments are reproducible, they should state which ones are omitted from the script and why.
        \item At submission time, to preserve anonymity, the authors should release anonymized versions (if applicable).
        \item Providing as much information as possible in supplemental material (appended to the paper) is recommended, but including URLs to data and code is permitted.
    \end{itemize}

\item {\bf Experimental setting/details}
    \item[] Question: Does the paper specify all the training and test details (e.g., data splits, hyperparameters, how they were chosen, type of optimizer) necessary to understand the results?
    \item[] Answer: \answerYes{} 
    \item[] Justification: See~\cref{appendix:experimental_setup} which specifies all experimental parameters.
    Further details for reproducing ancillary results, in which we vary various parameters, are provided with the respective experiments in~\cref{appendix:extra_experiments}.
    \item[] Guidelines:
    \begin{itemize}
        \item The answer \answerNA{} means that the paper does not include experiments.
        \item The experimental setting should be presented in the core of the paper to a level of detail that is necessary to appreciate the results and make sense of them.
        \item The full details can be provided either with the code, in appendix, or as supplemental material.
    \end{itemize}

\item {\bf Experiment statistical significance}
    \item[] Question: Does the paper report error bars suitably and correctly defined or other appropriate information about the statistical significance of the experiments?
    \item[] Answer: \answerYes{}
    \item[] Justification: Our main contribution is that the proposed algorithms are sampling-free, i.e., there is no randomness involved.
    For the Monte Carlo baseline, we show raw Monte Carlo estimates in the main experimental figure, but we additionally calculate high probability confidence intervals in~\cref{appendix:monte_carlo_whp_bounds}.
    For the CIFAR model training experiment in~\cref{appendix:cifar_experiments}, we show standard deviations across random seeds.
    \item[] Guidelines:
    \begin{itemize}
        \item The answer \answerNA{} means that the paper does not include experiments.
        \item The authors should answer \answerYes{} if the results are accompanied by error bars, confidence intervals, or statistical significance tests, at least for the experiments that support the main claims of the paper.
        \item The factors of variability that the error bars are capturing should be clearly stated (for example, train/test split, initialization, random drawing of some parameter, or overall run with given experimental conditions).
        \item The method for calculating the error bars should be explained (closed form formula, call to a library function, bootstrap, etc.)
        \item The assumptions made should be given (e.g., Normally distributed errors).
        \item It should be clear whether the error bar is the standard deviation or the standard error of the mean.
        \item It is OK to report 1-sigma error bars, but one should state it. The authors should preferably report a 2-sigma error bar than state that they have a 96\% CI, if the hypothesis of Normality of errors is not verified.
        \item For asymmetric distributions, the authors should be careful not to show in tables or figures symmetric error bars that would yield results that are out of range (e.g., negative error rates).
        \item If error bars are reported in tables or plots, the authors should explain in the text how they were calculated and reference the corresponding figures or tables in the text.
    \end{itemize}

\item {\bf Experiments compute resources}
    \item[] Question: For each experiment, does the paper provide sufficient information on the computer resources (type of compute workers, memory, time of execution) needed to reproduce the experiments?
    \item[] Answer: \answerYes{} 
    \item[] Justification: We specify the worker type in~\cref{appendix:experimental_setup}. We did not track compute used across the entire project,
    but measure wall-clock runtime for our algorithms across a wide range of parameter choices in~\cref{appendix:empirical_runtime}.
    \item[] Guidelines:
    \begin{itemize}
        \item The answer \answerNA{} means that the paper does not include experiments.
        \item The paper should indicate the type of compute workers CPU or GPU, internal cluster, or cloud provider, including relevant memory and storage.
        \item The paper should provide the amount of compute required for each of the individual experimental runs as well as estimate the total compute. 
        \item The paper should disclose whether the full research project required more compute than the experiments reported in the paper (e.g., preliminary or failed experiments that didn't make it into the paper). 
    \end{itemize}
    
\item {\bf Code of ethics}
    \item[] Question: Does the research conducted in the paper conform, in every respect, with the NeurIPS Code of Ethics \url{https://neurips.cc/public/EthicsGuidelines}?
    \item[] Answer: \answerYes{}{} 
    \item[] Justification: We read the ethics guideline and agreed to them on OpenReview.
    \item[] Guidelines:
    \begin{itemize}
        \item The answer \answerNA{} means that the authors have not reviewed the NeurIPS Code of Ethics.
        \item If the authors answer \answerNo, they should explain the special circumstances that require a deviation from the Code of Ethics.
        \item The authors should make sure to preserve anonymity (e.g., if there is a special consideration due to laws or regulations in their jurisdiction).
    \end{itemize}

\item {\bf Broader impacts}
    \item[] Question: Does the paper discuss both potential positive societal impacts and negative societal impacts of the work performed?
    \item[] Answer: \answerYes{}{}{} 
    \item[] Justification: Yes (see \cref{appendix:broader_impact}, which we reference from~\cref{section:conclusion}). However, our method is primarily intended to mitigate negative societal impact by providing formal privacy guarantees. There are potentials for misuse, but these are shared with any other work on differentially private machine learning.
    \item[] Guidelines:
    \begin{itemize}
        \item The answer \answerNA{} means that there is no societal impact of the work performed.
        \item If the authors answer \answerNA{} or \answerNo, they should explain why their work has no societal impact or why the paper does not address societal impact.
        \item Examples of negative societal impacts include potential malicious or unintended uses (e.g., disinformation, generating fake profiles, surveillance), fairness considerations (e.g., deployment of technologies that could make decisions that unfairly impact specific groups), privacy considerations, and security considerations.
        \item The conference expects that many papers will be foundational research and not tied to particular applications, let alone deployments. However, if there is a direct path to any negative applications, the authors should point it out. For example, it is legitimate to point out that an improvement in the quality of generative models could be used to generate Deepfakes for disinformation. On the other hand, it is not needed to point out that a generic algorithm for optimizing neural networks could enable people to train models that generate Deepfakes faster.
        \item The authors should consider possible harms that could arise when the technology is being used as intended and functioning correctly, harms that could arise when the technology is being used as intended but gives incorrect results, and harms following from (intentional or unintentional) misuse of the technology.
        \item If there are negative societal impacts, the authors could also discuss possible mitigation strategies (e.g., gated release of models, providing defenses in addition to attacks, mechanisms for monitoring misuse, mechanisms to monitor how a system learns from feedback over time, improving the efficiency and accessibility of ML).
    \end{itemize}
    
\item {\bf Safeguards}
    \item[] Question: Does the paper describe safeguards that have been put in place for responsible release of data or models that have a high risk for misuse (e.g., pre-trained language models, image generators, or scraped datasets)?
    \item[] Answer: \answerNA{}{} 
    \item[] Justification: We do not release any new data or models.
    \item[] Guidelines:
    \begin{itemize}
        \item The answer \answerNA{} means that the paper poses no such risks.
        \item Released models that have a high risk for misuse or dual-use should be released with necessary safeguards to allow for controlled use of the model, for example by requiring that users adhere to usage guidelines or restrictions to access the model or implementing safety filters. 
        \item Datasets that have been scraped from the Internet could pose safety risks. The authors should describe how they avoided releasing unsafe images.
        \item We recognize that providing effective safeguards is challenging, and many papers do not require this, but we encourage authors to take this into account and make a best faith effort.
    \end{itemize}

\item {\bf Licenses for existing assets}
    \item[] Question: Are the creators or original owners of assets (e.g., code, data, models), used in the paper, properly credited and are the license and terms of use explicitly mentioned and properly respected?
    \item[] Answer: \answerYes{} 
    \item[] Justification: Yes, we cite the authors/owners in~\cref{appendix:experimental_setup} and specify the respective licenses in~\cref{appendix:licenses}.
    \item[] Guidelines:
    \begin{itemize}
        \item The answer \answerNA{} means that the paper does not use existing assets.
        \item The authors should cite the original paper that produced the code package or dataset.
        \item The authors should state which version of the asset is used and, if possible, include a URL.
        \item The name of the license (e.g., CC-BY 4.0) should be included for each asset.
        \item For scraped data from a particular source (e.g., website), the copyright and terms of service of that source should be provided.
        \item If assets are released, the license, copyright information, and terms of use in the package should be provided. For popular datasets, \url{paperswithcode.com/datasets} has curated licenses for some datasets. Their licensing guide can help determine the license of a dataset.
        \item For existing datasets that are re-packaged, both the original license and the license of the derived asset (if it has changed) should be provided.
        \item If this information is not available online, the authors are encouraged to reach out to the asset's creators.
    \end{itemize}

\item {\bf New assets}
    \item[] Question: Are new assets introduced in the paper well documented and is the documentation provided alongside the assets?
    \item[] Answer: \answerNA{}{}{} 
    \item[] Justification: We do not introduce new assets.
    \item[] Guidelines:
    \begin{itemize}
        \item The answer \answerNA{} means that the paper does not release new assets.
        \item Researchers should communicate the details of the dataset\slash code\slash model as part of their submissions via structured templates. This includes details about training, license, limitations, etc. 
        \item The paper should discuss whether and how consent was obtained from people whose asset is used.
        \item At submission time, remember to anonymize your assets (if applicable). You can either create an anonymized URL or include an anonymized zip file.
    \end{itemize}

\item {\bf Crowdsourcing and research with human subjects}
    \item[] Question: For crowdsourcing experiments and research with human subjects, does the paper include the full text of instructions given to participants and screenshots, if applicable, as well as details about compensation (if any)? 
    \item[] Answer: \answerNA{}{} 
    \item[] Justification: We did not conduct studies with human subjects.
    \item[] Guidelines:
    \begin{itemize}
        \item The answer \answerNA{} means that the paper does not involve crowdsourcing nor research with human subjects.
        \item Including this information in the supplemental material is fine, but if the main contribution of the paper involves human subjects, then as much detail as possible should be included in the main paper. 
        \item According to the NeurIPS Code of Ethics, workers involved in data collection, curation, or other labor should be paid at least the minimum wage in the country of the data collector. 
    \end{itemize}

\item {\bf Institutional review board (IRB) approvals or equivalent for research with human subjects}
    \item[] Question: Does the paper describe potential risks incurred by study participants, whether such risks were disclosed to the subjects, and whether Institutional Review Board (IRB) approvals (or an equivalent approval/review based on the requirements of your country or institution) were obtained?
    \item[] Answer: \answerNA{}{} 
    \item[] Justification: We did not conduct studies with human subjects.
    \item[] Guidelines:
    \begin{itemize}
        \item The answer \answerNA{} means that the paper does not involve crowdsourcing nor research with human subjects.
        \item Depending on the country in which research is conducted, IRB approval (or equivalent) may be required for any human subjects research. If you obtained IRB approval, you should clearly state this in the paper. 
        \item We recognize that the procedures for this may vary significantly between institutions and locations, and we expect authors to adhere to the NeurIPS Code of Ethics and the guidelines for their institution. 
        \item For initial submissions, do not include any information that would break anonymity (if applicable), such as the institution conducting the review.
    \end{itemize}

\item {\bf Declaration of LLM usage}
    \item[] Question: Does the paper describe the usage of LLMs if it is an important, original, or non-standard component of the core methods in this research? Note that if the LLM is used only for writing, editing, or formatting purposes and does \emph{not} impact the core methodology, scientific rigor, or originality of the research, declaration is not required.
    \item[] Answer: \answerNA{}{} 
    \item[] Justification: We do not use LLMS as an important, original, or non-standard component of the core methods.
    \item[] Guidelines:
    \begin{itemize}
        \item The answer \answerNA{} means that the core method development in this research does not involve LLMs as any important, original, or non-standard components.
        \item Please refer to our LLM policy in the NeurIPS handbook for what should or should not be described.
    \end{itemize}

\end{enumerate}

\end{document}